\documentclass[letterpaper]{article} 
\usepackage{times}  
\usepackage{helvet}  
\usepackage{courier}  
\usepackage[hyphens]{url}  
\usepackage{graphicx} 
\usepackage{natbib}  
\usepackage{caption} 
\usepackage{amsthm}

\usepackage{mathtools}

\newcommand{\diagentry}[1]{\mathmakebox[1.8em]{#1}}
\newcommand{\xddots}{%
  \raise 4pt \hbox {.}
  \mkern 6mu
  \raise 1pt \hbox {.}
  \mkern 6mu
  \raise -2pt \hbox {.}
}

\usepackage{comment}

\usepackage{cite}
\usepackage{amsmath,amsfonts}
\usepackage{amssymb}
\usepackage{algorithm,algorithmic}

\usepackage{newfloat}
\usepackage{listings}
\DeclareCaptionStyle{ruled}{labelfont=normalfont,labelsep=colon,strut=off} 
\floatstyle{ruled}
\newfloat{listing}{tb}{lst}{}
\floatname{listing}{Listing}
\newtheorem{definition}{Definition}
\newtheorem{theorem}{Theorem}
\newtheorem{prop}{Proposition}

\newtheorem{remark}{Remark}

\newtheorem{lemma}{Lemma}

\newtheorem{define}{Definition}
\newtheorem{assumption}{Assumption}
\newtheorem{cond}{Condition}

\begin{document}

\title{FLUE: Federated Learning with Un-Encrypted model weights

}

\author{Elie Atallah
        \\
        {Resch School of Engineering}\\
{University of Wisconsin, Green Bay}\\
{atallahe@uwgb.edu}
\thanks{}
}

\maketitle

\begin{abstract}

Federated Learning enables diverse devices to collaboratively train a shared model while keeping training data locally stored, avoiding the need for centralized cloud storage. Despite existing privacy measures, concerns arise from potential reverse engineering of gradients, even with added noise, revealing private data. To address this, recent research emphasizes using encrypted model parameters during training. This paper introduces a novel federated learning algorithm, leveraging coded local gradients without encryption, exchanging coded proxies for model parameters, and injecting surplus noise for enhanced privacy. Two algorithm variants are presented, showcasing convergence and learning rates adaptable to coding schemes and raw data characteristics. Two encryption-free implementations with fixed and random coding matrices are provided, demonstrating promising simulation results from both federated optimization and machine learning perspectives.

\end{abstract}

\vspace{1cm}

\ ML: Distributed Machine Learning \& Federated Learning, ML: Optimization, Gradient Coding

\vspace{1cm}

\section{Introduction}

The widespread adoption of IoT technologies has led to a surge in user data, prompting the need for privacy-preserving measures. Federated Learning (FL) has emerged as a decentralized solution, allowing clients to update their local models with encrypted global parameters instead of raw data during training. This paper introduces Federated Learning with Unencrypted Model Weights (FLUE), presenting an algorithm that ensures data privacy through coded local gradients without relying on encryption. The exchange of coded combinations' proxies for learned model parameters, coupled with the injection of surplus noise, strengthens privacy safeguards. The paper offers two algorithm variants and demonstrates their convergence, establishing learning rates adaptable to coding schemes and raw data characteristics. The study showcases promising simulation results from federated optimization and machine learning perspectives, implementing encryption-free solutions with fixed and random coding matrices.

\subsection{Key Contributions}

The key contributions of FLUE are:

\begin{itemize}
\item Implementing privacy-preserving Federated Learning without relying on encrypted model parameters or exchanging local gradients. 
\item Employing proxy model parameters between server and nodes to facilitate data training and model parameter discovery on nodes. 
\item Enhancing privacy-preserving measures by leveraging surplus noise. 
\item Incorporating coded gradients in local training processes.
\item The encoding of gradients is easily established, any singular matrix can be used for encoding.
\item Utilization of a gradient descent and a gradient ascent step in the learning process.
\item We infer that proper encoding of data (gradients) can considerably enhance the learning process (convergence rate).
\end{itemize}

\indent The remainder of the paper is organized as follows. In Section~2, we present the problem setup. In Section~3 we formulate our proposed algorithm FLUE with its different forms, implementations and privacy mitigating features and present the method for forming the encoding and decoding matrices. Consequently in Section~4, we discuss the convergence analysis of the algorithm. In Section~5, we present the convergence rate. We complement our work in Section~6 with simulation results. Finally, Section~7 concludes our paper. In Appendix~A, we present the paper with more details. In Appendix~B, we state our main fundamental theorem for the validity of the algorithm convergence for two of FLUE variants in static networks. In Appendix~C we analyze  the convergence of our algorithm in static networks in more details. In Appendix~D, we prove the convergence of FLUE in time-varying networks. Afterwards, we find the convergence rate in Appendix~E. In Appendix~F, we list preliminary lemmas and postulates.

Conventional mathematical nomenclature is used in this paper.


\section{ Problem Setup }

\begin{assumption}\label{Assump1}
We consider the optimization problem
\begin{equation}\label{minimization_eqn}
\min_{x \in \mathbf{R}^{N}} f({\bf x}) = \frac{1}{\bar{n}}\sum_{i=1}^{\bar{n}}f_{i}({\bf x})=\mathbb{E}_{d_{i} \in \mathcal{D}}[f_{i}({\bf x})].
\end{equation}
where $f({\bf x})$ is the global overall function to be minimized, $f_{i}({\bf x})$ are local functions related to the used partition, and $ \bar{n} $ as the total number of data points in the distribution $ \mathcal{D} $ and $d_{i} $ correspond to data point $ i $.
We list the following assumptions essential for the applicability of our algorithm. \\



\end{assumption}

\vspace{-0.8cm}

\section{ Main Algorithm }

\subsection{FLUE: Initialization}

In the federated setting, the server initially organizes clients into clusters, considering factors like data type, heterogeneity, complexity, and estimated stragglers. Each cluster, denoted as $ c \in {1,2,\ldots,m} $, comprises $ N_{c} $ clients. The server creates a fixed gradient encoding matrix $ {\bf B}_{c} \in \mathbb{R}^{n_{c} \times p_{c} } $ for each cluster, where $ p_{c} $ represents the partitions on clients within the cluster. This matrix is singular and sent to all clients in the cluster. FLUE is updated on nodes, aligning iterations with rows of $ {\bf B}_{c} $ (i.e., for each $ 2n $ iterations $ {\bf B}_{c} $ repeats itself two times, one for the gradient descent updating equation and one for the gradient ascent updating equation). Gradient coding is applied to data partitions, using $ {\bf B}_{c} $ for encoding. The proposed (Federated Learning with Un-Encrypted model parameters) FLUE algorithm is then employed to accomplish the learning task without encrypting model parameters.

\subsection{FLUE: General Form}

For the initial $n$ iterations $j$ modulo $2n$, each of the connected $n$ clients updates its model weights ${\bf x}^+_i(2nk+j)$ using the coded data-formed gradient with matrix $ {\bf B} = {\bf B}_{c} $ and the coded gradient descent. This involves utilizing the surplus variable ${\bf y}^+_i(2nk+j)$ to form its proxy weight ${\bf \bar{x}}^+_i(2n(k+1)+j)$. Subsequently, each client updates its surplus variable, incorporating the previously sent proxy weights from the server and its prior model weight and surplus variables from previous iterations. The client then transmits its updated proxy weight ${\bf \bar{x}}^+_i(2n(k+1)+j)$ to the server, which aggregates all clients' proxy weights, forming the server proxy update ${\bf \bar{x}}(2n(k+1)+j)$. The server then dispatches its updated proxy weights to the chosen $n$ clients, and each of these clients finds its model weights ${\bf x}^+_i(2n(k+1)+j)$ using the server-sent proxy weights ${\bf \bar{x}}(2n(k+1)+j)$. Similarly, for the last $n$ iterations modulo $2n$, the process remains consistent as clients update their model weights ${\bf x}^-_i(2nk+j)$ using the coded data-formed gradient, coded gradient ascent, and surplus variable ${\bf y}^-_i(2nk+j)$. The clients then form their proxy weight ${\bf \bar{x}}^-_i(2n(k+1)+j)$ and continue the iterative process until convergence.

FLUE performs the following updating iterations at each node $i$ for $i\in\{1,2,\ldots,n\}$. Please note that $\Gamma_{i}= \Gamma_{i}(k) $ is the fixed support of the row of $ {\bf A}$ identified with iteration $ i $:
{\scriptsize
\begin{align}\label{updating_eqn1}
\begin{split}
{\bf \bar{x}}_{i}^{+} & (2n(k+1)+j)= \bar{\bf A}^{i}_{jj}(k){\bf x}_{i}^{+}(2nk+j) \\ & - \epsilon \sum_{s}[{\bf D}^{i}_{+}(k)]_{js} {\bf y}_{i}^{+}(2nk+s) - \alpha_{k} {\bf B}_{ji} \nabla{f}_{i}({\bf x}_{i}^{+}(2nk+j)) \\
{\bf y}_{i}^{+} & (2n(k+1)+j)=  {\bf x}_{i}^{+}(2nk+j) - {\bf \bar{x}}(2nk+j) -\epsilon {\bf y}_{i}^{+}(2nk+j) \\ & + \sum_{s}[{\bf D}^{i}_{+}(k)]_{js} {\bf y}_{i}^{+}(2nk+s).
\end{split}
\end{align}}
\vspace{-0.4cm}
{\scriptsize
\begin{align}\label{updating_eqn1a}
\begin{split}
{\bf \bar{x}}(2n(k+1)+j)&= \frac{1}{N}\sum_{i =1}^{N}{\bf \bar{x}}_{i}^{+}(2n(k+1)+j)
\end{split}
\end{align}}
\vspace{-0.5cm}
{\scriptsize
\begin{align}\label{updating_eqn1b}
\begin{split}
{\bf x}^{+}_{i}(2n(k+1)+j)&= {\bf \bar{x}}(2n(k+1)+j) \\ & +\sum_{s \in \Gamma_{j} \backslash \{j\} }\bar{\bf A}^{i}_{js}(k){\bf x}_{i}(2nk+s) + \epsilon {\bf y}_{i}^{+}(2nk+j) 
\end{split}
\end{align}}
\vspace{-0.8cm}
{\scriptsize
\begin{align}\label{updating_eqn2}
\begin{split}
{\bf \bar{x}}_{i}^{-} & (2n(k+1)+j)=  \bar{\bf A}^{i}_{(j-n),(j-n)}(k){\bf x}_{i}^{+}(2nk+j) \\ & - \epsilon \sum_{s}[{\bf D}^{i}_{-}(k)]_{js} {\bf y}_{i}^{+}(2nk+s)  + \alpha_{k} {\bf B}_{(j-n),i} \nabla{f}_{i}({\bf x}_{i}^{-}(2nk+j)) \\
{\bf y}_{i}^{-} & (2n(k+1)+j)=  {\bf x}_{i}^{-}(2nk+j) - {\bf \bar{x}}(2nk+j) - \epsilon {\bf y}_{i}^{-}(2nk+j) \\ & + \sum_{s}[{\bf D}^{i}_{-}(k)]_{js} {\bf y}_{i}^{-}(2nk+s).
\end{split}
\end{align}}
\vspace{-0.5cm}
{\scriptsize
\begin{align}\label{updating_eqn2a}
\begin{split}
{\bf \bar{x}}(2n(k+1)+j)&= \frac{1}{N}\sum_{i =1}^{N}{\bf \bar{x}}_{i}^{-}(2n(k+1)+j)
\end{split}
\end{align}}
\vspace{-0.5cm}
{\scriptsize
\begin{align}\label{updating_eqn2b}
\begin{split}
{\bf x}^{-}_{i}(2n(k+1)+j)&= {\bf \bar{x}}(2n(k+1)+j) \\ & +\sum_{s \in \Gamma_{j} \backslash \{j\} }\bar{\bf A}^{i}_{(j-n),s}(k){\bf x}_{i}(2nk+s) + \epsilon {\bf y}_{i}^{-}(2nk+j) 
\end{split}
\end{align}}

Each node $ i \in V $ maintains four vectors: two estimates ${\bf x}_{i}^{+}(2nk+j)$, $ {\bf x}_{i}^{-}(2nk+j)$  for $ 1 \le j \leq n $ and for $ n+1 \le j \le 2n $, respectively. And two surpluses $ {\bf y}_{i}^{+}(2nk+j)$ and ${\bf y}_{i}^{-}(2nk+j)$  for $ 1 \le j \leq n $ and for $ n+1 \le j \le 2n $, respectively.  All in $ \mathbb{R}^{N}$, where $ 2nk+j $ is the discrete time iteration. We use $ \hat{\bf x}_{i}(2nk+j) $ to mean either ${\bf x}_{i}^{+}(2nk+j)$ and $ {\bf x}_{i}^{-}(2nk+j-n)$ for $ 1 \leq j \leq n $ and $ n+1 \leq j \leq 2n $, respectively.

\begin{remark}
    The above general form is for FLUE with $ \bar{\bf A}^{l} $ replicated two times every $ 2 n $ iterations. 
\end{remark}
We refer to Appendix~A for different forms and variants of FLUE algorithm.

\subsection{Forming the Matrix $  \bar{\bf A}^{l} $ and $ {\bf D}^{l} $ at each node $ l $ }

The server administers the algorithm for each cluster, employing three key parameters:

\begin{itemize}
\item $ n_{c} $ (the number of iterations per cycle is $ 2 n_{c} $, WLOG $ n = n_{c} $),
\item $ N_{c} $ (the set of nodes in the cluster engaged in local learning, with whom the server communicates),
\item At each client $ l $, a partitions set $ \mathcal{P}_{l} $ is designated, outlining the assignment of datapoints to partitions (that correspond to the same scaling factor for the encoding matrix $ {\bf B}_{c} $). The set $ \mathcal{P}_{l} $ consists of $ p_{l} = \| \mathcal{P}_{l} \| $ partitions.
\end{itemize}
Once these parameters are specified, the server initializes by transmitting the encoding matrix $ {\bf B}_{c} $ to cluster $ c $. This matrix, $ {\bf B}_{c} $, is an $ n_{c} \times p_{c} $ matrix, where $ p_{c} $ denotes the total number of partitions across all communicating nodes $ N_{c} $, that is, $ p_{c} = \sum_{l=1}^{N_{c}} p_{l} $.

Once the server specifies the parameters, it initializes by transmitting the encoding matrix ${\bf B}_{c}$ to cluster $c$. This matrix, ${\bf B}={\bf B}_{c}$, is an $n_{c} \times p_{c}$ matrix, where $p_{c}$ is the total number of partitions across all communicating nodes $N_{c}$.

At each client $l$, the data points associated with partition $j$ are encoded using the scaling factor ${\bf B}_{i, \sum_{s=1}^{l-1}p_{s}+j}$ for each iteration $t$ where $t \mod n = i$. The server requires the matrix ${\bf B}$ to be singular, ensuring the decoding matrix ${\bf A}$ satisfies ${\bf A} {\bf B} = {\bf C}$, where ${\bf C} = \mathbf{1}_{n_{c} \times p_{c}}$. The pseudoinverse of ${\bf B}$ allows nodes flexibility in choosing matrices $\bar{\bf A}^{l}$, preventing a fixed deterministic $\bar{\bf A}^{l}$ that might compromise security.

To achieve security, it is crucial that ${\bf B}$ is a singular matrix without zeros and that the coefficients of ${\bf B}$ are not equal to 1. This ensures all nodes' datapoints undergo gradient descent steps during training, eliminating the reliance on proxy models being mere scaled versions. Further precautions involve scaling and modifying ${\bf B}$ appropriately before transmission from the server. Nodes can then determine their ${\bf A}$ and $\bar{\bf A}^{l}$ based on the provided value of ${\bf B}$.
\\
Each $ 2 n \times 1 $ row of $ \bar{\bf A}^{l} $ are chosen from rows of $ {\bf A} $ normalized by their respective $ l_{1} $-norm.
We form a row $ \bar{r} $ of $ {\bf A}^{l} $ by first choosing any row $ r $ of $ {\bf A} $ where we keep every positive coefficient $ [{\bf A}]_{r,j} $ in its same $ j $ position and normalize by the $ l_{1} $-norm of row $ r $  (i.e., $ [{\bf A}^{l}]_{\bar{r}, j} = \frac{[{\bf A}]_{r, j}}{ \|{\bf A}_{r, :} \|_{1}} $ (corresponding to a gradient descent step of the coded gradient)). And we move every negative coefficient $ [{\bf A}]_{r,j} $ in the $ j + n $ position, take its absolute value and normalize by the $ l_{1} $-norm of row $ r $ (i.e., $ [{\bf A}^{l}]_{\bar{r}, j+n} = -\frac{[{\bf A}]_{r, j}}{ \|{\bf A}_{r, :} \|_{1}} $ (corresponding to a gradient ascent step of the coded gradient)). All other coefficients of row $ \bar{r} $ of $ {\bf A}^{l} $ are set to zero.
In this process we need to ensure that the formed matrix 
$ \bar{\bar{\bf A}}^{l} = \begin{pmatrix} \bar{\bf A}^{l}_{+} \\ \bar{\bf A}^{l}_{-} \end{pmatrix} $ \footnote{We denote by $ \bar{\bf A}^{l}_{+} $ and $ \bar{\bf A}^{l}_{-} $ to be the matrix $ \bar{\bf A}^{l} $ for the first $ n $ iterations and the last $ n $ iterations of the period $ 2 n $, respectively (i.e., FLUE has two forms $ \bar{\bf A}^{l} 
$ replicated twice in $ 2 n $ iterations or not).} is Stochastic Indecomposable Aperiodic matrix (SIA).


Coordination of certain coefficients of matrix $ {\bf A}^{l} $ across the nodes during the aggregation step requires specific conditions for FLUE forms, ensuring model privacy through adjustments to $\bar{\bf A}^{l}$ and the use of SIA matrices for the matrix $\hat{\bf A}$. In the varying form of FLUE we require that $ \hat{\bf A} $ vary infinitely but from a finite set of different $ \hat{\bf A} $ for the convergence of the algorithm (i.e., through utilizing SIA Theorem~1 (2) in \citep{Xia2015}, product of finite sets of SIA matrices is SIA). Ensuring $\hat{\bf A}(k)$ is from a finite set is achieved by requiring each node to choose the random $\bar{\bf A}^{l}$ from a finite set, facilitating repeated use to expedite convergence. The only coordination needed during the aggregation step demands constancy in $ [\bar{\bf A}^{l}_{+}(k)]_{i,i} $ and $ \bar{\bf A}^{l}_{-}(k)_{i,i+n} $ for $ 1 \leq i \leq n $, $ 1 \leq l \leq N_{c} $, and each $ k $ in FLUE form where $ \bar{\bf A}^{l}(k)_{+} \neq \bar{\bf A}^{l}(k)_{-} $. Alternatively, $ [\bar{\bf A}^{l}(k)]_{i,i} $ should be constant for $ 1 \leq i \leq n $, $ 1 \leq l \leq N_{c} $, and each $ k $ in FLUE forms where $ \bar{\bf A}^{l}(k) = \bar{\bf A}^{l}(k)_{+} = \bar{\bf A}^{l}(k)_{-} $. Furthermore, $\hat{\bf A}$ must be an SIA matrix, ensured by having at least one positive column in $\tilde{\bf A}^{(+)} = \big(({\bf A_{i,j}})_+\big)_{1\le i,j\le n}$ \footnote{Where $ ({\bf A}_{ij})_{+} = {\bf A}_{ij} $ if $ {\bf A}_{ij} \geq 0 $ and $ 0 $ otherwise.} for FLUE forms where $ \bar{\bf A}^{l}(k) = \bar{\bf A}^{l}(k)_{+} = \bar{\bf A}^{l}(k)_{-} $. Adjustments in $\bar{\bf A}^{l}$ are required to ensure unanimity in $ [\bar{\bf A}^{l}_{+}(k)]_{i,i} = \gamma_{i} $ or $ [\bar{\bf A}^{l}_{-}(k)]_{i,i+n} = \gamma_{i+n} $ for the first FLUE form and $ [\bar{\bf A}^{l}(k)]_{i,i} = \gamma{i} $ for the latter, where $ 0 < \gamma_{i} < 1 $ or $ 0 < \gamma_{i+n} < 1 $ is crucial for obtaining a stochastic matrix $ \bar{\bar{\bf A}}^{l} $ encoding weights to ensure model privacy. The scaling of rows with weights $ w_{i} $ after normalization avoids further scaling so that $ \bar{\bf A}^{l}_{ii} = \gamma < 1 $ according to specified conditions, already considered in the final matrix $ \bar{\bar{\bf A}}^{l} $.

Two cases of computing $ {\bf D}^{l} $ for the FLUE algorithm exist, both requiring no coordination among nodes and no encryption. The first case involves a fixed $ {\bf D}^{l} $ every $ 2n $ iterations at each node $ l $. The second case involves a variable (random) $ {\bf D}^{l} $ every $ 2n $ iterations at each node $ l $, chosen from a finite or infinite set. For both cases $ {\bf D}^{l} $ only need to be a column stochastic matrix and can be repeated twice every $ 2 n $ iterations or not (i.e., $ {\bf D}^{l} = {\bf D}^{l}_{+} = {\bf D}^{l}_{-} $ or $ {\bf D}^{l}_{+} \neq {\bf D}^{l}_{-} $). We denote by $ {\bf D}^{l}_{+} $ and $ {\bf D}^{l}_{-} $ to be the matrix $ {\bf D}^{l} $ for the first $ n $ iterations and the last $ n $ iterations of the period $ 2 n $, respectively. Privacy is maintained through the use of $ \bar{\bf A}^{l} $ and surpluses.

In the process of encoding data points using the matrix $ {\bf B }$ to derive coded gradients, there are two approaches. The first involves calculating the uncoded gradient on the original data points and subsequently encoding that gradient using the coefficients of $ {\bf B} $. Alternatively, in scenarios such as linear regression problems or supervised learning in neural networks, we can scale the labels and feature parameters of the data points by the respective coefficients of $ {\bf B} $ corresponding to the node partition and the iteration modulo $ n $. Then the gradients of theses adjusted (weighted) datapoints are the corresponding coded gradients $ \nabla{g_{i}} $.

\vspace{-0.35cm}

\subsection{Forming Matrices $ {\bf A} $, $ \bar{\bar{\bf A}}^{l} $ and $ {\bf B} $}

Any singular matrix $ {\bf B} $ with $ rank({\bf B}) \leq \min(n_{c},p_{c}) - 1 $ will work as an encoding matrix. Then for the matrix $ \bar{\bar{\bf A}}^{l} $ whether for the fixed or time-varying case we identify two scenarios in which $ \hat{\bf A} $ must be an SIA matrix. And to ensure that we must have first $ \bar{\bf A}_{ii}^{l} = \gamma_{i} $ where $ 1 \leq l \leq N_{c} $ and $ 1 \leq i \leq n $ for FLUE where $ \bar{\bf A}^{l} = \bar{\bf A}^{l}_{+} = \bar{\bf A}^{l}_{-} $. $ \bar{\bar{\bf A}}_{ii}^{l} = \gamma_{i} $ where $ 1 \leq l \leq N_{c} $ and $ 1 \leq i \leq 2 n $ for FLUE where $  \bar{\bf A}^{l}_{+} \neq \bar{\bf A}^{l}_{-} $. And an easy way to ensure that $ \hat{\bf A} $ is SIA for each case is to have at least one positive column for at least one matrix $ \bar{\bf A}^{l} $ for the first FLUE form. That is, for the first form we ensure that $ \tilde {\bf A}^{(+)} $ contains one positive column of entries. And $ \bar{\bar{\bf A}}^{l} $ has one positive column of entries for the latter FLUE form.
To ensure the first conditions we need first to find the null space of matrix $ {\bf B}^{T} $. Then have $ \omega_{i} {\bf A}_{r_{i}}  = \gamma_{i} $ for each specified $ i $ according to the form of FLUE. And $ {\bf A}_{r_{i}} $ is any vector where $ {\bf A}_{r_{i}}{\bf B} = \mathbf{1}_{n_{p}} $ and $ \omega_{i} $ as defined earlier to be the inverse of the $ l_{1} $-norm of $ {\bf A}_{r_{i}} $.
Due to the limited space we refer you to Appendix~A for a detailed exposition of the process.

\vspace{-0.5cm}

\vspace{-0.1cm}

\section{Convergence Analysis}

Convergence analysis is deferred to Appendix~C and D.
We provide a convergence analysis in the distributed optimization setting that can be easily adjusted to the federated optimization setting based on sampling criteria for calculating batch gradients, the adequacy of approximating the learning function of a neural network, for example, by a convex function and viewing the learning process as a stochastic convergence process.

\vspace{-0.7cm}

\section{Convergence Rate}

Convergence rate is deferred to Appendix~E. We emphasize on the dependency of the learning rate (convergence rate) in neural networks on the coding of data (gradients). Thus, certain coding schemes show better learning rates than ordinary learning (with uncoded gradients) depending on the matching between the data structures and the used scheme. 

\section{ Simulation Results}\label{results}

In the simulation section, we delve into the convergence rate analysis of FLUE for decentralized optimization on a network of $N_c$ nodes and its federated learning form. The optimization problem under consideration is an unconstrained convex one, represented as:
{\footnotesize \begin{equation}\label{objective.simulation}
 \arg\min_{{\bf x}\in \mathbb{R}^{N}} \|  {\bf F}{\bf x}- {\bf y}  \|_{2}^{2},\end{equation}}
\vspace{-0.1cm}
where ${\bf F}$ is a random matrix of size $M\times N$, and ${\bf y}={\bf F} {\bf x}_o$, with ${\bf x}_o$ sampled from a uniform bounded random distribution. The solution $ {\bf x} ^{*}$ represents the least squares solution of the overdetermined system ${\bf y} = {\bf F}{\bf x}_o, {\bf x}_o\in \mathbb{R}^N$.

In our simulations, we consider a network with one cluster ($m=1$) composed of $N_c$ nodes, each partitioned into $\bar{n}_{l}$ data points. The matrices are of the same size, and we employ absolute error (AE)
{\small $ {\rm AE}:= {\bf x}_{1\le i\le N_c} \frac{\|{\bf x}_i(k)-{\bf x}_o\|_2}{\|{\bf x}_0\|_2} $} and consensus error (CE) {\small $ {\rm CE}:= {\bf x}_{1\le i\le N_c} \frac{\|{\bf x}_i(k)-\bar{\bf x}(k)\|_2}{\|{\bf x}_o\|_2} $} to measure the performance of FLUE. The convergence rate is analyzed for different FLUE variants, coding matrices, and step sizes, comparing FLUE with conventional distributed gradient descent (DGD) for solving the convex optimization problem.

It is worth mentioning that in the following simulations we employed a fixed coding matrix $ {\bf B} $ except the fourth simulation of Figure 6 (see Appendix~A) where we used a random coding matrix $ {\bf B} $.
In Figure 1, we have simulated the convergence of FLUE version with $ {\bf O} = {\bf Q}_{\epsilon, {\bf T}} $ for the above defined overdetermined system linear regression problem with step size $ \alpha_{k} = \frac{1}{(k+100)^{0.75}} $. While in Figure 5 (see Appendix~A), we show the convergence rate for FLUE version with $ {\bf O} = {\bf Q}_{\epsilon, {\bf T}}$. For both we use fixed coding matrices $ \bar{\bf A}^{l} $ (i.e., repeating on each node $ l $ every $ 2 n $ iterations) with the number of nodes $ N_c = 5 $.

Meanwhile, in Figure 4 (see Appendix~A) we show the behavior of FLUE with version $ {\bf O} = {\bf Q}_{\epsilon, {\bf T}} $ for the random coding matrices $ \bar{\bf A}^{l} $ with the number of nodes $ N_c = 5 $. While we address the behavior of FLUE with version $ {\bf O} = {\bf Q}_{\epsilon, {\bf T}} $ for $ N_c = 7 $-node network using fixed coding matrices $ \bar{\bf A}^{l} $ in Figure 5 and using random coding matrices $ \bar{\bf A}^{l} $ in Figure 6 (see Appendix~A). 
Both of the above mentioned simulations are with step size $ \alpha_{k} = \frac{1}{(k+100)^{0.75}} $.

For an explicit discussion of the simulation we refer to Appendix~A.

\begin{figure}[h]\label{fig_1}
\centering
\includegraphics[trim=0 0 0 0 cm , clip, width=10cm]{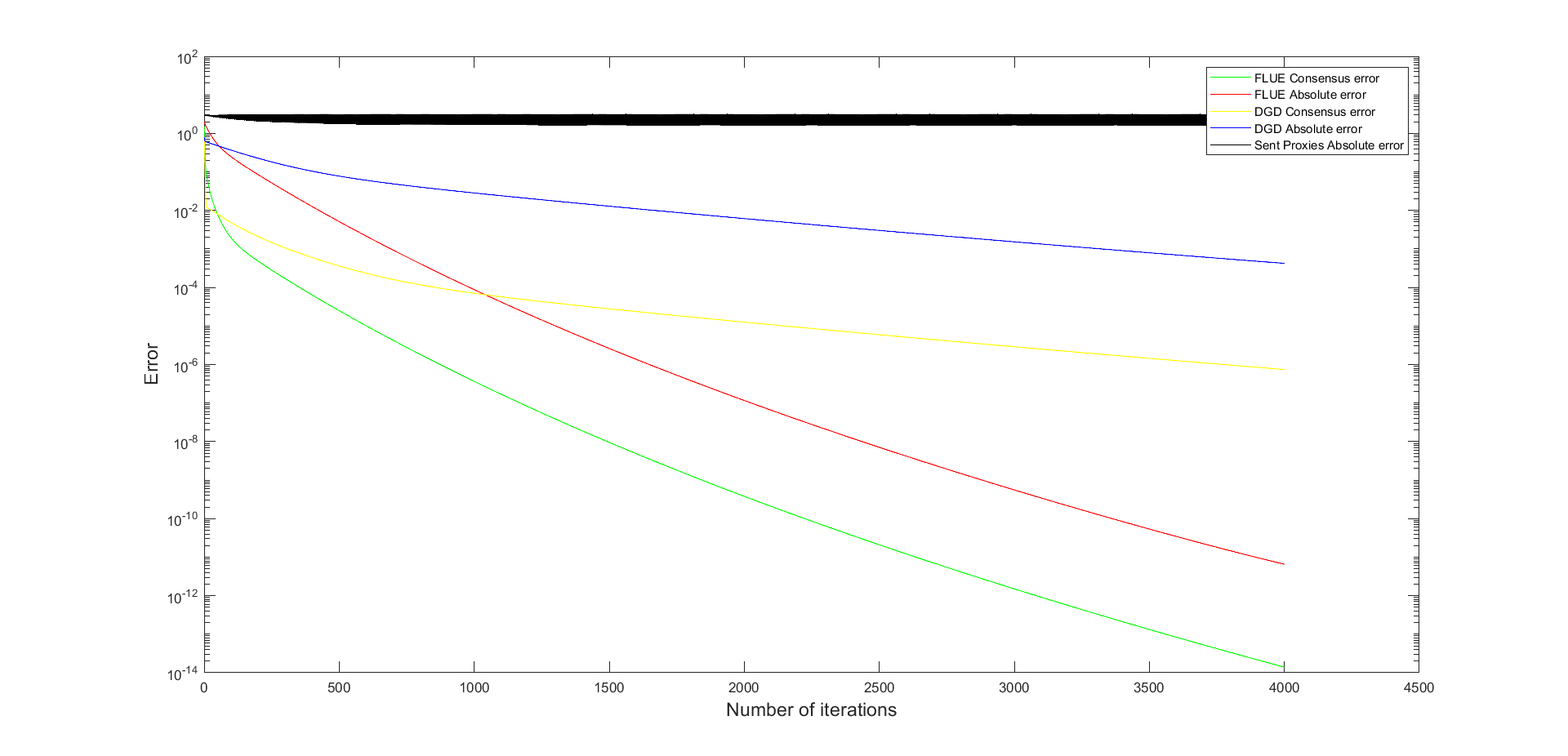}
\caption{\footnotesize Absolute and consensus errors vs iteration of FLUE and DGD with fixed coding matrices $ \bar{\bf A}^{l} $ using FLUE \eqref{fixedQepsT} on a $ 5 $-node network}
\end{figure}


\begin{figure}[h]\label{fig_2}
\centering
\includegraphics[trim=0 0 0 0 cm , clip, width=10cm]{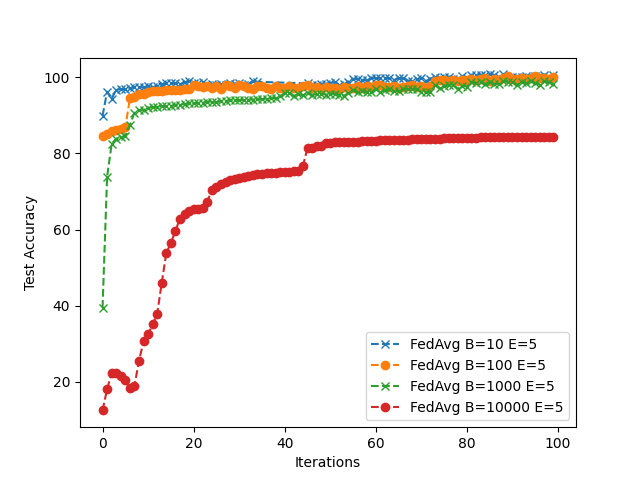}
\caption{\footnotesize Test Accuracy per iterations for FedAvg algorithm with $ E = 5 $ epochs and batchsize $ B = \{10, 100, 1000, 1000 \} $ on MNIST dataset.}
\end{figure}


\begin{figure}[h]\label{fig_3}
\centering
\includegraphics[trim=0 0 0 0 cm , clip, width=10cm]{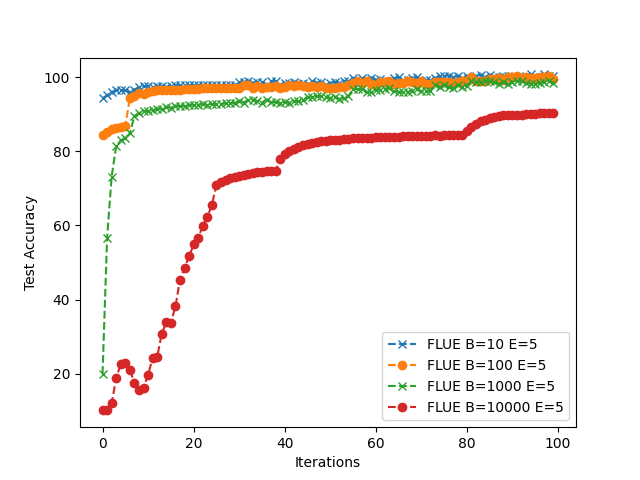}
\caption{\footnotesize Test Accuracy per iterations for FLUE algorithm with $ E = 5 $ epochs and batchsize $ B = \{10, 100, 1000, 1000 \} $ on MNIST dataset.}
\end{figure}

\vspace{-0.1cm}

We show a comparison between FedAvg algorithm and FLUE algorithm variant (with no surpluses) compatible for neural networks federated learning (see Algorithm~1 in Appendix~A). We partition the MNIST dataset into $ 100 $ clients each containing $ 600 $ examples. Each client utilizes a neural network with $ 2 $ hidden layers containing $ 200 $ nodes each and using ReLu activation function.  As was inferred from the convergence rate analysis FLUE can match or outperform ordinary uncoded learning depending on the coding scheme compatibility with the learning dataset.

\nocite{*}

\vspace{-0.3cm}

\section{Conclusion}\label{conlusion}

In this paper, we present FLUE, a federated learning algorithm that utilizes coded local gradients and exchanges coded combinations as model parameter proxies, ensuring data privacy without encryption. We have developed both general and special variants, showcasing convergence and improved learning rates. Our encryption-free implementations yield promising results for federated optimization and machine learning, with variations involving replicated and non-replicated matrices, as well as general and special forms. Furthermore, our approach demonstrates the efficacy of encoding data in accelerating the learning convergence rate. Coordination across all nodes for complete decoding matrices is unnecessary; instead, we limit coordination to specific entries for precise aggregation, maintaining privacy. While this approach doesn't compromise privacy, our ongoing security efforts aim to explore methods ensuring zero coordination among nodes in future research. Moreover, employing a shared encoding matrix does not compromise security. Future research directions could involve exploring algorithms incorporating random encoding matrices, drawing upon methods derived from the convergence of infinite products of random matrices.





\bibliography{ref}

\appendix

\section{Appendix A: Detailed View}

\section{A-1: Introduction}

\subsection{Overview}

Federated Learning empowers a diverse array of devices, such as mobile phones and computing devices, to collectively train a shared predictive model while preserving all training data on the local device. This approach decouples the ability to perform machine learning from the necessity of storing data in the cloud. Within the framework of federated learning (FL), clients cooperate in training a global model without exposing their raw data; instead, they share computed gradients or model parameters. While the local information can still be derived from the outputs transmitted to the parameter server, privacy concerns have led clients to introduce artificial noise or encryption to their local updates, which may include gradients or learned models, to protect the global model training process.

However, the effectiveness of using gradients, even with added noise, has proven limited, as these gradients can still be reverse-engineered, ultimately revealing the model and, consequently, the private data. Recent research has shifted towards using model parameters in training without exchanging local data or computed local gradients. To secure the trained models and mitigate the risk of inferring local data, encryption is employed. Additionally, the introduction of noise to the trained models further enhances privacy.

In this paper, we introduce a federated learning algorithm that harnesses coded local gradients during the local learning phase and exchanges coded combinations serving as proxies for the learned models' parameters. This approach ensures data privacy without relying on encryption. Furthermore, we inject surplus noise into each of these proxy variables to establish a more robust privacy framework. We develop two algorithm variants: a general form and a special form that does not depend on surplus noise. We demonstrate the convergence of these algorithms and establish learning convergence rates that can be improved based on the chosen coding scheme and its alignment with the characteristics and structure of the raw data.

We present two encryption-free implementations using fixed and random coding matrices and provide promising simulation results for the algorithms, viewed from both a federated optimization and a federated machine learning perspective.

\subsection{Introduction}

The widespread adoption of Internet of Things (IoT) technologies has led to the generation of massive amounts of user data. In an effort to protect data privacy, decentralized machine learning methods such as Federated Learning (FL) [have been proposed as an alternative to sharing individuals' raw data. During the training phase, each client periodically retrieves the global model from a parameter server and updates their local model with their own data. Instead of transmitting raw data, only encrypted model parameters are transferred, preserving data privacy. Additionally, the introduction of noise to the trained models can further enhance privacy.

In this paper, we introduce Federated Learning with Unencrypted Model Weights (FLUE), a Federated Learning algorithm that incorporates coded local gradients into its local learning process. We exchange coded combinations' proxies of the learned model parameters to ensure data privacy, all without the need for encryption. Furthermore, we inject surplus noise into these proxy variables to establish stronger privacy safeguards. We present two algorithm variants: a general form and a special form that does not utilize noise surpluses. We demonstrate the convergence of these algorithms and establish learning convergence rates that can be further optimized based on the chosen coding scheme and its compatibility with the characteristics and structure of the raw data.

We provide two encryption-free implementations using fixed and random coding matrices and showcase promising simulation results from both a federated optimization perspective and a federated machine learning perspective.

However, it is important to note that sharing model updates or gradients can potentially compromise clients' private information. Analyzing the differences in training parameters uploaded by clients can reveal sensitive data, as illustrated by the model inverse attack \citep{geiping2020inverting, zhao2020idlg} which allows for training data reconstruction by matching model gradients and optimizing randomly initialized inputs. These attacks represent significant threats to FL security and privacy.

To address the risk of privacy breaches, each client can add artificial noise \citep{abadi2016deep,wei2020federated,yang2022accuracy} to the transmitted parameters. Nevertheless, this approach can result in a noticeable loss of training accuracy while making it more challenging for attackers to reverse-engineer the original local data. Our algorithm mitigates this challenge by employing coded gradients, making it more difficult for eavesdroppers to infer the original local data through gradients. Additionally, the systematic introduction of noise surpluses in the learning update procedure does not compromise learning accuracy. Furthermore, the variable nature of the noise between iterations, while maintaining learning privacy, adds an extra layer of complexity for eavesdroppers attempting to extract data from the exchanged variables.

\subsection{Key Contributions}

The key contributions of FLUE are:

\begin{itemize}
\item Implementing privacy-preserving Federated Learning without relying on encrypted model parameters or exchanging local gradients. 
\item Employing proxy model parameters between server and nodes to facilitate data training and model parameter discovery on nodes. 
\item Enhancing privacy-preserving measures by leveraging surplus noise. 
\item Incorporating coded gradients in local training processes.
\item The encoding of gradients is easily established, any singular matrix can be used for encoding.
\item utilization of a gradient descent and a gradient ascent step in the learning process.
\item we infer that proper encoding of data (gradients) can considerably enhance the learning process (convergence rate).
\end{itemize}

\indent The remainder of the paper is organized as follows. In Appendix A-2, we present the problem setup, background material and needed assumptions. In Appendix A-3, we present the method for forming the encoding and decoding matrices. In Appendix A-4 we formulate our proposed algorithm FLUE with its different forms, implementations and privacy mitigating features. We complement our work in Appendix A-5 with simulation results. Finally, Appendix A-6 concludes our paper. In Appendix~B, we state our main fundamental theorem for the validity of the algorithm convergence for two of FLUE variants in static networks. In Appendix~C we analyze  the convergence of our algorithm in static networks in more details. In Appendix~D, we prove the convergence of FLUE in time-varying networks. Afterwards, we find the convergence rate in Appendix~E. In Appendix~F, we list preliminary lemmas and postulates.

Conventional mathematical nomenclature is used in this paper.


\subsection{ A-2: Problem Setup, Background Material and Assumptions }

\begin{assumption}\label{Assump1}
We consider the optimization problem
\begin{equation}\label{minimization_eqn}
\min_{x \in \mathbf{R}^{N}} f({\bf x}) = \frac{1}{\bar{n}}\sum_{i=1}^{\bar{n}}f_{i}({\bf x})=\mathbb{E}_{d_{i} \in \mathcal{D}}[f_{i}({\bf x})].
\end{equation}
where $f({\bf x})$ is the global overall function to be minimized, $f_{i}({\bf x})$ are local functions related to the used partition, and $\bar{n}$ as the total number of data points in the distribution $ \mathcal{D} $ and $d_{i} $ correspond to data point $ i $.
We list the following assumptions essential for the applicability of our algorithm. \\
Since minimum of $ f({\bf x}) $ is the same as minimum of a scaled $ f({\bf x}) $, we consider in the proof the equivalent optimization problem
\begin{equation}\label{minimization_eqn}
\min_{x \in \mathbf{R}^{N}} f({\bf x}) = \sum_{i=1}^{\bar{n}}f_{i}({\bf x})= \bar{n} \mathbb{E}_{d_{i} \in \mathcal{D}}[f_{i}({\bf x})].
\end{equation}

\begin{itemize}
\item \textbf{ (a) } The global function  $\mathit{f} : \mathbb{R} ^{N} \rightarrow \mathbb{R}  $ is convex.
\item \textbf{ (b) } The solution set of \eqref{minimization_eqn} and the optimal value exist.
$ {\bf x}^{*} \in {\bf x}^{*}= \{ x | f({\bf x})= \min_{x^{'}}f({\bf x}^{'}) \},\\ f^{*}=f({\bf x}^{*})=\min{f({\bf x})} $..
\item \textbf{ (c) } The gradients $\nabla{f_{i}}({\bf x})$, where $i \in V$ are bounded over the $\mathbb{R}^{N}$, i.e. there exists a constant $ F $ such that $ \| nabla{f_{i}}({\bf x}) \| \leq F $ for all $ x \in \mathbb{R}^{N}$ and all $i \in V $. That is $ \| \nabla{f} \| \leq n F $ and  $ \| g_{i}({\bf x}) \| \leq G = \sqrt{n} \| {\bf B} \|_{2,\infty} F $ for all $ x \in \mathbb{R}^{N}$ and all $i \in V $. 
\end{itemize}

\end{assumption}

\section{ A-3: Forming the Decoding and Encoding Matrices}

\subsubsection{Forming the Matrix $  \bar{\bf A}^{l} $ and $ {\bf D}^{l} $ at each node $ l $ }

The server administers the algorithm for each cluster, employing three key parameters:

\begin{itemize}
\item $ n_{c} $ (the number of iterations per cycle is $ 2 n_{c} $, WLOG $ n = n_{c} $),
\item $ N_{c} $ (the set of nodes in the cluster engaged in local learning, with whom the server communicates),
\item At each client $ l $, a partitions set $ \mathcal{P}_{l} $ is designated, outlining the assignment of datapoints to partitions (that correspond to the same scaling factor for the encoding matrix $ {\bf B}_{c} $). The set $ \mathcal{P}{l} $ consists of $ p_{l} = \| \mathcal{P}_{l} \| $ partitions.
\end{itemize}
Once these parameters are specified, the server initializes by transmitting the encoding matrix $ {\bf B}_{c} $ to cluster $ c $. This matrix, $ {\bf B}={\bf B}_{c} $, is an $ n_{c} \times p_{c} $ matrix, where $ p_{c} $ denotes the total number of partitions across all communicating nodes $ N_{c} $, that is, $ p_{c} = \sum_{l=1}^{N_{c}} p_{l} $.

At each client $l$, the data points associated with partition $j$ are encoded using the scaling factor ${\bf B}_{i, \sum_{s=1}^{l-1}p_{s}+j}$ for each iteration $t$ where $t \mod n = i$.
The server's requirement for designing the matrix $ {\bf B} = {\bf B}_{c} $ is simply that it be a singular matrix. This is because the decoding matrix ${\bf A}$ satisfies ${\bf A} {\bf B} = {\bf C}$, where ${\bf C} = \mathbf{1}_{n_{c} \times p_{c}}$ matrix. Therefore, ${\bf A} = {\bf C} {\bf B}^{-1}$. The use of the pseudoinverse of ${\bf B}$ to find ${\bf A}$ results in an infinite set of matrices ${\bf A}$, allowing nodes the flexibility to choose from various alternatives of rows to form their matrices $\bar{\bf A}^{l}$.
As a consequence, $\bar{\bf A}^{l}$ across different nodes need not be identical and can be selected to be fixed or even random every $2n$ iterations, whether replicated every $n$ iterations or not. Consequently, there is no fixed deterministic $\bar{\bf A}^{l}$ that can be easily deduced by an eavesdropper or the server, ensuring security is not compromised.
Turning our attention back to matrix ${\bf B}$, it is imperative to ensure that ${\bf B}$ is a singular matrix, preferably with no zeros \footnote{Since then the application of $ \bar{\bf A}^{l} $ on coded gradients is equivalent to having $ \bar{\bf A}^{l} $ then $ {\bf B}^{c} $ applied on uncoded gradients. And we know that consensus in distributed algorithms with local original gradients is easily established for stochastic matrices}. This guarantees that all nodes' datapoints undergo gradient descent steps during training, eliminating the reliance on proxy models $\bar{\bf x}_{i}$ being mere scaled versions of the models, which could pose security risks. While this concern is already addressed by the surpluses ${\bf y}_{i}^{+}$ or ${\bf y}_{i}^{-}$ and the factors $\bar{\bf A}^{l}_{ii}$, additional precautions are necessary.
Moreover, the coefficients of ${\bf B}$ should not be equaling $1$. This prevents exact uncoded gradient descent steps, thus thwarting attempts to deduce models through inverse engineering the gradients. This safeguard is further reinforced by the use of scaled proxy models.
It is essential that the coefficients of ${\bf B}$ differ across nodes $l$ and partition sets $\mathcal{P}_{l}$. Uniform coefficients across nodes could compromise security by facilitating the decoding of the same scaled factor applied to all gradients at node datapoints through reverse engineering of coded gradient steps.
In order to achieve these functionalities, it is essential to appropriately scale and modify the matrix ${\bf B}$ before transmitting it from the server. Subsequently, based on the provided value of ${\bf B}$, each node $l$ can determine an ${\bf A}$ and, consequently, $\bar{\bf A}^{l}$.
\\
We denote  $ \bar{\bf A}^{l}_{+} $ and $ \bar{\bf A}^{l}_{-} $ to be the matrix $ \bar{\bf A}^{l} $ for the first $ n $ iterations of the period $ 2 n $ and the last $ n $ iterations of the period $ 2 n $, respectively (i.e., FLUE has two forms $ \bar{\bf A}^{l} 
$ replicated twice in $ 2 n $ iterations or not. Each $ 2 n \times 1 $ row of $ \bar{\bf A}^{l} $ are chosen from rows of $ {\bf A} $ normalized by their respective $ l_{1} $-norm.
We form a row $ \bar{r} $ of $ {\bf A}^{l} $ by first choosing any row $ r $ of $ {\bf A} $ where we keep every positive coefficient $ [{\bf A}]_{r,j} $ in its same $ j $ position and normalize by the $ l_{1} $-norm of row $ r $  (i.e., $ [{\bf A}^{l}]_{\bar{r}, j} = \frac{[{\bf A}]_{r, j}}{ \|{\bf A}_{r, :} \|_{1}} $ (corresponding to a gradient descent step of the coded gradient)). And we move every negative coefficient $ [{\bf A}]_{r,j} $ in the $ j + n $ position, take its absolute value and normalize by the $ l_{1} $-norm of row $ r $ (i.e., $ [{\bf A}^{l}]_{\bar{r}, j+n} = -\frac{[{\bf A}]_{r, j}}{ \|{\bf A}_{r, :} \|_{1}} $ (corresponding to a gradient ascent step of the coded gradient)). All other coefficients of row $ \bar{r} $ of $ {\bf A}^{l} $ are set to zero.
In this process we need to ensure that the formed matrix 
$ \bar{\bar{\bf A}}^{l} = \begin{pmatrix} \bar{\bf A}^{l}_{+} \\ \bar{\bf A}^{l}_{-} \end{pmatrix} $ is Stochastic Indecomposable Aperiodic matrix (SIA).

It's worth noting that due to the fact that ${\bf A}$ is drawn from an infinite set, $\bar{\bf A}^{l}$ also belongs to an infinite set, as mentioned earlier. This implies the existence of different ${\bf A}$'s for different nodes, ensuring that the advantage of $\bar{\bf A}^{l}$ remains exclusive to node $ l $ with no knowledge accessible to other nodes or server.
The matrices $ \bar{\bf A}^{l} $ can either undergo duplication twice in each cycle (every $ 2 n $ iterations) or not. Specifically, $ \bar{\bf A}^{l}_{+} $, representing the first $ n $ iterations of the cycle, and $ \bar{\bf A}^{l}_{-} $, corresponding to the last $ n $ iterations, may or may not be identical.
There is also the option to insist that $ \bar{\bf A}^{l} $ for each node $ l $ remains constant throughout the algorithm, repeating every $ 2 n $ iterations, denoted as $ \bar{\bf A}^{l}(k) = \bar{\bf A}^{l} $ for $ 1 \leq l \leq N_{c} $.
Alternatively, it may be stipulated that $ \bar{\bf A}^{l} $ for each node $ l $ varies over time (randomly) throughout the algorithm, eliminating the need for $ \bar{\bf A}^{l} $ to repeat every $ 2n $ iterations. In other words, $ \bar{\bf A}^{l}(k) $ is not necessarily equal to $ \bar{\bf A}^{l}(k^{‘}) $ for $ k \neq k ^{‘} $.
For the fixed $ \bar{\bf A}^{l} $ per node $ l $, we have the corresponding matrix $ \hat{\bf A}(k)= \hat{\bf A} $ to be utilized in the convergence proof. To achieve this, $ {\bf D}^{l} $ should be a column stochastic matrix which also needs to remain fixed throughout the algorithm, with $ {\bf D}^{l}_{+} $ repeating every $ 2n $ iterations and $ {\bf D}^{l}_{-} $ repeating every $ 2n $ iterations. There is flexibility in whether $ {\bf D}^{l}_{+} = {\bf D}^{l}_{-} $ or not.
For varying $ \bar{\bf A}^{l} $ per node $ l $, we need it to be chosen from a compact set rather than an infinite set. That is, $ \bar{\bf A}^{l}(k) \in \mathcal{A}_{l} $ where $ \mathcal{A}_{l} = \{ \bar{\bf A}^{l}_{1}, \bar{\bf A}^{l}_{2}, \ldots, \bar{\bf A}^{l}_{n} \} $ where $ 1 \leq l \leq N{c} $ is a finite set. Here, we don’t require that $ {\bf D} ^{l} $ be fixed for each corresponding $ \bar{\bf A}^{l} $ since the convergence proof utilizes that assumption of Stochastic Indecomposable Aperiodic (SIA) matrices from a finite set $ \hat{\bf A}(k) $ only. This requirement that $ \hat{\bf A}(k) $ be from a finite set is accomplished by requiring each node to choose the random $ \bar{\bf A}^{l} $ from a finite set. Thus, it will ease the repetition of $ \hat{\bf A} $ infinitely many times to reach  convergence. \footnote{ In the time-varying form of FLUE we require that $ \hat{\bf A} $ vary infinitely but from a finite set of different $ \hat{\bf A} $ for the convergence of the algorithm (i.e., through utilizing SIA Theorem~1 (2) in \citep{Xia2015}. }
Regarding the coordination of nodes in the aggregation step, it is necessary for $ \bar{\bf A}^{l}_{+}(k)_{i,i} $ where $ 1 \leq i \leq n $ and  $ \bar{\bf A}^{l}_{-}(k)_{i,i+n} $ where $ 1 \leq i \leq n $, for $ 1 \leq l \leq N_{c} $ and each $ k $ to remain constant for FLUE of the form where $ \bar{\bf A}^{l}(k)_{+} \neq \bar{\bf A}^{l}(k)_{-} $. Alternatively, $ \bar{\bf A}^{l}(k)_{i,i} $ should be constant for $ 1 \leq i \leq n $, for $ 1 \leq l \leq N_{c} $ and each $ k $ for FLUE of the form where $ \bar{\bf A}^{l}(k) =\bar{\bf A}^{l}(k)_{+} = \bar{\bf A}^{l}(k)_{-} $. Additionally, the matrix $ \hat{\bf A} $ needs to be an SIA matrix.

An easy way to ensure this is to have at least one positive column in $ \tilde{\bf A}^{(+)} $, ensuring that $ \bar{\bf A}^{l} $ becomes an SIA matrix for FLUE of the form where $ \bar{\bf A}^{l}(k) = \bar{\bf A}^{l}(k)_{+} = \bar{\bf A}^{l}(k)_{-} $. And $ \bar{\bar{\bf A}}^{l} = \begin{pmatrix} \bar{\bf A}^{l}(k)_{+} \\ \bar{\bf A}^{l}(k)_{-} \end{pmatrix} $ to have at least one positive column for FLUE of the form where $ \bar{\bf A}^{l}(k)_{+} \neq \bar{\bf A}^{l}(k)_{-} $. Similarly,  ensuring that $ \bar{\bf A}^{l} $ becomes an SIA matrix.

Adjustments are also required for $ \bar{\bf A}^{l} $ so that $ \bar{\bf A}^{l}_{+}(k)_{i,i} = \gamma_{i} $ or $ \bar{\bf A}^{l}_{-}(k)_{i,i+n}= \gamma_{i+n} $ be unanimous for all $ l $ and all $ k $, for $ 1 \leq i \leq n $ for the first form of FLUE. And $ \bar{\bf A}^{l}(k)_{i,i} =\gamma_{i} $ for $ 1 \leq i \leq n $, should be unanimous for $ 1 \leq l \leq N_{c} $ and all $ k $ for the latter form. It is crucial that $ 0 < \gamma_{i} < 1 $ or $ 0 < \gamma_{i+n} < 1 $ to obtain a stochastic matrix $ \bar{\bar{\bf A}}^{l} = \begin{pmatrix} \bar{\bf A}^{l}(k)_{+} \\ \bar{\bf A}^{l}(k)_{-} \end{pmatrix} $ encoding weights that ensure model privacy. Consequently, the rows, which are already scaled with weights $ w_{i} $ after normalization, need not be further scaled so that $ \bar{\bf A}^{l}_{ii} = \gamma < 1 $ according to the specified conditions since this weight scaling is already considered in the final matrix 
$ \bar{\bar{\bf A}}^{l} $.


\footnote{It is important to note that a $ \frac{1}{n} $ scaling for $ \nabla{g}_{i} $ is unnecessary, as demonstrated in the proof in Appendix, where normalizing weights achieves the required convergence for $ f = \sum_{i=1}^{n}f_{i} $. Therefore, only $ n_{i} $ is needed for $ f = \sum_{i=1}^{n}\frac{n_{i}}{n}F_{i} $ without the $ \frac{1}{n} $ term.}
In the process of encoding data points using the matrix $ {\bf B }$ to derive coded gradients, there are two approaches. The first involves calculating the uncoded gradient on the original data points and subsequently encoding that gradient using the coefficients of $ {\bf B} $. Alternatively, in scenarios such as linear regression problems or supervised learning in neural networks, we can scale the labels and feature parameters of the data points by the respective coefficients of $ {\bf B} $ corresponding to the node partition and the iteration modulo $ n $. Then the gradients of theses adjusted (weighted) datapoints are the corresponding coded gradients $ \nabla{g_{i}} $. \footnote{It is important to note that a $ \frac{1}{n} $ scaling for $ \nabla{g}_{i} $ is unnecessary, as demonstrated in the proof in Appendix, where normalizing weights achieves the required convergence for $ f = \sum_{i=1}^{n}f_{i} $. Therefore, only $ n_{i} $ is needed for $ f = \sum_{i=1}^{n}\frac{n_{i}}{n}F_{i} $ without the $ \frac{1}{n} $ term.}


Let the coded objective functions $g_j, 1\le j\le n$, and  the decoding matrix ${\bf A}=(a(i,j))_{1\le i,j\le n}$ be as described above.
Set decoding weights
\begin{equation}\label{weight.def} w_i= \Big(\sum_{j \in \Gamma_i} |a(i,j)|\Big)^{-1},\ 1\le i\le n,
\end{equation}
where
\begin{equation}\label{gammai.def}
\Gamma_{i}=\{j, \  a(i,j)\ne 0\}.\end{equation}

For the decoding matrix $ {\bf A} $,
we define  its normalized decoding matrix ${\bf A}_{\rm nde}= \big( \tilde a(i,j)\big)_{1\le i, j\le n}$ of size $n\times n$
by
\begin{equation}\label{normalizedfgi.eq}\tilde a(i,j)= w_i a(i,j) ,\  1\le i, j\le n,
\end{equation}
 and  its row stochastic decoding matrix $ \bar{\bar{\bf A}}^{l} $ of size $ 2 n \times 2 n $ for FLUE of the form where $ \bar{\bf A}^{l}(k) =\bar{\bf A}^{l}(k)_{+} = \bar{\bf A}^{l}(k)_{-} $ by
\begin{equation} \label{Afit.def}
 \bar{\bar{\bf A}}^{l} =\left(\begin{array}{cc}
\tilde {\bf A}^{(+)}  &  \tilde {\bf A}^{(-)} \\
\tilde {\bf A}^{(+)} &\tilde {\bf  A}^{(-)}
\end{array}\right),
\end{equation}
 where   $w_i, 1\le i\le n$, are decoding weights given in \eqref{weight.def}, and
 $$\tilde {\bf A}^{(+)} =\big((\tilde a(i,j))_+\big)_{1\le i,j\le n}\ \ {\rm  and}\  \ \tilde {\bf A}^{(-)} =\big((-\tilde a(i,j))_+\big)_{1\le i,j\le n}$$ are
  positive/negative parts of  the normalized decoding matrix $\tilde {\bf A}=(\tilde a(i,j))_{1\le i, j \leq n}$  respectively.
Where $ a(i,j)_{+} = a(i,j) $ if $ a(i,j) \geq 0 $ and $ 0 $ otherwise.

\begin{definition}\label{A_hat}
$\hat{\bf A} =
\begin{pmatrix}
    \diagentry{\bar{\bf A}^{1}_{+}}\\
    \diagentry{\bar{\bf A}^{1}_{-}}\\
    &\diagentry{\bar{\bf A}^{2}_{+}}\\
    &\diagentry{\bar{\bf A}^{2}_{-}}\\
    &&\diagentry{\xddots}\\
    &&\diagentry{\xddots}\\
    &&&\diagentry{\bar{\bf A}^{n}_{+}}\\
    &&&\diagentry{\bar{\bf A}^{n}_{-}}
\end{pmatrix}$  
\end{definition}

\begin{definition}\label{D_hat}
The  column stochastic matrix
$\hat{\bf D} = 
 \begin{pmatrix}
    \diagentry{\frac{1}{n}{\bf D}^{1}} & & \frac{1}{n}{\bf D}^{2}  & & & \ldots &  \frac{1}{n}{\bf D}^{n}\\
    &\diagentry{\frac{1}{n}{\bf D}^{1}} & & \frac{1}{n}{\bf D}^{2} & & & \ldots &  \frac{1}{n}{\bf D}^{n}\\
    \frac{1}{n}{\bf D}^{1} & & \diagentry{\frac{1}{n}{\bf D}^{2}}\\
    & \frac{1}{n}{\bf D}^{1} & &\diagentry{\frac{1}{n}{\bf D}^{2}}\\
    &&&&\diagentry{\xddots}\\
    &&&&&\diagentry{\xddots}\\
    \frac{1}{n}{\bf D}^{1}&&&&&&\diagentry{\frac{1}{n}{\bf D}^{n}}\\
    &\frac{1}{n}{\bf D}^{1}&&&&&&\diagentry{\frac{1}{n}{\bf D}^{n}}
\end{pmatrix}$   
\end{definition}

\begin{equation}
\begin{split}
& \hat{\bf D}^{+}(i,j)  = 
\hat{\bf D} (i- (i \ div \ 2n )n, j- (j \ div \ 2n )n)  \\ &  for \ 1 \leq i \mod 2n \leq n \ \& \ 1 \leq j \mod 2n \leq n  
\end{split}
\end{equation}

\begin{equation}
\begin{split}
& \hat{\bf D}^{-}(i,j)  = 
\hat{\bf D} (i- (i \ div \ 2n + 1 )n, j- (j \ div \ 2n + 1)n)  \\ &  for \ n+1 \leq i \mod 2n \leq 2n \ \& \ n+1 \leq j \mod 2n \leq 2n 
\end{split}
\end{equation}

Then the $ n^{2} \times n^{2} $ matrix

\begin{definition}\label{D_hat}
$ \bar{\hat{\bf D}} = \hat{\bf D}^{+} = \hat{\bf D}^{-} = 
 \begin{pmatrix}
    \diagentry{\frac{1}{n}{\bf D}^{1}} & \frac{1}{n}{\bf D}^{2}  & & \ldots  \frac{1}{n}{\bf D}^{n}\\
    \diagentry{\frac{1}{n}{\bf D}^{1}} &  \frac{1}{n}{\bf D}^{2} & & \ldots  \frac{1}{n}{\bf D}^{n}\\
    &&&&\diagentry{\xddots}\\
    &&&&&\diagentry{\xddots}\\
    \diagentry{\frac{1}{n}{\bf D}^{1}} & \frac{1}{n}{\bf D}^{2}  & & \ldots  \frac{1}{n}{\bf D}^{n}\\
\end{pmatrix}$   
\end{definition}
is column stochastic.

\begin{definition}\label{D_hat}
$ {\bf T} = 
 \begin{pmatrix}
    \diagentry{{\bf T}^{1}} & & {\bf T}^{2}  & & & \ldots &  {\bf T}^{n}\\
    &\diagentry{{\bf T}^{1}} & & {\bf T}^{2} & & & \ldots &  {\bf T}^{n}\\
    {\bf T}^{1} & & \diagentry{{\bf T}^{2}}\\
    & {\bf T}^{1} & &\diagentry{{\bf T}^{2}}\\
    &&&&\diagentry{\xddots}\\
    &&&&&\diagentry{\xddots}\\
    {\bf T}^{1}&&&&&&\diagentry{{\bf T}^{n}}\\
    &{\bf T}^{1}&&&&&&\diagentry{{\bf T}^{n}}
\end{pmatrix}$   
\end{definition}

\begin{equation}
\begin{split}
& {\bf T}^{+}(i,j)  = 
{\bf T} (i- (i \ div \ 2n )n, j- (j \ div \ 2n )n)  \\ &  for \ 1 \leq i \mod 2n \leq n \ \& \ 1 \leq j \mod 2n \leq n  
\end{split}
\end{equation}

\begin{equation}
\begin{split}
& {\bf T}^{-}(i,j)  = 
{\bf T} (i- (i \ div \ 2n + 1 )n, j- (j \ div \ 2n + 1)n)  \\ &  for \ n+1 \leq i \mod 2n \leq 2n \ \& \ n+1 \leq j \mod 2n \leq 2n  
\end{split}
\end{equation}

\begin{definition}\label{D_hat}
$ \bar{\bf T} = {\bf T}^{+} = {\bf T}^{-} = 
 \begin{pmatrix}
    \diagentry{{\bf T}^{1}} & {\bf T}^{2}  & & \ldots  {\bf T}^{n}\\
    \diagentry{{\bf T}^{1}} &  {\bf T}^{2} & & \ldots  {\bf T}^{n}\\
    &&&&\diagentry{\xddots}\\
    &&&&&\diagentry{\xddots}\\
    \diagentry{{\bf T}^{1}} & {\bf T}^{2}  & & \ldots {\bf T}^{n}\\
\end{pmatrix}$   
\end{definition}

\subsection{Forming Matrices $ {\bf A} $, $ \bar{\bar{\bf A}}^{l} $ and $ {\bf B} $}

Any singular matrix $ {\bf B} $ with $ rank({\bf B}) \leq \min(n_{c},p_{c}) - 1 $ will work as an encoding matrix. Then for the matrix $ \bar{\bar{\bf A}}^{l} $ whether for the fixed or time-varying case we identify two scenarios in which $ \hat{\bf A} $ must be an SIA matrix. And to ensure that we must have first $ \bar{\bf A}_{ii}^{l} = \gamma_{i} $ where $ 1 \leq l \leq N_{c} $ and $ 1 \leq i \leq n $ for FLUE where $ \bar{\bf A}^{l} = \bar{\bf A}^{l}_{+} = \bar{\bf A}^{l}_{-} $. $ \bar{\bar{\bf A}}_{ii}^{l} = \gamma_{i} $ where $ 1 \leq l \leq N_{c} $ and $ 1 \leq i \leq 2 n $ for FLUE where $  \bar{\bf A}^{l}_{+} \neq \bar{\bf A}^{l}_{-} $. And an easy way to ensure that $ \hat{\bf A} $ is SIA for each case is to have at least one positive column for at least one matrix $ \bar{\bf A}^{l} $ for the first FLUE form. That is, for $ \bar{\bar{\bf A}^{l}} $ of the form where

\begin{equation} 
 \bar{\bar{\bf A}^{l}} =\left(\begin{array}{cc}
\tilde {\bf A}^{(+)}  &  \tilde {\bf A}^{(-)} \\
\tilde {\bf A}^{(+)} &\tilde {\bf  A}^{(-)}
\end{array}\right),
\end{equation}
has $ \tilde {\bf A}^{(+)} $ containing one positive column of entries. And $ \bar{\bar{\bf A}^{l}} $ has one positive column of entries for the later FLUE form.
To ensure the first conditions we need first to find the null space of matrix $ {\bf B} $. Then have $ \omega_{i} {\bf A}_{r_{i}}  = \gamma_{i} $ for each specified $ i $ according to the form of FLUE. And $ {\bf A}_{r_{i}} $ is any vector where $ {\bf A}_{r_{i}}{\bf B} = \mathbf{1}_{n_{p}} $ and $ \omega_{i} $ as defined earlier to be the inverse of the l1-norm of $ {\bf A}_{r_{i}} $.
To satisfy that we follow the following approach.
After identifying $ {\bf Y} $, the vector basis of null space of $ {\bf B}^{T} $ of rank $ d $, $ {\bf A}_{r_{i}} = \mathbf{1}_{n_{p}}^{T}{\bf B}^{\dag} + \beta{\bf x}^{T}{\bf Y} $ where $ {\bf x} \in \mathbf{R}^{d} $ and $ \beta \in \mathbf{R} $. And to have that the l1- norm of $ {\bf A}_{r_{i}} $ to be inverse of $ \omega_{i} $ we multiply $ {\bf A}_{r_{i}} $ by $ \bar{1}_{n_{p}} $ where $ \bar{1}_{n_{p}} $ is a vector of $ +1 $ and $ -1 $. So that, $ {\bf A}_{r_{i}} \bar{1}_{n_{p}}^{T} = (\mathbf{1}_{n_{p}}^{T}{\bf B}^{\dag} + \beta{\bf x}^{T}{\bf Y}) \bar{1}_{n_{p}}^{T} = \frac{1}{\omega_{i}} $ But at the end we need to ensure for the chosen combination of $ \bar{1}_{n_{p}} $ that the absolute values of $ {\bf A}_{r_{i}} $ match so that it has an l1-norm to be inverse of $ \omega_{i} $. And we need $ {\bf A}_{r_{i}}  = (\mathbf{1}_{n_{p}}^{T}{\bf B}^{\dag} + \beta{\bf x}^{T}{\bf Y})_{ii} = \frac{\gamma_{i}}{\omega_{i}} $.
Having the dimension of the null space of $ {\bf B} $ to be $ d $, then we choose $ d - 2 $ random entries of $ {\bf x} $, one random value of $ \beta $ and one random value for $ \omega_{i} $. Then we find the other entries values of $ {\bf x} $ thus forming $ {\bf A}_{r_{i}} $ and consequently a possible $ \bar{\bar{\bf A}^{l}} $ row. After forming a number of rows. We form the matrix $ \bar{\bar{\bf A}}^{l} $ ensuring the conditions to make $ \hat{\bf A} $ SIA for either form of FLUE.. Subsequently, we either fix $ \bar{\bar{\bf A}^{l}} $ throughout the iterations or form a new one for each iteration but chosen from a finite set, for both the fixed or time-varying case of the algorithm.

\section{ A-4: Main Algorithm }

\subsubsection{FLUE: Initialization}

The recent surge in successful deep learning applications has predominantly relied on various adaptations of stochastic gradient descent (SGD) for optimization. Neural network training uses gradient descent in its learning process. Therefore, it's a natural progression to develop algorithms for federated learning, building upon the foundation of SGD. The initial application of SGD to federated optimization may appear straightforward, involving a single batch gradient calculation, typically performed on a randomly selected client, during each communication round. SGD in its construction allows the utilization of part of the data in computing the gradient maintaining the convergence to the optimizer in a stochastic convergence process. Therefore, for our baseline approach, we opt for large-batch synchronous SGD. In the federated setting, the server divides at the beginning of the process the clients into $ m $ clusters taking into considerations the complexity, heterogeneity, type of data and the estimated number of stragglers per cluster. Each cluster $ c \in \{1,2,\ldots,m\} $ has $ N_{c} $ number of clients. For each of those $ m $ clusters the server forms a fixed matrix $ {\bf B}^{c} \in \mathbb{R}^{n_{c} \times p_{c} } $ where $ p_{c} $ is the number of partitions on all clients $ i $ of cluster $ c $ which can be related to the number of silos or batch sizes at that node. Then the sever sends the matrix $ {\bf B}^{c} $ to all clients in the cluster $ c $. This matrix $ {\bf B}_{c} $ can be any $ n_{c} \times p_{c} $ singular matrix. We update FLUE on the nodes such that each iteration $ i $ modulo $ n $ corresponds to row $ i $ of matrix $ {\bf B} = {\bf B}_{c} $ (i.e., for each $ 2n $ iterations $ {\bf B} $ repeats itself two times, one for the gradient descent updating equation and one for the gradient ascent updating equation). And $ {\bf B} = {\bf B}_{c} $ is fixed across cluster $ c $. After the matrix $ {\bf B}_{c} $ is sent to all nodes in cluster $ c $. Each client in the cluster utilizes gradient coding to the data partitions available to it. This can be done by either partitioning the data and encoding the datapoints labels and features according to $ {\bf B} ={\bf B}_{c} $ or through partitioning the data and applying the learning algorithm with encoded gradient steps according to $ {\bf B} $ encoding scheme. Then a corresponding decoding matrix $ \bar{\bar{\bf A}^{l}} $ for each node $ l $ of the cluster is computed. Then a proposed algorithm Federated Learning with Un-Encrypted model parameters (FLUE) is employed to accomplish the learning task described by \eqref{minimization_eqn}.

In general, the parameter server selects a set of $ s $ clusters of clients in each round and allows the learning of the model estimates through the computation of the gradient of the loss over all data held by these $ N = \sum_{c=1}^{s}N_{c} $ clusters’ clients. Thus, the fraction $f = \frac{ \sum_{c=1}^{s}\bar{N}_{c}}{\sum_{c=1}^{m}N_{c}} $ determines the global batch size, with $ f = 1 $ equivalent to full-batch (non-stochastic) gradient descent. For $ f = 1 $ we have all clusters with all there clients performing a full-batch gradient descent. That is, we have all $ m $ clusters with all $ \sum_{c=1}^{m}N_{c} $ clients with all datapoints $ \bar{n} = \sum_{c=1}^{m}\sum_{l=1}^{N_{c}}\bar{n}_{l} $ performing gradient descent. This baseline algorithm is referred to as  "FedSGD." In a typical FedSGD implementation with $ f = 1 $ and a fixed learning rate $ \alpha $ each client $ l $ calculates $ \nabla{F}_{l}=\frac{1}{\bar{n}_{l}}\sum_{i=1}^{\bar{n}_{l}}\nabla{f}_{i}(x(k)) $ representing the average gradient based on its local data at the current model $ x(k) $ on all data points $\bar{n}_{l} $ of client $ l $. In essence, each client performs one local gradient descent step using its data, and the server computes a weighted average of the resulting models. By structuring the algorithm this way, additional computation can be introduced at each client by iterating the local update multiple times before the averaging step. This extended approach is known as "FederatedAveraging" or "FedAvg" \citep{mcmahan2016communication}. The level of computation can be controlled through three key parameters: $f $(the fraction of clients performing computation per round), $ E $ (the number of training passes each client makes over its local dataset per round), and $ B $ (the local minibatch size for client updates). $ B $ can vary from the full local dataset which can be treated as a single minibatch to a minibatch consisting of only one datapoint. Thus, one end of this algorithm family corresponds to $ B = n_{l} $ and $ E =1 $ which precisely matches FedSGD. For a client with $ n_{i} $ local datapoints, the number of local updates per round is $ u_{i}=\frac{E n_{i}}{B} $.
In our approach we code the gradients after partitioning the data on each client and apply FLUE with any range of the hyperparameters $ f $, $ E $ and $ B $.

\subsubsection{FLUE: General Form}

For the first $n$ iterations $j$ modulo $2n$, each of the connected $n$ clients performs an updating step on its model weights ${\bf x}^+_i(2nk+j)$ using its coded data formed gradient according to matrix ${\bf B}$, utilizing the coded gradient descent and surplus variable ${\bf y}^+_i(2nk+j)$. Thus, forming its proxy weight ${\bf \bar{x}}^+_i(2n(k+1)+j)$. Then each client updates its surplus variable using previously sent proxy weights from the server and its previous model weight and surplus variables from previous iterations. It then sends its updated proxy weight ${\bf \bar{x}}^+_i(2n(k+1)+j)$ to the server which in turn aggregates all clients proxy weights received from the connected clients forming the server proxy update ${\bf \bar{x}}(2n(k+1)+j)$. Afterwards the server sends its update proxy weights to each of the chosen $n$ clients. Consequently, each of those clients finds its model weights ${\bf x}^+_i(2n(k+1)+j)$ using the server sent proxy weights ${\bf \bar{x}}(2n(k+1)+j)$.  Similarly, for the last $n$ iterations modulo $2n$, the process follows the same structure as the clients perform an updating step on their model weights  ${\bf x}^-_i(2nk+j)$ using its coded data formed gradient according to matrix ${\bf B}$, utilizing the coded gradient ascent and surplus variable ${\bf y}^-_i(2nk+j)$. Thus, each client forms its proxy weight ${\bf \bar{x}}^-_i(2n(k+1)+j)$. Then each client updates its surplus variable using previously sent proxy weights from the server and its previous model weight and surplus variables from previous iterations. It then sends its updated proxy weight ${\bf \bar{x}}^-_i(2n(k+1)+j)$ to the server which in turn aggregates all clients proxy weights received from the connected clients forming the server proxy update ${\bf \bar{x}}(2n(k+1)+j)$. And the process continues henceforth until convergence.

FLUE performs the following updating iterations at each node $i$ for $i\in\{1,2,\ldots,n\}$. Please note that $\Gamma_{i}= \Gamma_{i}(k) $ is the fixed support of the row of $ {\bf A}$ identified with iteration $ i $:
{\footnotesize
\begin{align}\label{updating_eqn1}
\begin{split}
{\bf \bar{x}}_{i}^{+} & (2n(k+1)+j) =   \bar{\bf A}^{i}_{jj}(k){\bf x}_{i}^{+}(2nk+j) \\
&  - \epsilon \sum_{s}[{\bf D}^{i}_{+}(k)]_{js} {\bf y}_{i}^{+}(2nk+s) - \alpha_{k} {\bf B}_{ji} \nabla{f}_{i}({\bf x}_{i}^{+}(2nk+j)) \\
{\bf y}_{i}^{+} & (2n(k+1)+j) =   {\bf x}_{i}^{+}(2nk+j) - {\bf \bar{x}}(2nk+j) \\
& -\epsilon {\bf y}_{i}^{+}(2nk+j)  + \sum_{s}[{\bf D}^{i}_{+}(k)]_{js} {\bf y}_{i}^{+}(2nk+s).
\end{split}
\end{align}}

{\footnotesize
\begin{align}\label{updating_eqn1a}
\begin{split}
{\bf \bar{x}}(2n(k+1)+j)&= \frac{1}{N}\sum_{i =1}^{N}{\bf \bar{x}}_{i}^{+}(2n(k+1)+j)
\end{split}
\end{align}}

{\footnotesize
\begin{align}\label{updating_eqn1b}
\begin{split}
{\bf x}^{+}_{i}(2n(k+1)+j)&= {\bf \bar{x}}(2n(k+1)+j) \\ & +\sum_{s \in \Gamma_{j} \backslash \{j\} }\bar{\bf A}^{i}_{js}(k){\bf x}_{i}(2nk+s) + \epsilon {\bf y}_{i}^{+}(2nk+j) 
\end{split}
\end{align}}

{\footnotesize
\begin{align}\label{updating_eqn2}
\begin{split}
{\bf \bar{x}}_{i}^{-} & (2n(k+1)+j) =   \bar{\bf A}^{i}_{(j-n),(j-n)}(k){\bf x}_{i}^{+}(2nk+j) \\
& - \epsilon \sum_{s}[{\bf D}^{i}_{-}(k)]_{js} {\bf y}_{i}^{+}(2nk+s)  + \alpha_{k} {\bf B}_{j-n,i} \nabla{f}_{i}({\bf x}_{i}^{-}(2nk+j)) \\
{\bf y}_{i}^{-} & (2n(k+1)+j)=   {\bf x}_{i}^{-}(2nk+j) - {\bf \bar{x}}(2nk+j) \\ 
& - \epsilon {\bf y}_{i}^{-}(2nk+j) + \sum_{s}[{\bf D}^{i}_{-}(k)]_{js} {\bf y}_{i}^{-}(2nk+s).
\end{split}
\end{align}}

{\footnotesize
\begin{align}\label{updating_eqn2a}
\begin{split}
{\bf \bar{x}}(2n(k+1)+j)&= \frac{1}{N}\sum_{i =1}^{N}{\bf \bar{x}}_{i}^{-}(2n(k+1)+j)
\end{split}
\end{align}}

{\footnotesize
\begin{align}\label{updating_eqn2b}
\begin{split}
{\bf x}^{-}_{i}(2n(k+1)+j)&= {\bf \bar{x}}(2n(k+1)+j) \\
& +\sum_{s \in \Gamma_{j} \backslash \{j\} }\bar{\bf A}^{i}_{j-n,s}(k){\bf x}_{i}(2nk+s) + \epsilon {\bf y}_{i}^{-}(2nk+j) 
\end{split}
\end{align}}

Each node $ i \in V $ maintains four vectors: two estimates ${\bf x}_{i}^{+}(2nk+j)$, $ {\bf x}_{i}^{-}(2nk+j)$  for $ 1 \le j \leq n $ and for $ n+1 \le j \le 2n $, respectively. And two surpluses $ {\bf y}_{i}^{+}(2nk+j)$ and ${\bf y}_{i}^{-}(2nk+j)$  for $ 1 \le j \leq n $ and for $ n+1 \le j \le 2n $, respectively.  All in $ \mathbb{R}^{N}$, where $ 2nk+j $ is the discrete time iteration. We use $ \hat{\bf x}_{i}(2nk+j) $ to mean either ${\bf x}_{i}^{+}(2nk+j)$ and $ {\bf x}_{i}^{-}(2nk+j-n)$ for $ 1 \leq j \leq n $ and $ n+1 \leq j \leq 2n $, respectively.

\begin{remark}
    The above general form is for FLUE with $ \bar{\bf A}^{l} $ replicated two times every $ 2 n $ iterations. 
\end{remark}

\subsubsection{FLUE: Special Form}

In conjunction to this general form of the algorithm which uses the advantage of coding and surplus noise to secure privacy, we present a special variant where there are no surpluses. To keep our presentation comprehensive we present the updating equations for this special variant below:

{\footnotesize
\begin{align}\label{updating_eqn1'}
\begin{split}
{\bf \bar{x}}_{i}^{+}(2n(k+1)+j)&=\bar{\bf A}^{i}_{jj}(k){\bf x}_{i}^{+}(2nk+j) \\
& - \alpha_{k} {\bf B}_{ji} \nabla{f}_{i}({\bf x}_{i}^{+}(2nk+j))\\
\end{split}
\end{align}}

{\footnotesize
\begin{align}\label{updating_eqn1'a}
\begin{split}
{\bf \bar{x}}(2n(k+1)+j)&= \frac{1}{N}\sum_{i =1}^{N}{\bf \bar{x}}_{i}^{+}(2n(k+1)+j)
\end{split}
\end{align}}

{\footnotesize
\begin{align}\label{updating_eqn1'b}
\begin{split}
{\bf x}^{+}_{i}(2n(k+1)+j)&= {\bf \bar{x}}(2n(k+1)+j) \\
& +\sum_{s \in \Gamma_{j} \backslash \{j\} }\bar{\bf A}^{i}_{js}(k){\bf x}_{i}(2nk+s) 
\end{split}
\end{align}}

{\footnotesize
\begin{align}\label{updating_eqn2'}
\begin{split}
{\bf \bar{x}}_{i}^{-}(2n(k+1)+j)&=\bar{\bf A}^{i}_{(j-n),(j-n)}(k){\bf x}_{i}^{-}(2nk+j) \\
& + \alpha_{k} {\bf B}_{j-n,i} \nabla{f}_{i}({\bf x}_{i}^{-}(2nk+j))\\
\end{split}
\end{align}}

{\footnotesize
\begin{align}\label{updating_eqn2'a}
\begin{split}
{\bf \bar{x}}(2n(k+1)+j)&= \frac{1}{N}\sum_{i =1}^{N}{\bf \bar{x}}_{i}^{-}(2n(k+1)+j)
\end{split}
\end{align}}

{\footnotesize
\begin{align}\label{updating_eqn2'b}
\begin{split}
{\bf x}^{-}_{i}(2n(k+1)+j)&= {\bf \bar{x}}(2n(k+1)+j) \\
& +\sum_{s \in \Gamma_{j} \backslash \{j\} }\bar{\bf A}^{i}_{j-n,s}(k){\bf x}_{i}(2nk+s) 
\end{split}
\end{align}}

For the proof of this special case, we use can utilize the proof of the general case.

\begin{algorithm}
 \caption{FLUE: Special Form (No Noise Surpluses) Used for Neural Networks Federated Learning} 
\begin{algorithmic}[1]
 \STATE \textbf{ Initialization: The server divides the nodes into $ m $ clusters. Then it forms the encoding matrix $ {\bf B}_{c} $ of size $ n_{c} \times p_{c} $ and send it to each cluster. And identifies which nodes in each cluster use which partitions and the corresponding samples encoded according to their divisions. Each node $ l $ forms the decoding matrix $ \bar{\bar{\bf A}}^{l} = \begin{pmatrix} \bar{\bf A}^{l}_{+} \\ \bar{\bf A}^{l}_{-} \end{pmatrix} $ used in the algorithm in either the time-invariant or time-varying forms. }
  \FOR { each iteration $ t = 1 , 2 , \ldots $}
         \FOR { each cluster $ c $}
         \STATE $ k \leftarrow t \ | \ 2 n_{c} $
        \STATE $ N_{c}(k)  \leftarrow  \max(C N_{c}, 1) $ 
        \STATE { $ S_{c}(k)  \leftarrow $ set \ of \ random $ N_{c}(k) $ clients \ on \ cluster $ c $ at \ iteration $ t $ } 
        \FOR { each subiteration $ j = t \mod 2 n_{c} $ in cluster $ c $ }

        \FOR { each client $ l \in S_{c}(k) $ in parallel }

\STATE $ \bar{\bf x}_{l}(t) \leftarrow ClientProxyUpdate(l, {\bf x}_{l}(2n(k-1)+j)) $
\STATE $ \bar{m}_{c}(k) \leftarrow \sum_{l \in S_{c}(k)} \bar{n}_{l} $ 
\STATE { $ \bar{\bf x}(2nk+j) \leftarrow \sum_{l \in S_{c}(k)} \frac{\bar{n}_{l}}{\bar{m_{c}(k)}} \bar{\bf x}_{l}(t) $ }
\STATE $ {\bf x}_{l}(2nk + j) \leftarrow ClientUpdate(l, \bar{\bf x}(2nk + j )) $
    
    \ENDFOR
    \ENDFOR
    \ENDFOR
    \ENDFOR
    \STATE \STATE  Client $ l $ : 
    \STATE $ \mathcal{B}  \leftarrow $ (split partitions set $ P_{l} $ into \ batches \ of \ size $ B $)
    \FOR { each local epoch $ i $ from $ 1 $ to $ E $ } 
    \FOR { batch $ b \in \mathcal{B} $ }
    \STATE Compute coded gradient $ {\bf B}_{j \mod n, i} \nabla{f}_{i}({\bf x}_{i}^{+}(2nk+j)) $ on each batch according to its partition division
    \ENDFOR
    \ENDFOR
   \STATE $ ClientProxyUpdate(l, {\bf x}_{l}) \leftarrow \bar{\bf A}^{l}_{j \mod n, j \mod n}(k){\bf x}_{l}(2n(k-1)+ (j \mod n)) -  \alpha_{k} (-1)^{(j-1) \ div \ n} {\bf B}_{j \mod n, i} \nabla{f}_{i}({\bf x}_{i}(2nk+j)) $ 
   \STATE $ ClientUpdate(l, \bar{\bf x}(2nk + j )) \leftarrow {\bf \bar{x}}(2nk+j)+\sum_{s \in \Gamma_{j} \backslash \{j\} }\bar{\bf A}^{l}_{j \mod n, s}(k){\bf x}_{l}(2n(k-1)+s) $ 
   
\end{algorithmic}
\end{algorithm}

\begin{remark}
The following federated learning FLUE algorithm form is a special form variant with no surpluses and with replication of $ \bar{\bf A}^{l} $ two times every $ 2 n $ iterations.
\end{remark}

In the implementation of our algorithm, matrix $ {\bf B} $ must be fixed, repeating every $ n $ iterations (and $ 2n $). This requirement is essential for the computation of matrix $ {\bf A} $ and $ \bar{\bf A}^{l} $. The matrix $ {\bf B} $ should be known to all nodes without encryption, ensuring privacy through the utilization of $ \bar{\bf A}^{l} $ and surpluses. Two cases of computing $ {\bf D}^{l} $ for the FLUE algorithm exist, both requiring no coordination among nodes and no encryption. The first case involves a fixed $ {\bf D}^{l} $ every $ 2n $ iterations at each node $ l $. The second case involves a variable (random) $ {\bf D}^{l} $ every $ 2n $ iterations at each node $ l $, chosen from a finite or infinite set. For both cases $ {\bf D}^{l} $ need to be a column stochastic matrix and can be repeated twice every $ 2 n $ iterations or not. Privacy is maintained through the use of $ \bar{\bf A}^{l} $ and surpluses.

In the fixed $ \bar{\bf A}^{l} $ scenario, nodes independently calculate and repeat $ \bar{\bf A}^{l} $ every $ 2n $ iterations, either with or without replication. This case provides convergence proof for both scenarios of replication. In the variable (random) $ \bar{\bf A}^{l} $ scenario, each node independently and randomly computes $ \bar{\bf A}^{l} $ every $ 2n $ iterations, enhancing privacy. The convergence proof is extended to this scenario, highlighting its effectiveness in preserving privacy.

The FLUE mechanism starts with the server organizing clients into clusters and forming fixed matrices $ {\bf B}^{c} $ for each cluster. The server transmits $ {\bf B}^{c} $ to all clients within the cluster, with $ {\bf B}^{c} $ usually stochastic. The FLUE algorithm operates without encryption, and privacy is further enhanced by using either on-the-fly random $ \bar{\bf A}^{l} $ rows or different fixed $ \bar{\bf A}^{l} $ rows for nodes per iteration modulo $ 2n $. Privacy is maintained through the convergence of stochastic matrices, and the scaling of $ \bar{\bf A}^{l} $ is uniformly applied to nodes.

To address privacy concerns during convergence, the addition of surpluses $ {\bf y}^{+}_{i} $ and $ {\bf y}^{-}_{i} $ with noise is proposed. This additional noise helps mitigate the risk of compromising privacy as the algorithm converges on model weight parameters. The use of variable $ \epsilon $ is suggested to counter potential privacy issues arising from a constant added coefficient. The proposed FLUE algorithm provides a privacy-preserving federated learning approach without the need for encryption, ensuring convergence and robust security features.

\footnote{Notice in line 11 of Algorithm~1 the weighting factor is chosen so that the batch coded gradient $ \nabla{g} $ computed still satisfy $ \mathcal{E}[\nabla{g}] = \nabla{f} $. Similar to the original uncoded gradient case \citep{mcmahan2016communication} }


\subsubsection{Implementation: No Encryption}

As for the implementation of our algorithm, we require that matrix $ {\bf B} $ be fixed, that is repeating every $ n $ iterations (and evidently $ 2n $) (i.e., one for the gradient descent coding step and one for the gradient ascent coding step). As the algorithm for computing the matrix $ {\bf A} $ and consequently $ \bar{\bf A}^{l} $ all matrix $ {\bf B} $ should be available. This requires that $ {\bf B} $ be known to all nodes.
As for the matrix $ {\bf D}^{l} $, we do not require it to be replicated two times every $ 2n $ iterations but it could be. Requiring only column stochastic matrices for each $ n $ iterations will suffice. Moreover, either requirement doesn't need coordination among nodes as long as each node repeats the $ {\bf D}^{l} $ matrix every $ 2n $ iterations two times in its surpluses updating equations for the latter case. Meanwhile, we don't require encryption of $ {\bf D} $ as privacy is still protected through utilizing $ \bar{\bf A}^{l} $. And is further leveraged using surpluses. We differentiate between  two cases of computing $ {\bf D}^{l} $ for the FLUE algorithm, both of which need no coordination among nodes. And thus no encryption is used in computing this $ {\bf D}^{l} $ matrix. 
The first case is for fixed $ {\bf D}^{l} $ every $ 2n $ iterations at each node $ l $. For this case we don't require coordination among nodes as each node can replicate $ {\bf D}^{l} $ every $ 2n $ iterations independently thus keeping it fixed along the algorithm process. Thus, each node $ l $ calculates a column stochastic $ {\bf D}^{l} $ independently, repeating it every $ 2n $ iterations. So, for the case where the matrix $ \hat{\bf D} $ would be fixed (i.e., having two similar replicas every $ 2n $ iterations) in \eqref{D_hat}.  This allows us to easily prove convergence of our algorithm for the fixed $ \hat{\bf A} $ fixed $ \hat{\bf D} $ approach, whether in the scenario of replicating $ {\bf D}^{l} $ two times every $ 2 n $ iterations or not. Or even having $ \hat{\bf A} $ fixed and $ \hat{\bf D} $ varying from a finite set in either replication scenario. However, we restrict our analysis to the first structure.  

In our time-varying analysis case we can have for variable (random) $ {\bf D}^{l} $ every $ 2n $ iterations at each node $ l $ to be chosen from finite or infinite set of such $ {\bf D}^{l} $ at each node $ l $. For this case we also don't require coordination among nodes as each node does not need to replicate $ {\bf D}^{l} $ every $ 2n $ iterations along the algorithm process. So, for this case the matrix $ \hat{\bf D} $ would be random (i.e., varying every $ 2n $ iterations) in \eqref{D_hat}. This make us use the time-varying random $ \hat{\bf A} $ and time-varying $ \hat{\bf D} $ approach to prove convergence of our algorithm. 

\begin{remark}
It is worth noting that we could have developed a proof for variable (random) $ {\bf B} $ (i.e., time-varying doesn't repeat every $ 2 n $ iterations). That is for $ {\bf B} $ with no coordination between nodes while we can still preserve the replication of the same used $ {\bf B} $ two times every $ 2n $ iterations for this only requires each node to establish that independently. However, the proof requires results in the convergence of the product of stochastic matrices which further restricts the form of the utilized $ {\bf B} $.
\end{remark}

Thus, by requiring a coordinated $ {\bf B} $, we let the server form this $ {\bf B} $ at the beginning of the process and send it to all nodes. Here, we don't require encryption of $ {\bf B} $ as privacy is still protected through utilizing $ \bar{\bf A}^{l} $. And is further leveraged using surpluses.

As we have mentioned earlier, we can prove the convergence of FLUE algorithm in its general form  for two different cases, fixed $ \bar{\bf A}^{l} $ every $ 2n $ iterations at each node $ l $ and variable (random) $ \bar{\bf A}^{l} $ every $ 2n $ iterations at each node $ l $. We could have established convergence with no requirement to replicate $ \bar{\bf A}^{l} $ two times every $ 2n $ iterations (i.e., one for the gradient descent step and one for the corresponding gradient ascent step). For this case we don't require coordination among nodes as each node can replicate $ \bar{\bf A}^{l} $ independently. The proof we provided holds for the two mentioned cases of replicating $ \bar{\bf A}^{l} $ two times every $ 2 n $ iterations or not at each node $ l $.

\subsubsection{Fixed $ \bar{\bf A}^{l} $ every $ 2n $ iterations at each node $ l $}

For this case we don't require coordination among nodes as each node can replicate $ \bar{\bf A}^{l} $ every $ 2n $ iterations independently thus keeping it fixed along the algorithm process. Thus given the fixed $ {\bf B} $ each node $ l $ calculates $ \bar{\bf A}^{l} $ independently, repeating or not repeating every $ 2n $ iterations. The matrix $ \hat{\bf A} $ would be fixed (i.e., repeating every $ 2n $ iterations) in \eqref{A_hat}. This allows us to prove convergence of our algorithm using the fixed static $ \hat{\bf A} $ and fixed $ \hat{\bf D} $ approach. This has a moderate privacy concerns as although the server or another eavesdropper tries to learn the coding weights $ [\bar{\bf A}^{l}]_{ij} $ coefficients, it needs to learn from different nodes having different respective coefficients for each iteration which makes a cumbersome task. Meanwhile, since this learning can be done on a long period of time where the eavesdropper can utilize the repeating of these coefficients every $ 2n $ iterations along the federated learning process of a node $ l $ to its favor to compromise privacy. To further mitigate that, we upgrade our case to a random (variable) $ \bar{\bf A}^{l} $ every $ 2n $ iterations at each node $ l $, keeping the replication of each $ \bar{\bf A}^{l} $ two times every $ 2 n $ iterations. However, also utilizing the noise surplus in this case can provide a further privacy shield.
Since each node $ l $ preserves this structure of $ \bar{\bf A}^{l} $ independently  without exposing it we don't require any encryption $ \bar{\bf A}^{l} $ while still maintaining all the positive leveraged security privileges. We also showed that our convergence analysis holds for the case of having a fixed $ \bar{\bf A}^{l} $ but not replicated two times every $ 2n $ iterations. There is no privacy concerns for this case especially if we move a step forward to the time-varying $ \bar{\bf A}^{l} $ scenario.

\subsubsection{variable (random) $ \bar{\bf A}^{l} $ every $ 2n $ iterations at each node $ l $}

For this case we don't require coordination among nodes as each node does not need to replicate $ \bar{\bf A}^{l} $ every $ 2n $ iterations along the algorithm process. Thus given the fixed $ {\bf B} $ each node $ l $ calculates $ \bar{\bf A}^{l} $ independently and randomly. So, for the case of non-stragglers the matrix $ \hat{\bf A} $ would be random (i.e., varying every $ 2n $ iterations) in \eqref{updating_eqn}.  This make us use the time-varying (i.e., random) $ \hat{\bf A} $ and time-varying (i.e., random) $ \hat{\bf D} $ approach to prove convergence of our algorithm. This has the best leveraged secure implementation  as although the server or another eavesdropper tries to learn the coding weights $ [\bar{\bf A}^{l}]_{ij} $ coefficients, it needs to learn from different nodes having different respective coefficients for each iteration without any repetition. So that for each node and each iteration we have a different row vector of $\bar{\bf A}^{l} $ which makes this an extremely tedious task. Meanwhile, we can further mitigate privacy by also utilizing the noise surplus.
Since each node $ l $ need not preserve but the stochastic structure of $ \bar{\bf A}^{l} $ independently without exposing it we don't require any encryption $ \bar{\bf A}^{l} $ while still maintaining all the positive leveraged security privileges.
Meanwhile, $ \bar{\bf A}^{l} $  in the current algorithm convergence proof although not repeated every $ 2n $ iterations can be replicated twice for each $ 2n $ iterations or not. For the first case this can somehow compromise security, mainly if an eavesdropper can learn twice about the $ [\bar{\bf A} ]_{ij} $ coefficients. However, this is not a long time to use this information as in the fixed  $ \bar{\bf A}^{l} $ case. This is because $ \bar{\bf A}^{l} $ is varying every $ 2n $ iterations.
It is worth emphasizing that  $ \bar{\bf A}^{l} $ can be varying every $ 2n $ iterations with no replication during these iterations (i.e., varying every $ n $ iterations). This case is the most secure case as no coefficients are orderly repeated to allow learning. The proof analysis holds for both cases of replicating $ \bar{\bf A}^{l} $ or not two times every $ 2 n $ iterations.

\vspace{0.2cm}

Matrix $ \bar{\bf A}^{l} $ is row stochastic with simple eigenvalue one by construction.

\begin{lemma}\label{L1}
The matrix $ \hat{\bf A} $ is row stochastic with simple eigenvalue one.
\end{lemma}

\subsubsection{ FLUE: Encryption Free }

In the initial phase, the server organizes clients into $ m $ clusters based on factors like data complexity, data type an heterogeneity. Each cluster $ c $  within the set $ \{ 1, 2, \ldots, m \} $  consists of $ N_{c} $ clients.  For each of those $ m $ clusters the server forms fixed matrices $ {\bf B}^{c} \in \mathbb{R}^{n_{c} \times p_{c} } $ where $ p_{c} $ is the number of partitions on all nodes $ i $ of cluster $ c $ which can be related to the number of silos or batch sizes at that node. And $ n_{c} $ is the period of repetition of the algorithm iterations. Subsequently, the server transmits matrix $ {\bf B}^{c} $ identified with cluster $ c $ to all clients within cluster $ c $. The matrix $ {\bf B}^{c} $ in its reduced form can be any singular matrix. But it is preferable to have $ {\bf B}^{c} $ stochastic (normalized rows) \footnote{Since then the application of $ \bar{\bf A}^{l} $ on coded gradients is equivalent to having $ \bar{\bf A}^{l} $ then $ {\bf B}^{c} $ applied on uncoded gradients. And we know that consensus distributed algorithms of stochastic matrices on local gradients converge}. The FLUE (Federated Learning with Unencrypted Updates) mechanism is then updated on the nodes. Each iteration $ i $ modulo $ n $ corresponds to row $ i $ of matrix $ {\bf B} = {\bf B}^{c} $ (i.e., for each $ 2 n $ iterations $ {\bf B} $ repeats itself two times, one for the gradient descent updating equation and one for the gradient ascent updating equation). And $ {\bf B} = {\bf B}^{c} $ is fixed across all nodes in cluster $ c $. Without loss of generality, let's assume that $ p_c =n_c $ (i.e., each nodes corresponds to one partition $ f_{i} $ of $ \sum_{i=1}^{n}f_{i} $) resulting in $ {\bf B}^{c} \in \mathbb{R}^{ n_{c} \times n_{c} } $. This analysis naturally extends to the general scenario of any partition size $ p_c $ on condition $ p_c $ doesn't exceed the number of raw datapoints. Let us assume that there is one cluster, i.e., $ m = 1 $. This is applicable because federated optimization or distributed optimization commonly employs Stochastic Gradient Descent (SGD), which permits the decomposition of the global function $ f(x) $ into $ f(x)=\sum_{i=1}^{n}f_{i}(x) $ comprising any number of local functions. The critical aspect is maintaining the coefficients of $1 $ in the sum, while the local functions can be formulated relative to the available data. It is possible to decompose $ f(x) $ into local functions $ f_{i}(x) $ to a certain degree; beyond that point, the local functions $ f_{i}(x) $ cannot be formed. This extreme case corresponds to an SGD with a batch size of one ($ B = 1 $). 

After receiving the variable proxy weights $ \bar{\bf x} $ from the server, each node computes a random row $ [\bar{\bf A}^{l}]_{i:} $ where $ t=2kn+i $ (can choose from a set of possible rows). Importantly, this row need not be consistent across all nodes, and no coordination between nodes is required. As a result, the server and other nodes possess no information about the specific row of $ \bar{\bf A}^{l} $ calculated at node $ l $ during a particular iteration. We distinguish two scenarios of matrix $ \bar{\bf A}^{l} $, either fixed every $ 2n $ iterations, or random every $ 2n $ iterations. For both scenarios, we require the replication of matrix $ \bar{\bf A}^{l} $ two times every $ 2n $ iterations or not. 
This approach ensures greater privacy for the model weights without relying on encryption. Our algorithm for the time-varying scenario relies on the convergence of stochastic matrices of specific forms (i.e., SIA matrices), of which the $ \bar{\bf A}^{l} $ matrices are a part, to demonstrate the convergence of the algorithm, in random (time-varying) $ \bar{\bf A}^{l} $ matrices cases without encryption. Alternatively, in the fixed $ \bar{\bf A}^{l} $ per iteration modulo $ 2 n $ scenario, this fixed $ \bar{\bf A}^{l} $ can vary among different nodes. For both scenarios, this approach does not require coordination from a specific location, such as a node or server, and avoids encryption in the algorithm's implementation. In both scenario, the server generates the matrix $ {\bf B}={\bf B}^{c} $ and sends it to all nodes of cluster $ c $. Each node then creates all possible rows of $ \bar{\bf A}^{l} $ and selects $ 2n $ rows to use in every $ 2n $ iterations. This method ensures that the server remains unaware of which rows of $ \bar{\bf A}^{l} $ each node uses per iteration, and nodes do not need to coordinate the selection of $ \bar{\bf A}^{l} $ rows. However, it's worth noting that that the fixed $ \bar{\bf A}^{l} $scenario variant of the algorithm has a limitation compared to the main FLUE algorithm which relies on random matrices $ \bar{\bf A}^{l} $ no-encryption implementation, as it uses fixed rows of $ \bar{\bf A}^{l} $ per $ 2 n $ iterations, which can potentially compromise privacy more easily if repeated a long time along the process.

\subsubsection{ FLUE: Privacy Mitigating Features}

Privacy enhancement is preserved through no coordination in $ \bar{\bf A}^{l} $ across nodes or with the server. Thus the server is only aware of the proxy model $ \bar{\bf x}_{l} $ from each node with no ability to infer the exact model $ {\bf x}_{l} $. The operation of finding $ {\bf x}_{l} $ is reserved to each node $ l $ solely with the use of its corresponding matrix $ \bar{\bf A}^{l} $.
Further enhancements in privacy preservation for model updates can be achieved by introducing surplus variables in each learning iteration at the nodes. Although our algorithm typically employs a fixed $ {\bf B} $ matrix across all nodes within a client's cluster, which is usually managed from a centralized server (in non-encryption implementation), we also investigate the scenario of utilizing random (time-varying) $ {\bf B} $ matrices in our simulations. However, for the consistency of our proofs in the current paper, we limit our analysis to the case of a fixed $ {\bf B} $ matrix, while acknowledging the inclusion of random $ {\bf B} $ in the simulation section. Meanwhile, to maintain privacy preservation, we can either implement a random (time-varying) $ \bar{\bf A}^{l} $ matrix without encryption, where each node independently updates using randomly generated $ \bar{\bf A}^{l} $ coefficients with no coordination among nodes. Alternatively, we can employ a variable $ \bar{\bf A}^{l} $ across nodes $ l $ which is fixed per iteration modulo $ 2n $ at each node without coordination among nodes that can be either replicated two times or not per these iterations.


\subsubsection{ Matrices $ {\bf A} $ and $ {\bf B} $ Scaling}\label{Scaling_A_B }

Regarding the case of scaling of $ \bar{\bf A}^{l} $, we apply it to all nodes $ l $ per iteration.

This becomes particularly relevant when we opt not to employ additional privacy safeguards, such as the inclusion of surplus variables like $ y^{+}_{i} $ and $ y^{-}_{i} $. 

When scaling $ \bar{\bf A}^{l} $ within an iteration, it is essential to apply the scaling uniformly to all instances of $ A_{ii}^{l} $  for all nodes $ l $  participating in the iteration $ (2kn+i) $. 

The coordination for this scaling can be established either during the initial creation of both matrices or on a per-iteration basis, depending on the chosen implementation approach.

\subsubsection{ Matrices $ {\bf A} $ and $ {\bf B} $ Randomness}

In the FLUE implementations discussed in this paper, the matrix $ {\bf B} $ remains fixed for nodes within each iteration, operating in a modulo $ n $ fashion. Although we can introduce randomness to $ {\bf B} $ to enhance privacy, the associated proof becomes more complex. Nevertheless, we have provided simulations to explore this scenario. Regarding the $ \bar{\bf A}^{l} $ matrices, we can use a fixed row across all nodes per iteration modulo $ 2n $ , offering a degree of privacy preservation. However, this approach necessitates coordination, which can compromise privacy, and is best suited for use with encryption. In this setup, eavesdropping becomes a concern. A more effective strategy is to employ different rows of $ \bar{\bf A}^{l} $ for nodes per iteration modulo $ 2n $, eliminating the need for coordination and providing stronger privacy safeguards. Although the $ \bar{\bf A}^{l} $ row used in the learning update step is not fixed for different nodes per iteration, it repeats every $ 2n $ iterations due to the algorithm's updating equations cycle for the fixed $ \hat{\bf A} $ case. Following from this, it's worth noting that not only $ {\bf B} $ also remains constant for each node every $ 2n $ iterations, $ \bar{\bf A}^{l} $ is also fixed per node every $ 2n $ iterations, albeit with variations among nodes. To maximize the effectiveness of our FLUE implementation and enhance privacy protection, we rely on the core FLUE algorithm, which does not incorporate encryption. This enables us to use either on-the-fly random $ \bar{\bf A}^{l} $ rows for nodes per iteration, or different $ \bar{\bf A}^{l} $ rows for nodes per iteration fixed modulo $ 2n $, which vary across nodes. This approach eliminates the need for node coordination, a feature that encryption leverages effectively. Furthermore, the convergence analysis for this case can be derived from the special case of fixed $ \bar{\bf A}^{l} $ and $ {\bf B} $ matrices for nodes modulo $ 2n $, without requiring more complex analytical tools typically needed for completely random $ \bar{\bf A}^{l} $ matrices applied to nodes per iteration (i.e., we use Stochastic Indecomposable Aperiodic matrices $ {\bf A}^{l} $).
It is worth noting here, that is as we have mentioned before we require for both scenarios; of repeating $ {\bf A}^{l} $ every $ 2 n $ iterations per each node or random $ {\bf A}^{l} $ that $ {\bf A}^{l} $ replicates two times for each $ 2 n $ iterations or not. We provided forms of FLUE for both scenarios and showed a convergence proof that holds for both.

\subsubsection{ Surpluses $ {\bf y}^{+}_{i} $ and $ {\bf y}^{-}_{i} $ added Noise}

Privacy concerns persist even with the incorporation of additional privacy measures, such as scaling $ {\bf B} $. This concern becomes more pronounced as our learning algorithm converges on the model weight parameters. When we employ a fixed $ \epsilon $ at this stage, having $ y_{i} $ constant across node although $ y_{i} \neq x_{i} $ can jeopardize privacy. This is because it becomes easier for eavesdroppers to deduce $ x_{i} $ when it is subjected to a constant added coefficient applied consistently across nodes and iterations. Although scaling $ {\bf B} $ can help alleviate this issue, our analyzed implementation mandates a fixed $ {\bf B} $ per iteration, which makes it relatively straightforward for eavesdroppers to compromise privacy due to the fixed $ \epsilon $ and fixed $ {\bf B} $. To address this challenge, introducing a variable $ {\bf B} $ would add complexity to the adversary's learning strategy. However, it's essential to note that analyzing a variable $ {\bf B} $ would require a more intricate proof involving the convergence of specific stochastic matrices for both $ \bar{\bf A} $ and $ {\bf B} $. This is a complex undertaking that goes beyond the scope of this paper, and the current framework may not readily support such convergence. Nonetheless, in our specific context, we can employ a variable $ \epsilon $ to circumvent the potential privacy compromise mentioned earlier. And still the convergence proof for the learning algorithm is straightforward.

\section{ A-5: Simulation Results}\label{results}

We have studied the convergence rate analytically in the previous section from the perspective of federated optimization relative to the convexity of the local functions on each node and the coding matrices used. In this simulation section we shed a light on the convergence rate in optimization and federated learning applications.
We subsequently consider the following unconstrained convex optimization problem 
\begin{equation}\label{objective.simulation}
 \arg\min_{{\bf x}\in \mathbb{R}^{N}} \|  {\bf F}{\bf x}- {\bf y}  \|_{2}^{2},\end{equation}
on a decentralized network consisting of $N_c$ nodes, ${\bf F}$ is a random matrix of size $M\times N$  whose entries are independent and identically distributed and chosen from the standard normal distribution.
And
\begin{equation}
{\bf y}={\bf F} {\bf x}_o\in \mathbb{R}^M
\end{equation}
where the entries of ${\bf x}_o$ are identically independent random variables sampled from the uniform bounded random distribution between $ -1 $ and $ 1 $.
The solution $ {\bf x} ^{*}$ of the above optimization problem is the least squares solution of the overdetermined system  
${\bf y}  = {\bf F}{\bf x}_o, {\bf x}_o\in \mathbb{R}^N$. 
In this section,  we  demonstrate the  performance of the FLUE  algorithm  for both variants of updating equations \eqref{updating_eqn} and \eqref{updating_eqn_O*} for fixed as well as random coding matrices $ \bar{\bf A}^{l} $ implementations. And we employ to solve the convex optimization problem
\eqref{objective.simulation} and also compare it with the performance of the conventional distributed gradient descent algorithm (DGD).

Assume that the network has one cluster, that is, $ m=1 $ formed of $N_c$ nodes where each node $ l $ is partitioned to $ \bar{n}_{l} $ data points. Then we can partition  the data on those $N_c$ nodes, and accordingly, the random measurement matrix ${\bf F}$,  the measurement data ${\bf y}$,  and the objective function  $f({\bf x}):= \|  {\bf F}{\bf x}- {\bf y}  \|_{2}^{2}$ in \eqref{objective.simulation} as follows,
 $$f({\bf x})= \sum_{l=1}^{\bar{n}} f_l({\bf x}):= \sum_{l=1}^{N_c} \sum_{i=1}^{\bar{n}_{l}}\|{\bf F}_i {\bf x}-{\bf y}_i\|_2^2.  $$ where each row $ F_{i} $ can be thought as the set of features, each $ y_{i} $, a label of a data point in a linear regression problem. And $ {\bf x} $ as the model parameters to be learned.
In our simulations, we assume that the partitioned matrices have the same size $ \bar{n}_{l} =constant $ which is equal to the number of rows ${\bf F}_i$ or the number of labels ${\bf y}_i $ on each node $ i $ for $ 1 \leq i \leq N_c $ multiplied by the length of the model parameters vector $ N $.

In our simulations, we take $ (M, N)= (150, 100) $ for Figures 1, 4 and $ (M, N)= (70, 40) $ for  Figures 5 and 6.
We use absolute error
$ {\rm AE}:={\bf x}_{1\le i\le N_c} \frac{\|{\bf x}_i(k)-{\bf x}_o\|_2}{\|{\bf x}_0\|_2}  $
and  consensus error
$ {\rm CE}:={\bf x}_{1\le i\le N_c} \frac{\|{\bf x}_i(k)-\bar{\bf x}(k)\|_2}{\|{\bf x}_o\|_2} $
to measure the performance of the FLUE algorithm \eqref{updating_eqn_O*} and  $n_c$ is the number of nodes in the network.

It is worth mentioning that in the following simulations we employed a fixed coding matrix $ {\bf B} $ except the fourth simulation of Figure 6 where we used a random coding matrix $ {\bf B} $.
In Figure 1, we have simulated the convergence of FLUE version with $ {\bf O} = {\bf Q}_{\epsilon, {\bf T}} $ for the above defined overdetermined system linear regression problem with step size $ \alpha_{k} = \frac{1}{(k+100)^{0.75}} $. While in Figure 5, we show the convergence rate for FLUE version with $ {\bf O} = {\bf Q}_{\epsilon, {\bf T}}$. For both we use fixed coding matrices $ \bar{\bf A}^{l} $ (i.e., repeating on each node $ l $ every $ 2 n $ iterations) with the number of nodes $ N_c = 5 $.
$ \alpha_{k} = \frac{1}{(k+100)^{0.75}} $.

Meanwhile, in Figure 2 we show the behavior of FLUE with version $ {\bf O} = {\bf Q}_{\epsilon, {\bf T}} $ for the random coding matrices $ \bar{\bf A}^{l} $ with the number of nodes $ N_c = 5 $. While we address the behavior of FLUE with version $ {\bf O} = {\bf Q}_{\epsilon, {\bf T}} $ for $ N_c = 7 $-node network using fixed coding matrices $ \bar{\bf A}^{l} $ in Figure 5 and using random coding matrices $ \bar{\bf A}^{l} $ in Figure 6.  
Both of the above mentioned simulations are with step size $ \alpha_{k} = \frac{1}{(k+100)^{0.75}} $.


\begin{figure}[h]\label{fig_1}
\centering
\includegraphics[trim=0 0 0 0 cm , clip, width=9cm]{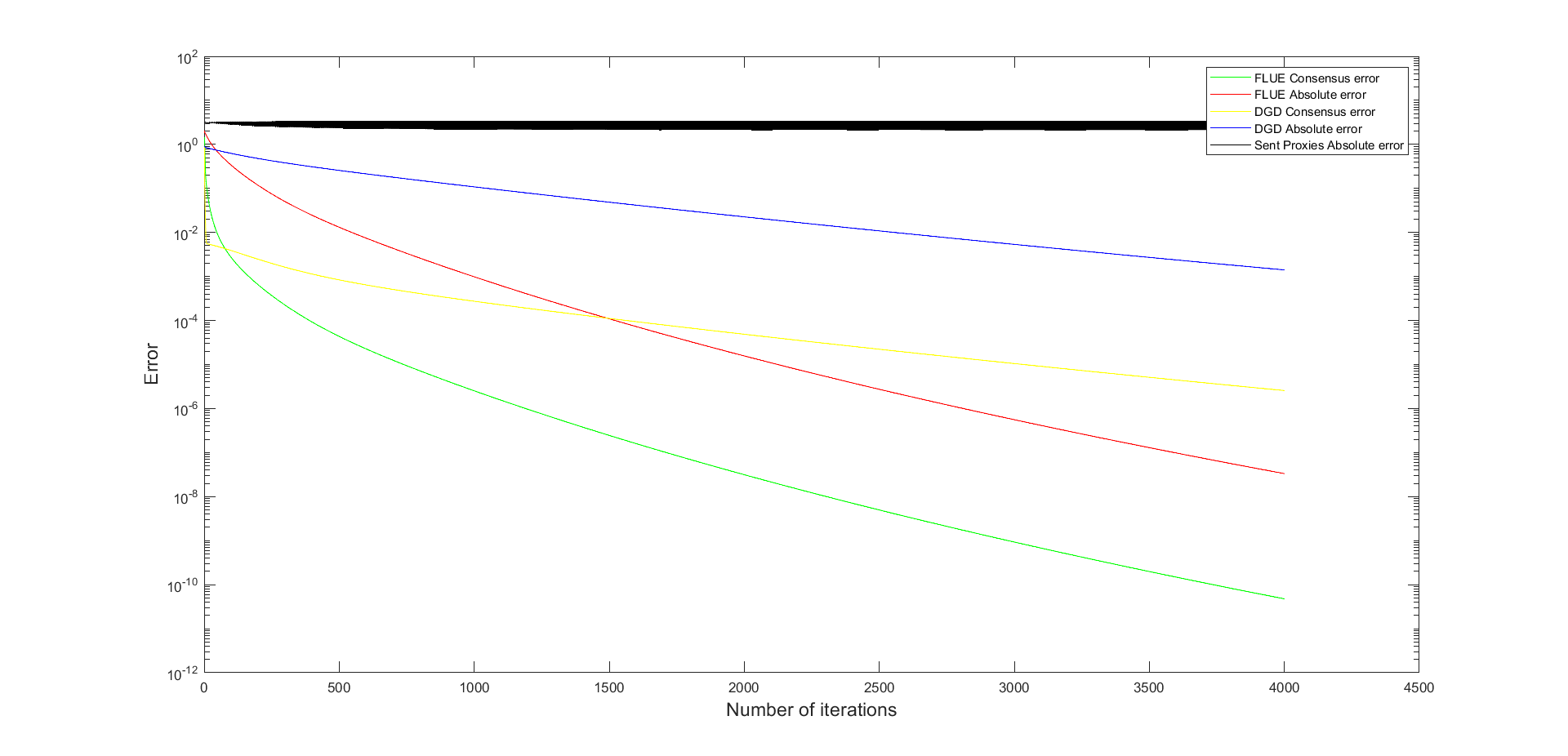}
\caption{Absolute and consensus errors vs iteration of FLUE and DGD for with random coding matrices $ \bar{\bf A}^{l} $ using FLUE \eqref{varyingQepsT} on a $ 5 $-node network}
\end{figure}

\begin{figure}[h]\label{fig_1}
\centering
\includegraphics[trim=0 0 0 0 cm , clip, width=9cm]{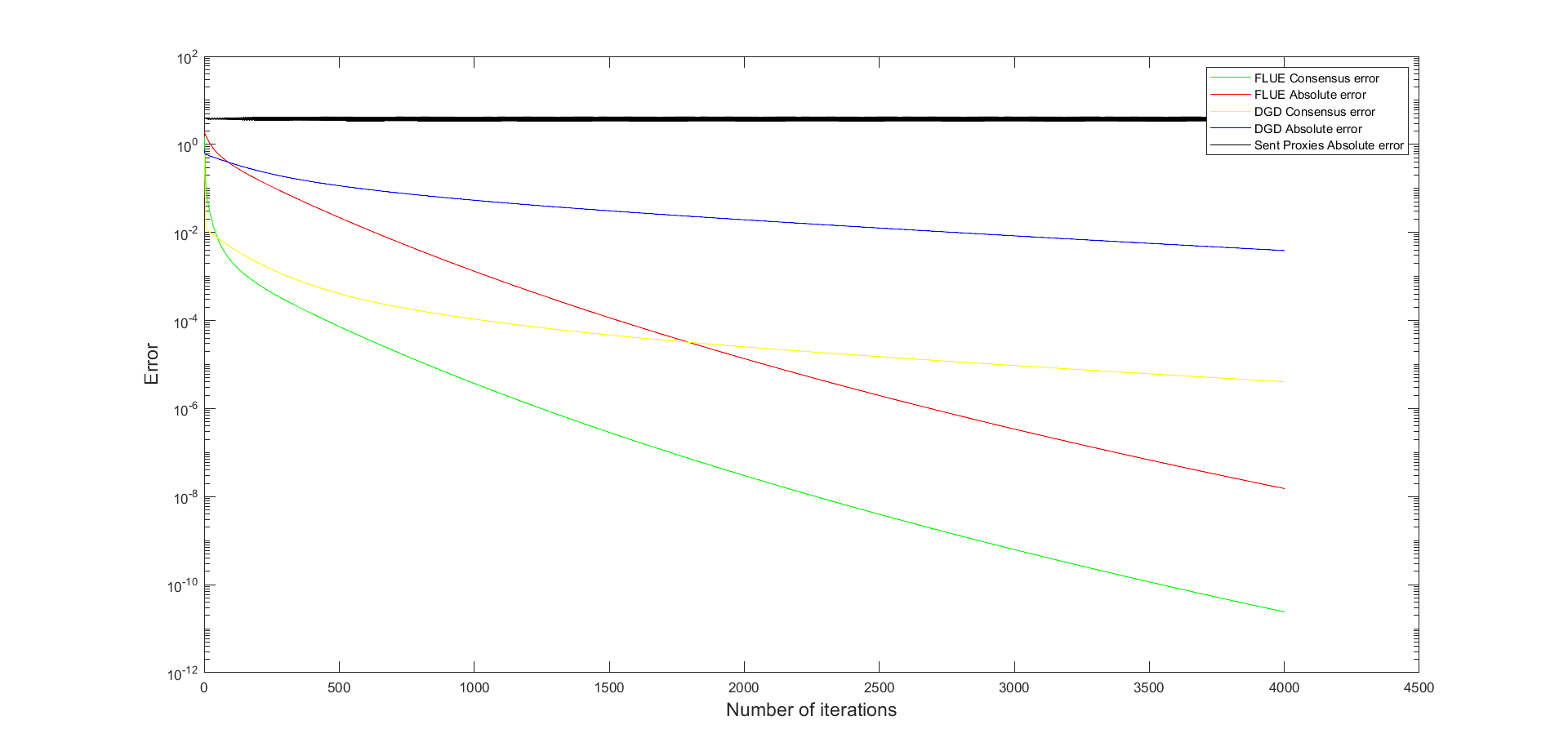}
\caption{Absolute and consensus errors vs iteration of FLUE and DGD with fixed coding matrices $ \bar{\bf A}^{l} $ using FLUE \eqref{fixedQepsT} on a $ 7 $-node network}
\end{figure}

\begin{figure}[h]\label{fig_1}
\centering
\includegraphics[trim=0 0 0 0 cm , clip, width=9cm]{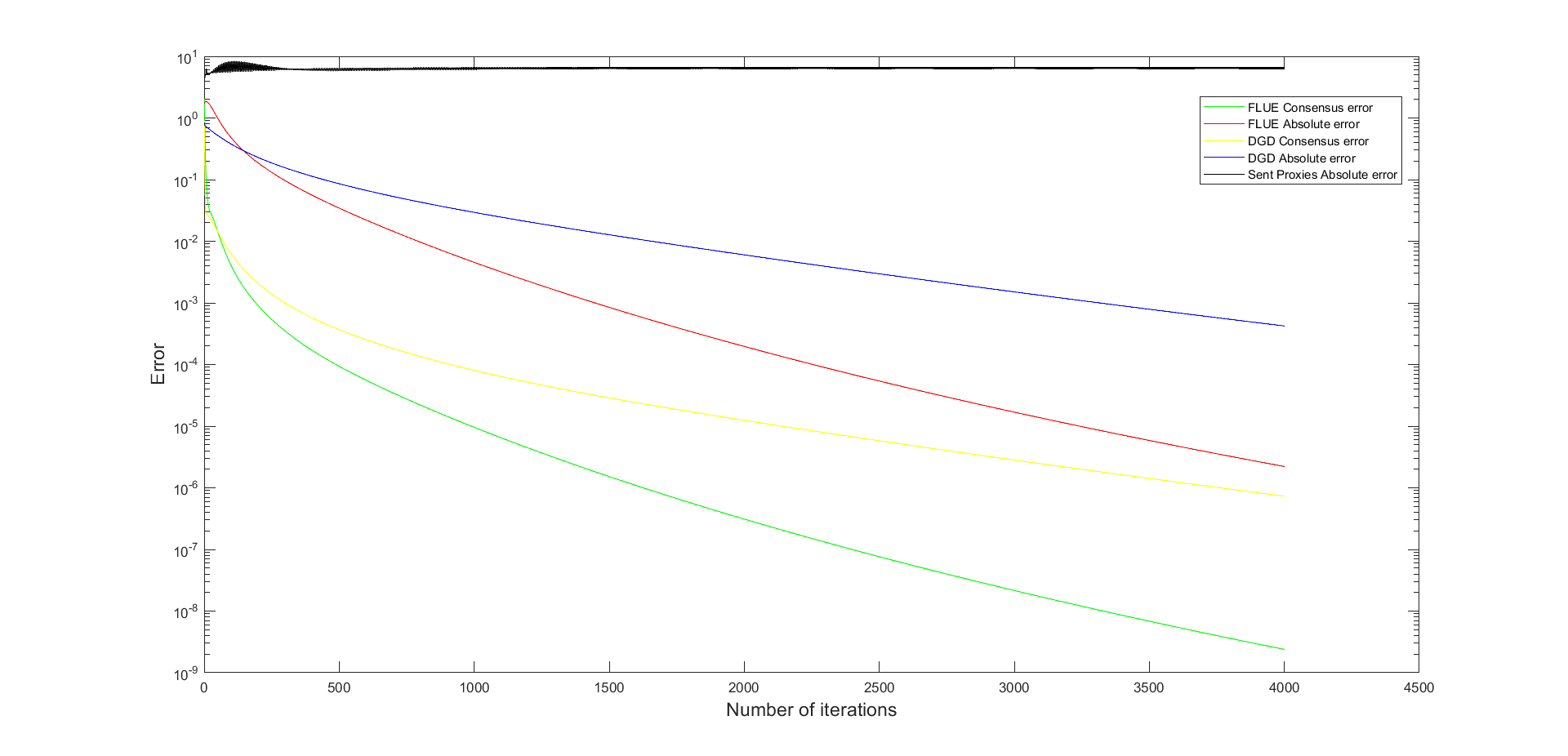}
\caption{Absolute and consensus errors vs iteration of FLUE and DGD with random coding matrices $ \bar{\bf A}^{l} $ using FLUE \eqref{varyingQepsT} on a $ 7 $-node network}
\end{figure}

For the fixed coding matrices $ \bar{\bf A}^{l} $ implementations of FLUE using \eqref{fixedQepsT}, we conceive that for a low value for the exponent in the denominator of the step size we have faster convergence than DGD. And as we increase the exponent the convergence gets degraded closer and closer to DGD. This justifies the decrease in performance of FLUE for the same network from Figure 5 to Figure 6. Also, the acceptable behavior in Figure 4 is evident although it is a time-varying network, since we have decreased the exponent of the denominator of the stepsizes (i.e., we anticipate more decreased performance if the exponent was still $ 0.75 $ as in Figure 1).

In the fixed coding matrices $ \bar{\bf A}^{l} $ implementation simulation, we noticed that FLUE version with $ {\bf O} = {\bf Q}_{\epsilon, {\bf T}} $ achieved better convergence rate than that with version $ {\bf O} = {\bf Q}_{\epsilon, {\bf I}} $ under matrices with close structure and same coding schemes and step sizes. This is because the convergence in the first case is guaranteed to behave better from the exact convergence proof and exact inequalities used while the latter is only proven to converge almost the same though verified empirically which can include a variance in the tight inequality bounds.

However, in the with random coding matrices $ \bar{\bf A}^{l} $ implementation simulation, we see that the more structure in these coding matrices the better is its convergence rate which is verified in having Figure 1 behavior better than that of Figure 4. That is, between fixed and random matrices $ \bar{\bf A}^{l} $ implementations, FLUE performs better in the first. Thus, the less structure in the random matrices $ \bar{\bf A}^{l} $ the smaller is the exponent of the denominator of the step size to allow better convergence. While we see that the consensus error behaves better than the absolute error in the fixed $ \bar{\bf A}^{l} $ implementations of FLUE and DGD. While in the random matrices $ \bar{\bf A}^{l} $ implementations, we can distinguish how the consensus error fluctuates frequently relative to the fluctuation in these matrices structure (i.e., more constrained structure to a more flexible structure). 
It is also worth mentioning that as we decrease the exponent in the denominator of the stepsize the convergence is degraded again to the divergence instability region. 
We can adjust our algorithm to behave with convergence by either increasing the exponent or calibrating the constant term added to $ k $ in the denominator of the step size.

It is also worth emphasizing that these algorithmic convergence behaviors are directly dependent on the condition number of each of these systems on the nodes. The above figures are those of matrices $ {\bf F} $ of a condition number almost equal to $ 1 $. Meanwhile, if we fix the stepsize utilized and compare the behavior of different $ {\bf F} $ matrices with different condition numbers we see that: an increased condition number will result in a slow convergence rate which might ultimately effect the behavior of FLUE with respect to DGD as the convergence of the first becomes more degraded towards that of the latter. Meanwhile, a smaller condition number will result in faster convergence until ultimately it may cause the algorithm to enter the instability region with an overshoot increase and initial divergence away from the desired solution.
Thus, the choice of an adequate stepsize that compromises the effect of the condition number and the stepsize magnitude will result in a better performance in the convergence rate of FLUE with respect to DGD.
Therefore, all these results are in correlation with the predicted analysis from the convergence rates \eqref{Convergence_Rate} affected by matrices $ {\bf A} $ and $ {\bf B} $, accordingly.

\nocite{*}

\section{ A-6: Conclusion}\label{conlusion}

Federated Learning enables various devices, like mobile phones and computers, to collaborate on training a shared model while keeping data on their own devices. This separates machine learning from cloud data storage. Clients collaborate to train a global model without revealing their raw data, sharing computed gradients or model parameters. To protect privacy, clients add noise or encryption.

However, using noisy gradients still has limitations. Recent research focuses on using model parameters without data exchange, relying on encryption, and adding noise for privacy.

In this paper, we presented a federated learning algorithm FLUE using coded local gradients during local learning and exchanging coded combinations as model parameter proxies. This ensures data privacy without encryption, and we add extra noise for stronger privacy. We develop two algorithm variants: a general form and a special form without surplus noise. We demonstrate algorithm convergence and improved learning rates based on coding schemes and data characteristics.

We presented two encryption-free implementations with fixed and random coding matrices, offering promising simulation results for federated optimization and machine learning. We listed two forms of the algorithm, one with replicated $ \bar{\bf A}^{l} $ twice every $ 2 n $ iterations listed in the main paper. And without replication shown in the Appendix. 
And a general form with surpluses noise listed in the main paper and another special form without surpluses shown in the Appendix.

We haven't necessitated coordination across all nodes for the complete decoding matrices; instead, we restricted coordination to specific entries to facilitate precise aggregation. While this doesn't pose a privacy compromise, in our ongoing efforts to enhance security, we aim to explore methods that ensure zero coordination among nodes in future research.

Additionally, the use of a shared encoding matrix will not jeopardize security. We can still explore this further by creating algorithms that incorporate random encoding matrices, employing methods derived from the convergence of infinite products of random matrices in our future research.


\section{Appendix B: Convergence Analysis}

The algorithm for the fixed $ {\bf B} $ case across nodes satisfying modulo $ 2n $ condition can be written in the following form the updating iterations at each node $l$ for $l\in\{1,2,\ldots,n\}$. Please note that $\Gamma_{i}= \Gamma_{i}^{l}(k) $ is the fixed support of the row of $ {\bf A} $ identified with node $ l $ and iteration $ 2n(k+1)+i $ :
{\scriptsize
\begin{align}\label{updating_eqn}
\begin{split}
\hat{\bf x}_{i}^{l +}(k+1)&=\sum_{j = 1}^{2n}\bar{\bf A}_{ij}^{l}\hat{\bf x}_{j}^{l}(k) + \epsilon {\bf y}_{i}^{l +}(k) - \epsilon \sum_{j = 1}^{2n}[{\bf D}^{l}_{+}]_{ij} \hat{\bf y}_{j}^{l +}(k) \\
& - \alpha_{k} \nabla{g}_{i}({\bf \hat{X}}_{i}^{+}(k))\\
\hat{\bf x}_{i}^{l -}(k+1)&=\sum_{j = 1}^{2n}\bar{\bf A}_{ij}^{l}\hat{\bf x}_{j}^{l}(k) + \epsilon {\bf y}_{i}^{l -}(k) - \epsilon \sum_{j = 1}^{2n}[{\bf D}^{l}_{-}]_{ij} \hat{\bf y}_{j}^{l +}(k) \\
& + \alpha_{k} \nabla{g}_{i}({\bf \hat{X}}_{i}^{-}(k))\\
\hat{\bf y}_{i}^{l +}(k+1)&=\hat{\bf x}_{i}^{l +}(k) - \sum_{j = 1}^{2n}\bar{\bf A}_{ij}^{l}\hat{\bf x}_{j}^{l}(k) + (1+\epsilon)\sum_{j = 1}^{2n}[{\bf D}^{l}_{+}]_{ij} \hat{\bf y}_{j}^{l +}(k) \\
& - \epsilon {\bf y}_{i}^{l +}(k) \alpha_{k} \nabla{g}_{i}({\bf \hat{X}}_{i}^{-}(k)) \\
\hat{\bf y}_{i}^{l -}(k+1)&=\hat{\bf x}_{i}^{l -}(k) - \sum_{j = 1}^{2n}\bar{\bf A}_{ij}^{l}\hat{\bf x}_{j}(k) + (1+\epsilon)\sum_{j = 1}^{2n}[{\bf D}^{l}_{-}]_{ij} \hat{\bf y}_{j}^{l -}(k)  \\
& - \epsilon {\bf y}_{i}^{l -}(k) - \alpha_{k} \nabla{g}_{i}({\bf \hat{X}}_{i}^{+}(k))\\.
\end{split}
\end{align}}

Each node $ l \in V $ maintains four vectors for each $2n$ iterations: two estimates $\hat{\bf x}_{i}^{l +}(k)={\bf x}_{l}^{+}(2nk+i)$, $ \hat{\bf x}_{i}^{l -}(k)={\bf x}_{l}^{-}(2nk+i+n)$, two surpluses $ \hat{\bf y}_{i}^{l +}(k)={\bf y}_{l}^{+}(2nk+i)$ and $\hat{\bf y}_{i}^{l -}(k)={\bf y}_{l}^{-}(2nk+i+n)$ all in $ \mathbb{R}^{N}$. We use $ \hat{\bf x}_{j}^{l}(k) $ to mean either $\hat{\bf x}_{j}^{l +}(k)={\bf x}_{l}^{+}(2nk+j)$ and $ \hat{\bf x}_{j-n}^{l -}(k)={\bf x}_{l}^{-}(2nk+j)$ for $ 1 \leq j \leq n $ and $ n+1 \leq j \leq 2n $ (if $ i \neq j $ for $ 1 \leq i \leq n $ and if $ i \neq j + n $ for $ n + 1 \leq i \leq 2n $), respectively. And for $ i = j $ for $ 1 \leq i \leq n $ and if $ i = j + n $ for $ n + 1 \leq i \leq 2n $ we have $\hat{\bf x}_{j}^{l +}(k)=\frac{1}{n_{R}}\sum_{l=1}^{n_{R}}{\bf x}_{l}^{+}(2nk+j)$ and $ \hat{\bf x}_{j-n}^{l -}(k)=\frac{1}{n_{R}}\sum_{l=1}^{n_{R}}{\bf x}_{l}^{-}(2nk+j)$, respectively.

And $ \nabla{g}_{i}({\bf \hat{X}}_{i}^{+}(k)) $ to mean the aggregation of the coded gradients at iteration $ 2n(k+1)+i $ of each node $ l $ model estimate $ {\bf x}_{l}^{+}(2nk+i) $. That is, $ \nabla{g}_{i}({\bf \hat{X}}_{i}^{+}(k)) = \sum_{l=1}^{n_{R,i,k}}{\bf B}_{il}\nabla{f}_{l}({\bf x}_{l}^{+}(2nk+i)) $.
And $ \nabla{g}_{i}({\bf \hat{X}}_{i}^{-}(k)) $ to mean the aggregation of the coded gradients at iteration $ 2n(k+1)+i $ of each node $ l $ model estimate $ {\bf x}_{l}^{-}(2nk+i) $. That is, $ \nabla{g}_{i}({\bf \hat{X}}_{i}^{-}(k)) = \sum_{l=1}^{n_{R,i,k}}{\bf B}_{(i-n)l}\nabla{f}_{l}({\bf x}_{l}^{+}(2nk+i)) $.

{\scriptsize
\begin{align}
\begin{split}
{\bf x}_{i}^{l +}(k+1)&=\sum_{j = 1}^{2n}\hat{\bf A}_{ij}\hat{\bf x}_{j}(k) + \epsilon {\bf y}_{i}^{l +}(k) -  \epsilon \sum_{j = 1}^{2n}[\hat{\bf D}_{+}]_{ij} \hat{\bf y}_{j}(k) \\
& - \alpha_{k} \nabla{g}_{i}({\bf \hat{X}}_{i}^{+}(k))\\
{\bf x}_{i}^{l -}(k+1)&=\sum_{j = 1}^{2n}\hat{\bf A}_{ij}\hat{\bf x}_{j}(k) + \epsilon {\bf y}_{i}^{l -}(k) - \epsilon  \sum_{j = 1}^{2n}[\hat{\bf D}_{-}]_{ij} \hat{\bf y}_{j}(k) \\
& + \alpha_{k} \nabla{g}_{i}({\bf \hat{X}}_{i}^{-}(k))\\
{\bf y}_{i}^{l +}(k+1)&=\hat{\bf x}_{i}^{l +}(k) - \sum_{j = 1}^{2n}\hat{\bf A}_{ij}\hat{\bf x}_{j}(k) + (1+\epsilon)\sum_{j = 1}^{2n}[\hat{\bf D}_{+}]_{ij} \hat{\bf y}_{j}(k) \\
& - \epsilon {\bf y}_{i}^{l +}(k) + \alpha_{k} \nabla{g}_{i}({\bf \hat{X}}_{i}^{-}(k)) \\
{\bf y}_{i}^{l -}(k+1)&=\hat{\bf x}_{i}^{l -}(k) - \sum_{j = 1}^{2n}\hat{\bf A}_{ij}\hat{\bf x}_{j}(k) + (1+\epsilon)\sum_{j = 1}^{2n}[\hat{\bf D}_{-}]_{ij} \hat{\bf y}_{j}(k) \\
& - \epsilon {\bf y}_{i}^{l -}(k) - \alpha_{k} \nabla{g}_{i}({\bf \hat{X}}_{i}^{+}(k))\\.
\end{split}
\end{align}}

This follows since each $  \bar{\bf A}^{l} $ is stochastic so $ \hat{\bf A} = \frac{1}{n} \sum_{i=1}^{n} \bar{\bf A}^{l} $ is also stochastic. And $ \hat{\bf D} $ is column stochastic 
since $ \hat{\bf D} = \frac{1}{n} \sum_{i=1}^{n} {\bf D}  $  where $ {\bf D} $ is column stochastic (and fixed as matrix $ {\bf B} $).

We restrict our analysis to the non-stragglers case for both fixed $ {\bf \bar{A}}^{l} $ across nodes or variable $ {\bf \bar{A}}^{l} $ and where $ {\bf \bar{A}}^{l} $ repeats itself every $ 2n $ iterations or not. Meanwhile, we include in the simulation an example including the stragglers case scenarios and defer the proof for this case for a later endeavor.



For the special variant algorithm with no surpluses noises added to further secure privacy, the updating equations reduce to:

{\footnotesize
\begin{align}\label{updating_eqn}
\begin{split}
\hat{\bf x}_{i}^{+}(k+1)&=\sum_{j = 1}^{2n}\bar{\bf A}_{ij}^{l}\hat{\bf x}_{j}(k)  -\alpha_{k} \nabla{g}_{i}({\bf \hat{X}}_{i}^{+}(k))\\
\hat{\bf x}_{i}^{-}(k+1)&=\sum_{j = 1}^{2n}\bar{\bf A}_{ij}^{l}\hat{\bf x}_{j}(k)  + \alpha_{k} \nabla{g}_{i}({\bf \hat{X}}_{i}^{-}(k))\\
\end{split}
\end{align}}

As it can be seen this becomes a simple stochastic weighting algorithm but with a gradient descent and gradient ascent steps. It can be easily proved. But to keep our analysis consistent the proof of this algorithm follows from the proof of the general FLUE algorithm with the surpluses by noticing that the matrix $ {\bf O} $ is now:

{\small
\begin{equation}\label{matrixO}
  {\bf O} = {\bf Q}_{\epsilon , {\bf T}}      =  \left(\begin{array}{cc}
                 \hat{\bf A} & \ \ \ \ \ \ \ \ \epsilon {\bf I}_{2n^{2} \times 2n^{2}}  \\
         {\bf I}_{2n^{2} \times 2n^{2}} -  \hat{\bf A}  &  
          \hat{\bf D}- \epsilon {\bf I}_{2n^{2} \times 2n^{2}}   \\ 
\end{array}\right) 
\end{equation}}

Thus, the proof of this special case utilizes the general case proof where the characteristics of the general $ {\bf O} $ matrix used in that proof are also evident in the above special $ {\bf O} $.
We see here that for the fixed $ {\bf O} $ case in this special variant we don't require coordination between nodes in order to utilize the static case convergence proof of the algorithm as long as each node repeats its $ \bar{\bf A}^{l} $ matrix every $ 2n $ iterations.

The FLUE algorithm \eqref{updating_eqn1} can be summarized by the following recursive equation:

\begin{equation}\label{updating_eqn}
\begin{split}
 {\bf z}_{i}(k+1)= \sum_{j=1}^{4n}[{\bf Q}_{\epsilon , {\bf T}}(k)]_{ij}{\bf z}_{j}(k)-\alpha_{k}{\nabla{\overline{g}}}_{i}({\bf z}_{i}(k)) \ for\ 1\leq i\leq4n,
\end{split}
\end{equation}
where

\begin{equation*}
\begin{split}
{\bf z}_{i}(k)=\left\{\begin{array}{cc}
                \hat{\bf x}_{{i \mod n}^{i \ div \ 2}}^{+}(k), &\ \ \ \ \ \ \ 1 \leq i \leq 2n^{2} \\
                &  and \ (i \ div \ n) \mod 2 = 1 \\
                \hat{\bf x}_{{i \mod n}^{i  \ div \ 2}}^{-}(k), &\ \ \ \ \ \ \ 1 \leq i \leq 2n^{2} \\
                & and \ (i \ div \ n) \mod 2 = 0 \\
                \hat{\bf y}_{{i \mod n}^{(i - 2n^{2}) \ div \ 2}}^{+}(k), &\ \ \ \ \ \ \ 2n^{2} + 1 \leq i \leq 4n^{2} \\
                & and \ (i \ div \ n) \mod 2 = 1 \\
                \hat{\bf y}_{{i \mod n}^{(i - 2n^{2}) \ div \ 2}}^{-}(k), &\ \ \ \ \ \ \ 2n^{2} + 1 \leq i \leq 4n^{2} \\
                & and \ (i \ div \ n) \mod 2 = 0 \\
        \end{array}\right\}   \ \ \
\end{split}
\end{equation*}

We take $ {\nabla{\overline{g}_{i}}}({\bf z}_{i}(k)) $ to correspond to the coded gradients relative to the first $ 2n $ variables identified with the estimates $ \bar{\bf X}_{i} $.
And $ \nabla{\overline{g}_{i}}(k) $ for $ 2n+1 \leq i \leq 4n $ corresponding to the surplus variables $ \bar{\bf Y}_{i} $. That is,

{\small \begin{equation*}
\begin{split}
{\nabla{\overline{g}_{i}}}({\bf z}_{i}(k))=\left\{\begin{array}{cc}
            \nabla{g_{i}}(\hat{\bf X}_{i}^{+}(k)),&\ 1 \leq i \leq n\\
        -\nabla{g_{i-n}}(\hat{\bf X}_{i}^{-}(k)),&\  n+1 \leq i \leq 2n\\ 
       -\nabla{g_{i}}(\hat{\bf X}_{i}^{+}(k)),&\ 2n+1 \leq i \leq 3n\\
          \nabla{g_{i-n}}(\hat{\bf X}_{i}^{-}(k)),&\  3n+1 \leq i \leq 4n
        \end{array}\right\}
\end{split}
\end{equation*}}

And $ g_{i} $ is the coded local objective function at iteration $ 2n(k+1)+i $ aggregated from all $ n_R $ received nodes at this iteration.

We have $ \| \nabla{\overline{g}}_{i} \|=\| \nabla{g}_{i} \| \leq  \sqrt{n} \| {\bf B} \| _{2, \infty} F = G $.

The step-size $ \alpha_{k} \geq 0 $ and satisfies $ \sum_{k=0}^{\infty}\alpha_{k} = \infty $,   $\sum_{k=0}^{\infty}\alpha_{k}^{2} < \infty $. 
The scalar $ \epsilon $ is a small positive number which is essential for the convergence of the algorithm and satisfying the conditions of Theorem 1. The steps of FLUE are summarized in Algorithm 2. We will prove in Section~\ref{convergence} that all nodes reach consensus to the optimal solution.

\subsection{Main Fundamental Theorem}\label{mainthm}

\subsection{For Fixed Matrix $ {\bf Q}_{\epsilon, {\bf T}} $}

For this purpose, we define a new matrix 

\begin{equation}\label{fixedQepsT}
{\bf Q}_{\epsilon , {\bf T}}= {\bf Q} + \epsilon {\bf G} 
\end{equation}

where $ {\bf Q} $ and $ {\bf G} $ and $ {\bf T} $  structures are shown in the final page.

\begin{prop}
Matrix $ {\bf Q} $ has spectral radius $ \rho({\bf Q}) = 1 $ with semi-simple eigenevalue $ 1 $ having algebraic multiplicity equal to its geometric multiplicity equal to $ 3 $ and all other eigenvalues having moduli less than $ 1 $.

\end{prop}

And we try to attain the same result through having the three independent eigenvectors corresponding to eigenvalue $ 1 $ to be exactly equal to that of $ {\bf Q} $ and all other eigenvalues have a magnitude less than $ 1 $.
We accomplish that through requiring some structure on matrix $ T $.

\begin{lemma}\label{L9n}
If we choose a suitable $ {\bf T} $ for $ {\bf O} = {\bf Q}_{\epsilon , {\bf T}} = {\bf Q} + \epsilon {\bf G} $, we guarantee that $ {\bf Q}_{\epsilon , {\bf T}} $ has only three right independent eigenvectors corresponding to eigenvalue $ 1 $, that are exactly equal to the corresponding eigenvectors of eigenvalue $ 1 $ for matrix $ {\bf Q} $. Such $ {\bf T} $ is $ {\bf T} = {\bf I}_{2n^{2} \times 2n^{2}} - \hat{\bf D} $ or $ {\bf T} = {\bf I}_{2n^{2} \times 2n^{2}}- \hat{\bf D}^{-1} $.
\end{lemma}

\textbf{Proof:}
We limit our proof to the case of $ {\bf T} $ where  $ {\bf T} = {\bf I}_{2n^{2} \times 2n^{2}} - \hat{\bf D} $.
Define
\begin{equation}
\hat{\bf A}=\left(\begin{array}{cc}
\hat{\bf A}_{+}^{1}  \\
\hat{\bf A}_{-}^{1} \\
\hat{\bf A}_{+}^{2}  \\
\hat{\bf A}_{-}^{2} \\
\xddots \\
\xddots \\
\hat{\bf A}_{+}^{n}  \\
\hat{\bf A}_{-}^{n}  \\
\end{array}\right)
\end{equation}
where $ \hat{\bf A}_{+}^{l} $ and  $ \hat{\bf A}_{-}^{l} $ are $ n \times 2n^{2} $ matrix, respectively.

We require for $ 1 \leq l \leq n $ that  $ \hat{\bf A}_{+}^{l}  = \hat{\bf A}_{-}^{l} $.

\begin{equation}
\hat{\bf A}_{+}(i,j)  = 
\hat{\bf A} (i - (i \ div \ 2n )n,j )  \   for \ 1 \leq i \mod 2n \leq n  \\
\end{equation}

\begin{equation}
\hat{\bf A}_{-}(i,j)  = 
\hat{\bf A} (i- (i \ div \ 2n + 1 )n, j )   \   for \ n+1 \leq i \mod 2n \leq 2n \\
\end{equation}

Then 
\begin{equation}
\hat{\bar{\bf A}} = \hat{\bf A}_{+} = \hat{\bf A}_{-}
\end{equation}
And for $ 1 \leq l \leq n $, we have
{\tiny
\begin{equation}
\hat{\bf A}_{1}^{l}(i,j)= \hat{\bf A}_{2}^{l}(i,j)=
\begin{cases}
\hat{\bar{\bf A}}_{1}^{l} & (i \mod  2 n, j \mod  2 n)  \\ &  for \ 2(l-1)n+1 \leq j \leq 2(l-1)n+n  \\
\hat{\bar{\bf A}}_{2}^{l} & (i \mod 2 n, (j \mod  2 n) - n) \\  &  for \ 2(l-1)n+n+1 \leq j \leq 2ln  \\
\hat{\bar{\bf A}}_{1}^{l} & (i \mod 2 n, i \mod  2 n) \\  &  for \ (i-j) \mod 2n = 0 \ and  \ 1 \leq i \mod 2 n \leq n \\
\hat{\bar{\bf A}}_{1}^{l} & (i \mod 2 n, i \mod  2 n) \\  &  for \ (i-j) \mod 2n = n \ and  \ n + 1 \leq i \mod 2 n \leq 2 n \\
0 & otherwise
\end{cases}
\end{equation}}
\footnote{We required here that $ \hat{\bf A}_{1}^{l} = \hat{\bf A}_{2}^{l} $. That is replication of $ \bar{\bf A}^{l} $ two times every $ 2 n $ iterations. But we could avoid such requirement while still maintaining the validity of this lemma and the whole proof.}
And the $ n \times n $ matrix
\begin{equation}
\hat{\bar{\bf A}}_{1}^{l}(i,j) =
\begin{cases}
\bar{\bf A}^{l}(i,j)   &  i \neq j  \\
\frac{\bar{\bf A}^{l^{'}}(i,j)}{n}   &  i=j \  where \ 1 \leq l^{'} \leq n
\end{cases}
\end{equation}

And the $ n \times n $ matrix
\begin{equation}
\hat{\bar{\bf A}}_{2}^{l}(i,j) =
\bar{\bf A}^{l}(i,j+n)  
\end{equation}

And denote the complement of $ \hat{\bar{\bf A}}_{1}^{l} \cup \hat{\bar{\bf A}}_{2}^{l} $ in 
$ \hat{\bf A}_{+}^{l} $ by the $ n \times (2n^{2} - n ) $ matrix $ \hat{\bar{\bf A}}_{*,+}^{l} $.
And the complement of $ \hat{\bar{\bf A}}_{1}^{l} \cup \hat{\bar{\bf A}}_{2}^{l} $ in 
$ \hat{\bf A}_{-}^{l} $ by the $ n \times (2n^{2} - n ) $ matrix $ \hat{\bar{\bf A}}_{*,-}^{l} $.

But 
\begin{equation}
 \hat{\bar{\bf A}}_{*}^{l}(i,j) = \hat{\bar{\bf A}}_{*,+}^{l}(i,j) = \hat{\bar{\bf A}}_{*,-}^{l}(i,j) \end{equation}
$ {\scriptsize 
 = \begin{cases}
\bar{\bf A}^{l}(i \ mod \ 2n , i \ mod \ 2n)   &  for \  (i-j) \ mod \ 2n = 0 \\
& and  \ j \leq 2(l-1)n \ or \ j \geq 2ln + 1  \\
0 & otherwise
\end{cases}} $

Matrix $ {\bf Q} $ has $ 3 $ independent eigenvectors for eigenvalue $ 1 $, (cf. Proposition 1). These eigenvectors are $ \left(\begin{array}{c}
{\bf 1}_{2n^{2} \times 1} \\ {\bf 0}_{2n^{2} \times 1} \end{array}\right) $,
and $ \left(\begin{array}{c}
{\bf 0}_{2n^{2} \times 1} \\ \bar{\bf v}^{+} 
\end{array}\right) $ and $ \left(\begin{array}{c}
{\bf 0}_{2n^{2} \times 1} \\ \bar{\bf v}^{-}\end{array}\right) $, where the $ 2n^{2} \times 1 $ vector
{\scriptsize
\begin{equation}
 \bar{\bf v}^{+} (j) = \begin{cases}  {\bf v}^{+} (j - (j \ div \ 2n)*n)   \ \ \  & if \ 1 \leq j \mod 2n \leq n \\
 0  \ \ \ & otherwise \end{cases}
\end{equation}}
and the $ 2n^{2} \times 1 $ vector
{\scriptsize
\begin{equation}
 \bar{\bf v}^{-} (j) = \begin{cases} {\bf v}^{-}(j - (j \ div \ 2n + 1 )*n)    \ \ \  & if \  n + 1 \leq j \mod 2n \leq 2 n \\
 0  \ \ \ & otherwise \end{cases}
\end{equation}}

And $ {\bf v}^{+} $ and $ {\bf v}^{-} $ are the right eigenvector of eigenvalue $ 1 $ for the column stochastic matrices $ \hat{\bf D}^{++} $ or $ \hat{\bf D}^{--} $, respectively (i.e., $ \hat{\bf D}^{++} {\bf v}^{+} =  {\bf v}^{+} $ and $ \hat{\bf D}^{--} {\bf v}^{-} = {\bf v}^{-} $), with all values positive and scaled such that $ {\bf 1}_{ n^{2} \times 1} ^{T} {\bf v}^{+} = {\bf 1}_{ n^{2} \times 1} ^{T} {\bf v}^{-} = 1 $. Where $ {\bf v}^{+}(i) = {\bf v}^{+}(j) $ for $ (i-j) \mod n = 0 $ and Where $ {\bf v}^{-}(i) = {\bf v}^{-}(j) $ for $ (i-j) \mod n = 0 $. 

We aim that $ {\bf Q}_{\epsilon , {\bf T}} $ has the same eigenvectors for eigenvalue $ 1 $.

$ \left(\begin{array}{c}
{\bf 1}_{2n^{2} \times 1} \\ {\bf 0}_{2n^{2} \times 1} \end{array}\right) $ is easily seen to be a right eigenvector of $ {\bf Q}_{\epsilon , {\bf T}} $ for the eigenvalue $ 1 $. However, for  
$ \left(\begin{array}{c}
{\bf 0}_{2n^{2} \times 1} \\ \bar{\bf v}^{+} 
\end{array}\right) $ and $ \left(\begin{array}{c}
{\bf 0}_{2n^{2} \times 1} \\ \bar{\bf v}^{-}\end{array}\right) $ to be right eigenvectors for $ {\bf Q}_{\epsilon , {\bf T}} $ of eigenvalue $ 1 $, we need

\begin{equation}\label{a1}
 \epsilon {\bf T} \bar{\bf v}^{+} = {\bf 0}_{2n^{2} \times 1},
\end{equation}
and 
\begin{equation}\label{a2}
 \epsilon {\bf T} \bar{\bf v}^{-} = \epsilon {\bf T} \bar{\bf v}^{-}  = {\bf 0}_{2n^{2} \times 1},
\end{equation}
 
that is

 And
 
\begin{equation}\label{b1}
(\hat{\bf D} - \epsilon {\bf T}) \bar{\bf v}^{+}  = \bar{\bf v}^{+},
 \end{equation}

 \begin{equation}\label{b2}
     (\hat{\bf D} -\epsilon {\bf T}) \bar{\bf v}^{-}  = \bar{\bf v}^{-},
 \end{equation}


 But
 \begin{equation}\label{c1}
    \hat{\bf D} \bar{\bf v}^{+}  = \bar{\bf v}^{+}.
 \end{equation}
 and
 \begin{equation}\label{c2}
    \hat{\bf D} \bar{\bf v}^{-}  = \bar{\bf v}^{-}.
 \end{equation}
Having (\ref{c1}) then (\ref{b1}) is equivalent to (\ref{a1}), so the requirement is reduced to (\ref{a1}). Similarly, having (\ref{c2}) then (\ref{b2}) is equivalent to (\ref{a2}), so the requirement is reduced to (\ref{a2}). That is
\begin{equation}
\epsilon {\bf T} {\bf v} = {\bf 0}_{2n^{2} \times 1}.
\end{equation}
where $ {\bf v} = \bar{\bf v}^{+} $ or $ {\bf v} = \bar{\bf v}^{-} $.  
To find a suitable $ {\bf T} $, we require $ {\bf T} {\bf v} = {\bf v} - {\bf v} $. One such $ {\bf T} $ is 
\begin{equation}
    {\bf T} {\bf v} = {\bf v} - \hat{\bf D} {\bf v},
\end{equation}
since $ {\bf v} = \hat{\bf D} {\bf v} $.
That is, $ {\bf T} = {\bf I}_{2n^{2} \times 2n^{2}} - \hat{\bf D} $.
 
Thus, by having $ {\bf T} = {\bf I}_{2n^{2} \times 2n^{2}} - \hat{\bf D} $ we have the $ 3 $ right independent eigenvectors of $ {\bf Q}_{\epsilon, T} $ corresponding to eigenvalue $ 1 $ to be the same as those corresponding to $ {\bf Q} $. 
We now prove that these are the only right eigenvectors corresponding to eigenvalue $ 1 $ for matrix $ {\bf Q}_{\epsilon , {\bf T}} $.


The geometric multiplicity of eigenvalue $ 1 $ is $ 3 $.


$ \left(\begin{array}{c}
{\bf 1}_{2n^{2} \times 1} \\ {\bf 0}_{2n^{2} \times 1} \end{array}\right) $, $ \left(\begin{array}{c}
{\bf 0}_{2n^{2} \times 1} \\ \bar{\bf v}^{+} 
\end{array}\right) $  and $ \left(\begin{array}{c}
{\bf 0}_{2n^{2} \times 1} \\ \bar{\bf v}^{-} 
\end{array}\right) $ and $ {\bf u} \neq {\bf w} $ 
i.e., $ {\bf u} \neq ({\bf Q}_{\epsilon , {\bf T}}- {\bf I}) {\bf u} $.

Then for the $ ({\bf Q}_{\epsilon , {\bf T}} - {\bf I}) {\bf u} = c \left(\begin{array}{c}
{\bf 1}_{2n^{2} \times 1} \\ {\bf 0}_{2n^{2} \times 1} \end{array}\right) $, we have 
{\small \begin{equation}
\begin{split}
     (\hat{\bf A} - {\bf I}_{2n^{2} \times 2n^{2}}) {\bf u}_{1}  +  & \epsilon ({\bf I}_{2n^{2} \times 2n^{2}} - \hat{\bf D}) {\bf u}_{2} = c {\bf 1}_{ 2n^{2} \times 1} \\
     ({\bf I}_{2n^{2} \times 2n^{2}} - \hat{\bf A})  {\bf u}_{1} + & (\hat{\bf D} -  \epsilon ({\bf I}_{2n^{2} \times 2n^{2}} - \hat{\bf D}) - {\bf I}_{2n^{2} \times 2n^{2}}){\bf u}_{2}  = {\bf 0}_{2n^{2} \times 1}.
\end{split}
\end{equation}}

Then adding the above subequations we get,
$ ( \hat{\bf D} - {\bf I}) {\bf u}_{2} = c {\bf 1}_{2n^{2} \times 1} $, where $ c \neq 0 $.
And subtracting the above subequations we get
\begin{equation}
\begin{split}
     2(\hat{\bf A} - {\bf I}_{2n^{2} \times 2n^{2}}) {\bf u}_{1}  +  & (2\epsilon - 1) ({\bf I}_{2n^{2} \times 2n^{2}} - \hat{\bf D}) {\bf u}_{2} = c {\bf 1}_{ 2n^{2} \times 1}
\end{split}
\end{equation}
Then
\begin{equation}\label{1}
\begin{split}
     2(\hat{\bf A} - {\bf I}) {\bf u}_{1} = 2 \epsilon  c {\bf 1}_{ 2n^{2} \times 1}
\end{split}
\end{equation}

Let $ {\bf v}_{1} $ be the right eigenvector of $ \hat{\bf A} $, that is, $ {\bf v}_{1}^{T} \hat{\bf A} = {\bf v}_{1} $.

Multiplying \eqref{1} by $ {\bf v}_{1}^{T} $ and having $ {\bf v}_{1}^{T}{\bf 1}_{ 2n^{2} \times 1} = 1 $, we get
\begin{equation}
\begin{split}
     {\bf v}_{1}^{T}(\hat{\bf A} - {\bf I}) {\bf u}_{1} & = \epsilon  c  {\bf v}_{1}^{T} {\bf 1}_{ 2n^{2} \times 1} \\
      {\bf v}_{1}^{T}\hat{\bf A} {\bf u}_{1} - {\bf v}_{1}^{T} {\bf u}_{1} & =  \epsilon  c \\
      {\bf v}_{1}^{T}{\bf u}_{1} - {\bf v}_{1}^{T} {\bf u}_{1} & =  \epsilon  c 
\end{split}
\end{equation}
But $ c \neq 0 $ therefore there exist no such $ {\bf u}_{1} $. Thus, there exist no generalized eigenvector of order $ 2 $ for eigenvalue $ 1 $ corresponding to the ordinary eigenvector $ \left(\begin{array}{c}
{\bf 1}_{2n^{2} \times 1} \\ {\bf 0}_{2n^{2} \times 1} \end{array}\right) $. 
Since this is true for $ m = 2 $. That is, $ ({\bf Q}_{\epsilon , {\bf T}} - {\bf I})^{m} {\bf x} = 0 $ and $ ({\bf Q}_{\epsilon , {\bf T}} - {\bf I})^{m-1} {\bf u} \neq 0 $. then it is true for every $ m > 2 $.
Therefore, there exist no generalized eigenvector for eigenvalue $ 1 $ corresponding to the ordinary eigenvector $ \left(\begin{array}{c}
{\bf 1}_{2n^{2} \times 1} \\ {\bf 0}_{2n^{2} \times 1} \end{array}\right) $. 

For $ ({\bf Q}_{\epsilon , {\bf T}} - {\bf I}) {\bf u} = \left(\begin{array}{c}
{\bf 0}_{2n^{2} \times 1} \\ \bar{\bf v}
\end{array}\right) $ where $ \bar{\bf v} = \bar{\bf v}^{+} $ or $ \bar{\bf v} = \bar{\bf v}^{-} $, we have 
{\small \begin{equation}
\begin{split}
     (\hat{\bf A} - {\bf I}_{2n^{2} \times 2n^{2}}) {\bf u}_{1} + &  \epsilon ({\bf I}_{2n^{2} \times 2n^{2}} - \hat{\bf D}) {\bf u}_{2}  = {\bf 0}_{2n^{2} \times 1} \\
     ({\bf I}_{2n^{2} \times 2n^{2}} - \hat{\bf A})  {\bf u}_{1} +  & (\hat{\bf D} -  \epsilon ({\bf I}_{2n^{2} \times 2n^{2}} - \hat{\bf D}) - {\bf I}_{2n^{2} \times 2n^{2}}){\bf u}_{2} = c {\bf v},
     \end{split}
\end{equation}}

Adding the above subequations wwe have that $ ( \hat{\bf D} - {\bf I}){\bf u}_{2} = c \bar{\bf v} $ where $ c \neq 0 $. Multiplying this result by $ {\bf 1}_{2n^{2} \times 1}^{T} $. And having  $ {\bf 1}_{2n^{2} \times 1} $ to be the right eigenvector of matrix $ \hat{\bf D} $ corresponding to eigenvalue $ 1 $. That is,  $ {\bf 1}_{2n^{2} \times 1}^{T} \hat{\bf D} =   {\bf 1}_{2n^{2} \times 1}^{T} $. And $ {\bf 1}_{2n^{2} \times 1}^{T} {\bf v} = 1 $ where $ \hat{\bf D}{\bf v} = {\bf v} $. Then we have
\begin{equation}
\begin{split}
     {\bf 1}_{2n^{2} \times 1}^{T} (\hat{\bf D}  - {\bf I}){\bf u}_{2} & = c {\bf 1}_{2n^{2} \times 1}^{T} {\bf v} \\
     {\bf 1}_{2n^{2} \times 1}^{T} \hat{\bf D}{\bf u}_{2}  -  {\bf 1}_{2n^{2} \times 1}^{T}{\bf u}_{2} & = c \\
    {\bf 1}_{2n^{2} \times 1}^{T}{\bf u}_{2}  -  {\bf 1}_{2n^{2} \times 1}^{T}{\bf u}_{2} & = c 
\end{split}
\end{equation}

But $ c \neq 0 $ therefore there exist no such $ {\bf u}_{2} $. Thus, there exist no generalized eigenvectors of order $ 2 $ for eigenvalue $ 1 $ corresponding to the ordinary eigenvectors $ \left(\begin{array}{c}
{\bf 0}_{2n^{2} \times 1} \\ \bar{\bf v}^{+} \end{array}\right) $ or $ \left(\begin{array}{c}
{\bf 0}_{2n^{2} \times 1} \\ \bar{\bf v}^{-} \end{array}\right) $, respectively. 
Since this is true for $ m = 2 $. That is, $ ({\bf Q}_{\epsilon , {\bf T}} - {\bf I})^{m} {\bf x} = 0 $ and $ ({\bf Q}_{\epsilon , {\bf T}} - {\bf I})^{m-1} {\bf u} \neq 0 $. then it is true for every $ m > 2 $.
Therefore,there exist no such $ {\bf u}_{2} $. Thus, there exist no generalized eigenvectors of order $ 2 $ for eigenvalue $ 1 $ corresponding to the ordinary eigenvectors $ \left(\begin{array}{c}
{\bf 0}_{2n^{2} \times 1} \\ \bar{\bf v}^{+} \end{array}\right) $ or $ \left(\begin{array}{c}
{\bf 0}_{2n^{2} \times 1} \\ \bar{\bf v}^{-} \end{array}\right) $, respectively.

Thus, we proved there exist no $ 4n^{2} \times 1 $ generalized eigenvectors. 
We exhausted all cases and therefore the algebraic multiplicity
of eigenvalue $ 1 $ for matrix $ {\bf Q}_{\epsilon , {\bf T}} $ is not equal to any $ t \geq 3 $, following the same analysis on generalized eigenvectors. Then the algebraic multiplicity of eigenvalue $ 1 $ is $ 3 $ and $ 1 $ is a semi-simple eigenvalue of $ {\bf Q}_{\epsilon , {\bf T}} $. 

Since what we have used is only $ ({\bf Q}_{\epsilon , {\bf T}} - {\bf I}) {\bf u} \neq {\bf u} $ which is common for all algebraic multiplicities $ t \geq 3 $ of eigenvalue $ 1 $ to disprove the existence of any such algebraic multiplicity of eigenvalue $ 1 $ to be greater than $ 3 $. That is, by disproving the existence of any generalized eigenvector for eigenvalue $ 1 $, beginning from the base case of order $ 2 $.

Therefore, the eigenvalue $ 1 $ of matrix $ {\bf Q}_{\epsilon , {\bf T}} $ has only $ 3 $ independent right eigenvectors that are the same as those corresponding to eigenvalue $ 1 $ in matrix $ {\bf Q} $. And the left eigenvectors corresponding to eigenvalue $ 1 $ can be proven to be the same for both $ {\bf Q}_{\epsilon , {\bf T}} $ and $ {\bf Q} $ similarly.

\begin{definition}
Define
\begin{equation}\label{eqn_e_double}
    \begin{split}
      {\epsilon_{0}({\bf Q},{\bf T})}  = \frac{1}{(20+12 n)^{n}(1+n)}(1-| \lambda_{ 2 n + 2 }({\bf Q}) |)^{n}.   
    \end{split}
\end{equation}
\end{definition}

\begin{lemma}\label{L9}
If the parameter $ \epsilon \in (0,{\epsilon_{0}({\bf Q},{\bf T})}) $ with $ {\epsilon_{0}({\bf Q},{\bf T})} $ defined in (\ref{eqn_e_double}) where $ \lambda_{2n+2}({\bf Q}) $ is the fourth largest eigenvalue of matrix $ {\bf Q} $, then $ 1 > | \lambda_{2n+2}({\bf Q}_{\epsilon, {\bf T}}) |,...,\\
| \lambda_{4n^{2}}({\bf Q}_{\epsilon, {\bf T}}) | > 0 $, the eigenvalues corresponding to matrix $ {\bf Q}_{\epsilon , {\bf T}} $.
\end{lemma}

\textbf{Proof:}  

Set $ {\bf H} = {\bf I} - \hat{\bf A} $ and $ \hat{\bf D} =  \begin{pmatrix}
    \diagentry{\frac{1}{n}{\bf D}^{1}} & & \frac{1}{n}{\bf D}^{2}  & & & \ldots &  \frac{1}{n}{\bf D}^{n}\\
    &\diagentry{\frac{1}{n}{\bf D}^{1}} & & \frac{1}{n}{\bf D}^{2} & & & \ldots &  \frac{1}{n}{\bf D}^{n}\\
    \frac{1}{n}{\bf D}^{1} & & \diagentry{\frac{1}{n}{\bf D}^{2}}\\
    & \frac{1}{n}{\bf D}^{1} & &\diagentry{\frac{1}{n}{\bf D}^{2}}\\
    &&&&\diagentry{\xddots}\\
    &&&&&\diagentry{\xddots}\\
    \frac{1}{n}{\bf D}^{1}&&&&&&\diagentry{\frac{1}{n}{\bf D}^{n}}\\
    &\frac{1}{n}{\bf D}^{1}&&&&&&\diagentry{\frac{1}{n}{\bf D}^{n}}
\end{pmatrix} $ \\
          then \\
           $ \| {\bf H} \|_{\infty} = 2 \max_{i \in \{1,..,2n^{2} \} } \sum_{j=1}^{2n^{2}} {\bf A}_{ij} < 2 $ and $ \| \hat{\bf D} \|_{2,\infty} < n $ since $ \| {\bf D}^{i} \| _{\infty} < n $. \\
          $ \| {\bf Q} \|_{\infty} \leq \| {\bf H} \|_{\infty} + \| \hat{\bf D} \| _{\infty} < 2+n $ and $ \| {\bf G} \| _{\infty} = \| {\bf I} - \hat{\bf D} \|_{\infty} \leq n + 1 $. \\
          By Lemma [\citep{bhatia2013matrix} , VIII.1.5]
          \begin{equation}
          \begin{split}
          d(\sigma({\bf Q}),\sigma({\bf Q}_{\epsilon, {\bf T}})) & \leq 4 ( \| {\bf Q} \| _{\infty} + \| {\bf Q}_{\epsilon , {\bf T}} \| _{\infty} )  ^{1-\frac{1}{n}} \| \epsilon {\bf G} \| _{\infty} ^ {\frac{1}{n}} \\
          & < 4(4+(2 + \epsilon) n+\epsilon)^{1-\frac{1}{n}} (1 + n) ^{\frac{1}{n}} \epsilon ^{\frac{1}{n}} \\
          & < 4(4+(2 + \epsilon) n+\epsilon) (1 + n) ^{\frac{1}{n}} \epsilon ^{\frac{1}{n}} \\
          & < 1 - | \lambda_{ 2 n + 2 }({\bf Q}) |,
          \end{split}
          \end{equation}
for $ \epsilon \in (0,{\epsilon_{0}({\bf Q},{\bf T})}) $ defined above.
But the unperturbed eigenvalues $ \lambda_{2n+2}({\bf Q}),..., \lambda_{4n^{2}}({\bf Q}) $ corresponding to matrix $ {\bf Q} $ lie inside the unit circle. Therefore, perturbing these eigenvalues by an amount less than $ {\epsilon_{0}({\bf Q},{\bf T})} $ will result in eigenvalues $ \lambda_{2n+2}({\bf Q}_{\epsilon, {\bf T}}),\ldots,\lambda_{4n^{2}}({\bf Q}_{\epsilon, {\bf T}}) $ of matrix $ {\bf Q}_{\epsilon , {\bf T}} $, which will remain in the unit circle.


\begin{lemma}\label{L10e}
Assume that $ {\bf Q}_{\epsilon , {\bf T}} $ is the updating matrix of the algorithm defined in \eqref{updating_eqn}. Then \\
(a) $  lim_{k\rightarrow \infty}{\bf Q}_{\epsilon , {\bf T}} ^{k} \rightarrow {\bf P}  $. \\
(b) For all $ i, j \in V, \ [ {\bf Q}_{\epsilon , {\bf T}} ^{k}]_{ij} $ converge to $ {\bf P} $ as $ k \rightarrow \infty $ at a geometric rate. That is, $ \| {\bf Q}_{\epsilon , {\bf T}} ^{k} - {\bf P} \| \leq  \Gamma \gamma^{k}  $, where $ 0 < \gamma < 1 $ and $ \Gamma > 0 $.
\end{lemma}

\textbf{Proof:}
The algebraic and geometric multiplicities of eigenvalue $ 1 $ of matrix $ {\bf Q}_{\epsilon , {\bf T}} $ are equal to $ 2 n + 1 $.
Since Lemma~\ref{L9} states that all eigenvalues have moduli less or equal to $ 1 $. Assume the left and right eigenvectors of eigenvalue $ 1 $ are $ {\bf v}_{1},{\bf v}_{2}, \ldots, {\bf v}_{2n+1} $ and $ {\bf u}_{1},{\bf u}_{2}, \ldots, {\bf u}_{2n+1} $ respectively. And $ {\bf v}_{2n+2},{\bf v}_{5},...,{\bf v}_{4n^{2}} $ and $ {\bf u}_{2n+2},{\bf u}_{5},\ldots,{\bf u}_{4n^{2}} $ are the left and right eigenvectors of eigenvalues $\lambda_{2n+2}({\bf Q}_{\epsilon, {\bf T}}),\lambda_{5}({\bf Q}_{\epsilon, {\bf T}}),\ldots,\lambda_{4n^{2}}({\bf Q}_{\epsilon, {\bf T}}) $ 
counted without multiplicity.
Then, $ {\bf Q}_{\epsilon , {\bf T}} $ represented in Jordan Block form decomposition is 
  
{\tiny
 \begin{equation}
\begin{split}
 {\bf Q}_{\epsilon , {\bf T}}       = \left(\begin{array}{c}
         {\bf u}_{1}(\epsilon) \\
         {\bf u}_{2}(\epsilon) \\
         {\bf u}_{3}(\epsilon)
         \xddots\\
         \xddots\\
         {\bf u}_{4n^{2}}(\epsilon)
     \end{array}\right) ^{T} \times
\end{split}
\end{equation}
\begin{equation}
\begin{split}    
\left(\begin{array}{cc} \begin{pmatrix}
    \diagentry{1}\\
    &\diagentry{1}\\
    &&\diagentry{1}\\
\end{pmatrix} \ & 0    \\
         0 \ &  \begin{pmatrix}
    \diagentry{\lambda_{4}({\bf Q}_{\epsilon, {\bf T}})}\\
    &\diagentry{\lambda_{5}({\bf Q}_{\epsilon , {\bf T}})}\\
    &&\diagentry{\xddots}\\
    &&&\diagentry{\xddots}\\
    &&&&\diagentry{\lambda_{4n^{2}}({\bf Q}_{\epsilon , {\bf T}})}\\
\end{pmatrix}
\end{array}\right)
\left(\begin{array}{c}
         {\bf v}_{1} \\
         {\bf v}_{2} \\
         {\bf v}_{3} \\
         \xddots\\
         \xddots\\
         {\bf v}_{4n^{2}}(\epsilon)
     \end{array}\right). 
\end{split}
\end{equation}}

Then
{\tiny \begin{equation}
\begin{split}
\begin{align*}
  {\bf Q}_{\epsilon , {\bf T}} ^{k}      = [{\bf u}_{1} \ {\bf u}_{2} \ {\bf u}_{3}]
   \begin{pmatrix}
    \diagentry{1}\\
    &\diagentry{1}\\
    &&\diagentry{1}\\
\end{pmatrix}^{k}
     \left(\begin{array}{c}
         {\bf v}_{1} \\
         {\bf v}_{2} \\
         {\bf v}_{3} \\
     \end{array}\right)   + \sum_{i=4}^{4n^{2}}({\bf P}^{*})_{i}({\bf J}^{*})^{k}_{i}({\bf O}^{*})_{i}.
    \end{align*} 
\end{split}
\end{equation}}

But $ \lim_{k \rightarrow \infty} {\bf J}^{k}_{i} = 0 $ since the diagonal entries of $ J_{i} $ are smaller than $ 1 $ in magnitude for all $ i $ (i.e the eigenvalues $ |\lambda_{i}({\bf Q}_{\epsilon , {\bf T}})| < 1 $)).

Thus,

$ \| {\bf Q}_{\epsilon , {\bf T}} ^{k} - {\bf P} \| = \|  \sum_{i=4}^{4n^{2}}({\bf P}^{*})_{i}({\bf J}^{*})^{k}_{i}({\bf O}^{*})_{i} \| \leq \sum_{i=4}^{4n} \| ({\bf P}^{*})_{i} \| \|({\bf J}^{*})^{k}_{i} \| \| ({\bf O}^{*})_{i} \| \leq \Gamma \gamma ^{k}$ where $ \Gamma < \infty $ and $ \gamma \in (0,1) $.





\begin{lemma}\label{L11e}
For $ 0 < \epsilon \leq \hat{\epsilon} \leq \epsilon_{0} ({\bf Q},{\bf T}) $ where $ {\epsilon_{0}({\bf Q},{\bf T})} $ is defined in (\ref{eqn_e_double}), then $ \hat{\epsilon} $ is a necessary bound for $ \lim_{k \rightarrow \infty} {\bf Q}_{\epsilon , {\bf T}} ^{k}=\lim_{k \rightarrow \infty}{\bf Q}^{k}= {\bf P} $.  \end{lemma}

$ \mathbf{Proof:} $
This follows from the contrapositive of Lemma~\ref{L9}. 

\begin{lemma}\label{L12e}
For $ 0 < \epsilon < \hat{\epsilon} \leq \epsilon_{0}({\bf Q},{\bf T}) $ where $ {\epsilon_{0}({\bf Q},{\bf T})} $ is defined in (\ref{eqn_e_double}), $ \hat{\epsilon} $ is a sufficient bound for $ \lim_{k \rightarrow \infty} {\bf Q}_{\epsilon , {\bf T}} ^{k}=\lim_{k \rightarrow \infty}{\bf Q}^{k}= {\bf P} $. 
\end{lemma}  

$ \mathbf{Proof:} $ 
This follows from Lemmas 1, \ref{L9n}, \ref{L9} and \ref{L10e}. i.e., the eigenvectors corresponding to eigenvalue $ 1 $ are the only contributing to the limit, all other eigenvalues contribution tends to $ 0 $ and these eigenvalue $ 1 $ eigenvectors are equal for both $ {\bf Q} $ and $ {\bf Q}_{\epsilon , {\bf T}} $.   

\begin{lemma}\label{L13e}
For $ 0 < \epsilon \leq {\epsilon_{0}({\bf Q},{\bf T})} $ where $ {\epsilon_{0}({\bf Q},{\bf T})} $ is defined in (\ref{eqn_e_double}), then \\
$ \lim_{k \rightarrow \infty}{\bf Q}_{\epsilon , {\bf I}} ^{k}=\lim_{k \rightarrow \infty}{\bf Q}^{k} = {\bf P} $. And $ {\epsilon_{0}({\bf Q},{\bf T})} $ is a necessary and sufficient bound for $ \lim_{k \rightarrow \infty}({\bf Q}_{\epsilon , {\bf T}})^{k}=\lim_{k \rightarrow \infty}{\bf Q}^{k}= {\bf P} $.
\end{lemma}

\textbf{Proof:}
This follows from Lemmas \ref{L11e} and \ref{L12e}. \\
That is, $ {\bf P} = \lim_{k \rightarrow \infty} {\bf Q}_{\epsilon , {\bf T}} ^{k} = \lim_{k \rightarrow \infty}{\bf Q}^{k} $. i.e., only the eigenvectors corresponding to eigenvalue $ 1 $ contribute to the limit and these eigenvectors are the same for both $ {\bf Q} $ and $ {\bf Q}_{\epsilon , {\bf T}} $.
But 
\begin{equation}
\begin{split}
{\bf P} =  \lim_{k \rightarrow \infty} {\bf Q}_{\epsilon , {\bf T}} ^{k} = \lim_{k \rightarrow \infty}{\bf Q} ^{k} = {\bf P} + \lim_{k \rightarrow \infty}\sum_{i=4}^{4n}{\bf P}_{i}{\bf J}_{i}{\bf Q}_{i} 
\end{split}
\end{equation}

See last page for the structure of $ {\bf P} $.

That is, 
\begin{equation}
\begin{split}
\begin{align}
{\bf P}=  \left(\begin{array}{cc}
              \lim_{k \rightarrow \infty}  \hat{\bf A}^{k} & \ \ \ \ \  {\bf 0}_{ 2n^{2} \times 2n^{2} }  \\
         {\bf P}_{3}  &  
        \left(\begin{array}{cc} {\bf P}_{4}   \end{array}\right)   \\
    \end{array}\right).  
    \end{align} \ \ \ 
\end{split}
\end{equation}

\vspace{1cm}

But $ \hat{\bf A} $ is a row stochastic matrix, then $ \lim_{k \rightarrow \infty} \hat{\bf A} ^{k} = {\bf 1}_{2n \times 1}{\boldsymbol \pi} ^{T} $ which is of dimension $ 2n \times 2n $ of rank $ 1 $ (repeated row $ {\boldsymbol \pi} $) where $ {\boldsymbol \pi} $ is the stationary state of matrix $ \hat{\bf A} $. This result is needed in Lemma~\ref{L11} and ~\ref{l12}, that is $ [{\bf P}]_{jl}=[{\bf P}]_{ql}= {\bf I}_{l} $ for $ 1 \leq j,q \leq 2n $.
Therefore, we have proved Theorem~\ref{Thm3e} which is the main edifice for the validity of our algorithm.



\begin{theorem}\label{Thm3e}
Suppose that $ {\bf Q}_{\epsilon , {\bf T}} $ is the matrix defined in \eqref{fixedQepsT} with the parameter $ \epsilon $ satisfying $ \epsilon \in (0, {\epsilon_{0}({\bf Q},{\bf T})}  ) $ where $ {\epsilon_{0}({\bf Q},{\bf T})} $ is defined in (\ref{eqn_e_double}), with $ \lambda_{4}({\bf Q}) $ the fourth largest eigenvalue of matrix $ {\bf Q}_{\epsilon , {\bf T}} $ by setting $ \epsilon = 0 $. Then \\
(a) $ \lim_{k\rightarrow \infty}{\bf Q}_{\epsilon , {\bf T}} ^{k} \rightarrow {\bf P}  $. Specifically, $ [\lim_{k\rightarrow \infty}{\bf Q}_{\epsilon , {\bf T}} ^{k}] _{[1:2n^{2}][1:2n^{2}]} = {\bf 1}_{2n^{2}} {\boldsymbol \pi}^{T} $ and \\
$ [\lim_{k\rightarrow \infty}{\bf Q}_{\epsilon , {\bf T}} ^{k}] _{[1:2n^{2}][2n^{2}+1:4n^{2}]}= {\bf 0}_{2n^{2} \times 2n^{2}} $, where $ [{\bf Q}_{\epsilon , {\bf T}}]_{[1:2n^{2}][1:2n^{2}]} = \hat{\bf A} $ and \\
$ [{\bf Q}_{\epsilon , {\bf T}}]_{[1:2n^{2}][2n^{2}+1:4n^{2}]} = \epsilon  {\bf T} $.\\
(b) For all $ i, j \in V, \ [({\bf Q}_{\epsilon , {\bf T}})^{k}]_{ij} $ converge to $ P $ as $ k \rightarrow \infty $ at a geometric rate. That is,
$ \| {\bf Q}_{\epsilon , {\bf T}} ^{k}- {\bf P} \| \leq  \Gamma \gamma^{k} $ where $ 0 < \gamma < 1 $ and $ \Gamma > 0 $. \\
(c) $ {\epsilon_{0}({\bf Q},{\bf T})} $ is a necessary and sufficient bound such that for every
$ \epsilon < {\epsilon_{0}({\bf Q},{\bf T})} $ we have (a) and (b) above. \end{theorem}
\textbf{Proof:}

Subsequently, using Lemma~\ref{L9n} to \ref{L13e} and Lemma~\ref{L1} we reach the above result.

\section{Appendix C: Convergence Analysis}\label{convergence}

\subsection{Auxiliary Variables Definitions}

The steps of FLUE are summarized in \eqref{updating_eqn1} -\eqref{updating_eqn2b}. We will prove in Section~\ref{convergence} that all nodes reach consensus to the optimal solution.

For $ 1 \leq l \leq 4n^{2} $, let
$ \hat{{\bf z}}_{l}(k)=\sum_{i=1}^{4n^{2}}[{\bf Q}_{\epsilon , {\bf T }}]_{li}{\bf z}_{i}(k) $

Then

{\small \begin{equation}
\begin{split}
\hat{{\bf z}}_{l}(k+1) = \sum_{i=1}^{4n^{2}}[{\bf Q}_{\epsilon , {\bf T }}]_{li}\sum_{j=1}^{4n^{2}}[{\bf Q}_{\epsilon , {\bf T }}]_{ij}{\bf z}_{j}(k)-\alpha_{k}\sum_{i=1}^{4n^{2}}[{\bf Q}_{\epsilon , {\bf T }}]_{li}\nabla{\overline{g}}_{i}(k), 
\end{split}
\end{equation}}

\begin{equation}
\begin{split}
 \hat{{\bf z}}_{l}(k+1) = \sum_{i=1}^{4n^{2}}[{\bf Q}_{\epsilon , {\bf T }}]_{li}\hat{{\bf z}}_{i}(k)-\alpha_{k}\nabla{f}\jmath({\bf z}_{l}(k)) 
\end{split}
\end{equation}

where the inexact gradient $ \omega_{l}\nabla{f}\jmath({\bf z}_{l}(k)) = \sum_{j \in \Gamma_{l}(k)}d_{j}(k) $, that is $ \nabla{f}\jmath({\bf z}_{l}(k)) $ 
\\
$ = \sum_{j \in \Gamma_{l}(k)} {\bf A}_{fit(l)_n,j}\nabla{\overline{g}}_{j}({\bf z}_{j}(k)) $ for $ 1 \leq l \leq 2n^2 $ where $ w_l $ is defined in \eqref{weight.def}. Also for consistency of notation $ j $ corresponds to the node identified with the index of the nonzero coefficients in the support. See Remark 1.

Then for $ 1 \leq l \leq 2n^2 $
\begin{equation}
\begin{split}
\hat{{\bf z}}_{l}(k+1) = \sum_{i=1}^{4n^{2}}[{\bf Q}_{\epsilon , {\bf T }}]_{li}\hat{{\bf z}}_{i}(k)-\alpha_{k}\sum_{j \in \Gamma_{l}(k)} {\bf A}_{fit(l)_n,j}\nabla{\overline{g}}_{j}({\bf z}_{j}(k)).
\end{split}
\end{equation}

For $ 1 \leq l \leq 2n^{2} $, let
$ \hat{\bar{{\bf z}}}_{l}(k) =  \sum_{i=1}^{4n^{2}} [{\bf P}]_{li}\hat{{\bf z}}_{i}(k) $

Then 
{\small \begin{equation}\label{4n_2n}
\begin{split}
\hat{\bar{{\bf z}}}_{l} & (k+1)  = \sum_{i=1}^{4n^{2}} [{\bf P}]_{li} \hat{{\bf z}}_{i}(k+1) \\
& = \sum_{i=1}^{4n^{2}} [{\bf P}]_{li}(\sum_{j=1}^{4n^{2}}[{\bf Q}_{\epsilon , {\bf T }}]_{ij}\hat{{\bf z}}_{i}(k)-\alpha_{k}\sum_{j \in \Gamma_{i}(k)} {\bf A}_{fit(i)_n,j}\nabla{\overline{g}}_{j}({\bf z}_{j}(k))) \\
& = \sum_{i=1}^{2n^{2}} [{\bf P}]_{li}(\sum_{j=1}^{2n^{2}}[{\bf Q}_{\epsilon , {\bf T }}]_{ij}\hat{{\bf z}}_{i}(k)-\alpha_{k}\sum_{j \in \Gamma_{i}(k)} {\bf A}_{fit(i)_n,j}\nabla{\overline{g}}_{j}({\bf z}_{j}(k))). 
\end{split}
\end{equation}}

Since $ {\bf P}_{[1:2n^{2}][1:4n^{2}]} = [{\bf 1}_{2n^{2} \times 1}{\boldsymbol \pi}^{T} \ {\bf 0}_{2n^{2} \times 2n^{2}}] $ for the first $ 2n^{2} $ rows, which justifies the use of the third step of (\ref{4n_2n}) in the above equality for $ 2n^{2}+1 \leq i \leq 4n^{2} $.

Therefore,
\begin{equation}\label{z^_i}
\footnotesize
\begin{split}
\hat{\bar{{\bf z}}}_{l}(k+1) = 
\sum_{i=1}^{4n^{2}} [{\bf P}]_{li} \hat{{\bf z}}_{i}(k)-\alpha_{k}\sum_{i=1}^{4n^{2}} [{\bf P}]_{li}\sum_{q \in \Gamma_{i}(k)} {\bf A}_{fit(i)_n,q}\nabla{\overline{g}}_{q}({\bf z}_{q}(k)) 
\end{split}
\end{equation}
for $ 1 \leq l \leq 2n^{2} $, by having $ {\bf P} = \lim_{k \rightarrow \infty}{\bf Q}_{\epsilon , {\bf T }} ^{k} $ and $ [{\bf P}]_{i,j} = 0 $ for $ 1 \leq i \leq 2n^{2} $ and $ 2n^{2} +1 \leq j \leq 4n^{2} $. That is, $ {\bf P}_{[1:2n^{2}][1:4n^{2}]} = [{\bf 1}_{2n^{2} \times 1} {\boldsymbol \pi}^{T} \ {\bf 0}_{2n^{2} \times 2n^{2} }]$

\begin{remark}
We define ${\bf A}_{fit(i)_n,j}$ to mean the entry of $ \hat{\bf A}$ in the row corresponding to node $ i $ and column corresponding to node $ j $. Then we have $ \sum_{j \in \Gamma_{i}} {\bf A}_{fit(i)_n,j} $  which makes it an easier notation for the analysis of the proof. Although we could have used $ \sum_{j}[ \hat{\bf A}]_{ij^{'}} $ where $ j^{'}=j $ for $ 1 \leq j^{'} \leq n $ and $ j^{'} = j+n $ for $ n+1 \leq j^{'} \leq 2n $ to mean the same quantity. Note that in $ \hat{\bf A}$, a node $ i$ corresponds to two rows: row $ i $ and row $ i+n $.
\end{remark} 
 
\subsection{Convergence Theorems}

\begin{theorem}\label{Thm3}
Suppose that $ {\bf O} = {\bf Q}_{\epsilon , {\bf T }} $ is the matrix defined in \eqref{fixedQepsT} with the parameter $ \epsilon $ satisfying $ \epsilon \in (0, \epsilon_{0}({\bf Q},{\bf I})  ) $ where $ \epsilon_{0}({\bf Q},{\bf I}) $ is defined in \eqref{eqn_e_double}. Then the algorithm defined by $ {\bf z}_{i}(k+1)= \sum_{j=1}^{4n^{2}}[{\bf Q}_{\epsilon , {\bf T }}]_{ij}{\bf z}_{j}(k)-\alpha_{k}\nabla{g}_{i}({\bf z}_{i}(k)) $ converges to the optimal result and consensus over all nodes. That is, for $ 1 \leq i, j \leq 2n $, $  \lim_{k \rightarrow \infty}f({\bf x}_{l}^{+}(2nk+i)) = \lim_{k \rightarrow \infty}f({\bf x}_{l}^{-}(2nk+i)) = \lim_{k \rightarrow \infty}f({\bf x}_{s}^{+}(2nk+i)) = \lim_{k \rightarrow \infty}f({\bf x}_{s}^{-}(2nk+i)) =  f ^{*} $ for $ 1 \leq l, s \leq n $ and $ 1 \leq i \leq 2n $. 
\end{theorem}

\begin{remark}
Note, in this analysis we used the general case of the dependency of $ \Gamma_{i}(k) $ on $ k $ since we are also going to use them to prove convergence in the time-varying cases with other small modifications, as we are going to show, i.e., the static case where all $ \Gamma_{i}(k) =  \Gamma_{i} $ is for surely valid.
\end{remark}

$ \mathbf{Proof:} $
We begin by subsequently proving the following lemmas to get the above result.

\begin{lemma}\label{L10}
Let \textbf{Assumption \ref{Assump1}} holds, then the sequence $ \hat{\bar{{\bf z}}}_{j}(k) $ for $ 1 \leq j \leq 2n^{2} $, defined earlier follows

\begin{equation}\label{eqn_L9}
\begin{split}
 \| & \hat{\bar{{\bf z}}}_{j}(k+1) - {\bf x} \| ^{2} =  \|  \hat{\bar{{\bf z}}}_{j}(k) - {\bf x} \| ^{2} \\
 & - 2 \alpha_{k} \sum_{i=1}^{4n^{2}} [{\bf P}]_{ji} \sum_{q \in \Gamma_{i}(k)} {\bf A}_{fit(i)_n,q}\nabla{\overline{g}}_{q}({\bf z}_{q}(k))(\hat{\bar{{\bf z}}}_{j}(k) - {\bf x})  \\
 + & \alpha_{k}^{2}\| \sum_{i=1}^{4n^{2}} [{\bf P}]_{ji} \sum_{q \in \Gamma_{i}(k)} {\bf A}_{fit(i)_n,q}\nabla{\overline{g}}_{q}({\bf z}_{q}(k)) \| ^{2}.
\end{split}
\end{equation}
\end{lemma}


\textbf{Proof:}\\
We have for $ 1 \leq j \leq 2n^{2} \ (\nabla{g}_{j} \neq 0) $ \\
\begin{equation*}
\hat{\bar{{\bf z}}}_{j}(k+1) = \hat{\bar{{\bf z}}}_{j}(k)
-\alpha_{k}\sum_{i=1}^{4n^{2}} [{\bf P}]_{ji}\sum_{q \in \Gamma_{i}(k)} {\bf A}_{fit(i)_n,q}\nabla{\overline{g}}_{q}({\bf z}_{q}(k)),  
\end{equation*} 

then 
{\small \begin{equation}
\begin{split}
 \|  \hat{\bar{{\bf z}}}_{j}(k+1) & - {\bf x} \| ^{2}  =  \|  \hat{\bar{{\bf z}}}_{j}(k) - {\bf x} \| ^{2} \\
 & - 2 \alpha_{k} \sum_{i=1}^{4n^{2}} [{\bf P}]_{ji} \sum_{q \in \Gamma_{i}(k)} {\bf A}_{fit(i)_n,q}\nabla{\overline{g}}_{q}({\bf z}_{q}(k))(\hat{\bar{{\bf z}}}_{j}(k) - {\bf x} )  \\
 + & \alpha_{k}^{2}\| \sum_{i=1}^{4n^{2}} [{\bf P}]_{ji} \sum_{q \in \Gamma_{i}(k)} {\bf A}_{fit(i)_n,q}\nabla{\overline{g}}_{q}({\bf z}_{q}(k)) \| ^{2}.
\end{split}
\end{equation}}

\begin{lemma}\label{L11}
Let \textbf{Assumption \ref{Assump1}} holds. Then $ {\bf z}_{j}(k) $ and $ \hat{\bar{{\bf z}}}_{j}(k) $ satisfy the following bounds:
\\
For $ 1 \leq j \leq 2n^{2} $ and $ k \geq 1 $,
{\small \begin{equation}\label{eq_L91}
\begin{split}
\| & \hat{\bar{{\bf z}}}_{j}(k) - {\bf z}_{j}(k) \| \leq  
\sum_{l=1}^{4n} \| {\bf z}_{l}(0) \| \Gamma \gamma^{k} 
 + 4n\sum_{r=1}^{k-1} \Gamma \gamma^{k-r} \alpha_{r-1} \| \hat{\bf A} \|_{2,\infty}  G \\
 & + 2 \sum_{r=1}^{k-1} \alpha_{r-1}  \| {\bf P} \|_{2,\infty} \| \hat{\bf A} \|_{2,\infty}   G  
 + \alpha_{k-1}  \| {\bf P} \|_{2,\infty} \| \hat{\bf A} \|_{2,\infty}   G 
  + \alpha_{k-1}  G.
\end{split}
\end{equation}}
\end{lemma}

$ \mathbf{Proof:} $ 

We have
{\small \begin{equation}
\begin{split}
  {\bf z}_{i}(k) = &
   \sum_{j=1}^{4n^{2}}[{\bf Q}_{\epsilon , {\bf T }} ^{k}]_{ij}{\bf z}_{j}(0) - \sum_{r=1}^{k-1}\sum_{j=1}^{4n^{2}}[{\bf Q}_{\epsilon , {\bf T }}^{k-r}]_{ij}\alpha_{r-1}\nabla{\overline{g}}_{j}({\bf z}_{j}(r-1))\\ 
  & - \alpha_{k-1}\nabla{\overline{g}}_{i}({\bf z}_{i}(k-1)).
\end{split}
\end{equation}}
Then by using (\eqref{updating_eqn}) and (\ref{z^_i}) respectively we get 

{\scriptsize \begin{equation}
\begin{split}
\| & \hat{\bar{{\bf z}}}_{j}(k) - {\bf z}_{j}(k) \| \leq  
\|  \sum_{l=1}^{4n}{\bf z}_{l}(0)([{\bf P}]_{jl}-[{\bf Q}_{\epsilon , {\bf T }} ^{k}]_{jl})\| \\ 
 + &\sum_{r=1}^{k-1}\sum_{l=1}^{4n} \| [{\bf P}]_{jl} - [{\bf Q}_{\epsilon , {\bf T }}^{k-r}]_{jl}\|\| \alpha_{r-1}\sum_{q \in \Gamma_{l}(k)} {\bf A}_{fit(l)_n,q}\nabla{\overline{g}}_{l}({\bf z}_{l}(r-1))\| \\
  + &\sum_{r=1}^{k-1}\sum_{l=1}^{4n} \| [{\bf P}]_{jl}\|\| \alpha_{r-1}\sum_{q \in \Gamma_{l}(k)} {\bf A}_{fit(l)_n,q}(\nabla{\overline{g}}_{q}({\bf z}_{q}(r-1))-\nabla{\overline{g}}_{l}({\bf z}_{l}(r-1))) \| \\
  + &\alpha_{k-1}\| \sum_{l=1}^{4n} [{\bf P}]_{jl} \sum_{q \in \Gamma_{l}(k)} {\bf A}_{fit(l)_n,q}\nabla{\overline{g}}_{q}({\bf z}_{q}(k-1)) \| + \alpha_{k-1}\| \nabla{\overline{g}}_{j}({\bf z}_{j}(k-1))\|,
\end{split}
\end{equation}}

since $ | [{\bf P}]_{jl}-[{\bf Q}_{\epsilon , {\bf T }} ^{k}]_{jl} | \leq \Gamma \gamma ^{k} $ (see $ \mathbf{Theorem \ 1 \ (b)}) $
and using the bounds on $ \| {\bf A}_{fit(i)_n} \|_{ 2, \infty},\| {\bf B} \|_{ \infty}, \| \nabla{\overline{g}}_{j} \| $, $ \| {\bf P} \|_{ 2, \infty}  $ ( in fact $ P $ is doubly stochastic so  $ \| {\bf P} \|_{\infty}=1 $)  and $ {\bf z}_{l}(0) $ the result follows.

Let
{\small \begin{equation}
\begin{split}
D_{k} & = \sum_{l=1}^{4n} \| {\bf z}_{l}(0) \| \Gamma \gamma^{k} + 4 n   \| \hat{\bf A} \|_{2,\infty}  \| {\bf B} \|_{2,\infty} \sqrt{n} F \sum_{r=1}^{k-1} \Gamma \gamma^{k-r} \alpha_{r-1}\\
& + 2 \| {\bf P} \|_{2,\infty}  \| \hat{\bf A} \|_{2,\infty}  \| {\bf B} \|_{2,\infty} \sqrt{n} F \sum_{r=1}^{k-1} \alpha_{r-1}   \\ 
 & +  \alpha_{k-1} \| {\bf P} \|_{2,\infty}  \| \hat{\bf A} \|_{2,\infty}  \| {\bf B} \|_{2,\infty} \sqrt{n} F  + \alpha_{k-1}  \| {\bf B} \|_{2,\infty} \sqrt{n} F.
\end{split}
\end{equation}}

\begin{lemma}\label{L12}
Let \textbf{Assumption \ref{Assump1}} holds. Then for $ 1 \leq j \leq 2n^{2} $ we have \\
$ \sum_{k=0}^{\infty} \alpha_{k} \| \hat{\bar{{\bf z}}}_{j}(k) - {\bf z}_{j}(k) \| < \infty $.
\end{lemma}

$ \mathbf{Proof:} $ \\
We have (\ref{eq_L91}) from Lemma~\ref{L11} thus we get

\begin{equation}
\begin{split}
\sum_{k=0}^{\infty} \alpha_{k} \| & \hat{\bar{{\bf z}}}_{j}(k) - {\bf z}_{j}(k) \| \leq  
  \sum_{k=0}^{\infty} \Gamma \alpha_{k}\gamma^{k} \sum_{l=1}^{4n} \| {\bf z}_{l}(0) \|  \\
  &  + 4n \sum_{k=0}^{\infty}\alpha_{k}\sum_{r=1}^{k-1} \Gamma \gamma^{k-r} \alpha_{r-1} \| \hat{\bf A} \|_{2,\infty} G  \\ 
  & + 2 \sum_{k=0}^{\infty}\alpha_{k} \sum_{r=1}^{k-1} \alpha_{r-1} \| {\bf P} \|_{2,\infty} \| \hat{\bf A} \|_{2,\infty}     G  \\
  & +  \sum_{k=0}^{\infty} \alpha_{k}\alpha_{k-1} \| {\bf P} \|_{2,\infty} ^{2}   \| \hat{\bf A} \|_{2,\infty}  G 
  + \sum_{k=0}^{\infty} \alpha_{k} \alpha_{k-1}  G.
\end{split}
\end{equation}

But
\begin{equation*}
\sum_{k=0}^{K} \alpha_{k}\gamma ^{k} \leq \frac{1}{2} \sum_{k=0}^{K} (\alpha_{k}^{2}+ \gamma ^{2k}) \leq \sum_{k=0}^{K} \frac{1}{2} \alpha_{k}^{2}+ \frac{1}{2} \frac{1}{1-\gamma^{2}} < \infty ,
\end{equation*}
since $ \sum_{k=0}^{\infty}\alpha_{k} ^{2} < \infty $ and $ 0 < \gamma < 1 $. 

Similarly, 
\begin{equation*}
\begin{split}
\sum_{k=0}^{K}&\sum_{r=1}^{k-1}  \alpha_{k}\alpha_{r-1} \gamma^{(k-r)} < \frac{1}{2}\sum_{k=0}^{K}\sum_{r=1}^{k-1}\alpha_{k} ^{2}\sum_{r=1}^{k-1}\gamma^{(k-r)} \\
& + \frac{1}{2}\sum_{r=1}^{K-1}\alpha_{r-1}^{2}\sum_{k=r+1}^{K}\gamma^{k} \leq \frac{1}{1-\gamma}\sum_{k=0}^{K}\alpha_{k}^{2}. 
\end{split}
\end{equation*}
Thus,
\begin{equation*}
 \sum_{k=0}^{\infty}\sum_{r=1}^{k-1}\alpha_{k}\alpha_{r-1} \gamma^{(k-r)} \leq \frac{1}{1-\gamma}\sum_{k=0}^{K} \alpha_{k}^{2} < \infty. 
\end{equation*}

Same for  $ \sum_{k=0}^{K}\sum_{r=1}^{k-1}\alpha_{k}\gamma^{2(k-r)} < \infty $ and $ \sum_{k=0}^{K}\sum_{r=1}^{k-1}\alpha_{k}\alpha_{k-1} < \infty $
and $ \| {\bf z}_{l}(0) \| $ bounded (By initialization, hence also bounded space).

But P, A and B are fixed thus $\| {\bf P} \|_{2, \infty},\| {\bf B} \|_{\infty} and \| {\bf A} \|_{2, \infty}$ are bounded. More precisely $ \| {\bf P} \|_{\infty}=1 $ (doubly stochastic matrix) and $\|{\bf A}_{fit(i)n}\|_{\infty} =1$ (row normalized matrix).
 Thus the solution follows.
\\
Let
\begin{equation}\label{Dk}
\begin{split}
\sum_{k=0}^{K} \alpha_{k} \|  \hat{\bar{{\bf z}}}_{j}(k) - {\bf z}_{j}(k) \| \leq \tilde{D}_{K}, 
\end{split}
\end{equation}
where
\begin{equation}\label{Dk}
\begin{split}
 \tilde{D} _{K} =  \sum_{k=0}^{K}\alpha_{k} D_{k}  & = \sum_{l=1}^{4n} \| {\bf z}_{l}(0) \| \Gamma (\sum_{k=0}^{K} \frac{1}{2} \alpha_{k}^{2} + \frac{1}{2} \frac{1}{1-\gamma^{2}})\\
 & + 4 n  \| \hat{\bf A} \|_{2,\infty}  \| {\bf B} \|_{2,\infty} \sqrt{n} F \Gamma \frac{1}{1-\gamma} \sum_{k=0}^{K} \alpha_{k}^{2} \\
& + 2 \| {\bf P} \|_{2,\infty}  \| \hat{\bf A} \|_{2,\infty}  \| {\bf B} \|_{2,\infty} \sqrt{n} F \sum_{k=0}^{K} \alpha_{k}^{2}   \\ 
 & + \| {\bf P} \|_{2,\infty}  \| \hat{\bf A} \|_{2,\infty}  \| {\bf B} \|_{2,\infty} \sqrt{n} F \sum_{k=0}^{K} \alpha_{k}^{2} \\
 & + \| {\bf B} \|_{2,\infty} \sqrt{n} F \sum_{k=0}^{K} \alpha_{k}^{2}.
\end{split}
\end{equation}

By using the inequalities above.

\begin{lemma}\label{l12}
Let \textbf{Assumption \ref{Assump1}} holds. Then
\\
(a) For $ 1 \leq j, q \leq 2n^{2} $ we have
\begin{equation}
\begin{split}
 \sum_{k=0}^{\infty} \alpha_{k} \| {\bf z}_{j}(k) - {\bf z}_{q}(k) \| < \infty, \ and
  \end{split}
\end{equation}
 (b)   \begin{equation}
\lim_{k \rightarrow \infty}\| {\bf z}_{j}(k) - {\bf z}_{q}(k) \|=0. 
\end{equation}
That is $  \lim_{k \rightarrow \infty}{\bf z}_{j}(k)= \lim_{k \rightarrow \infty}{\bf z}_{q}(k) $ for all $ j $ and $ q $ (i.e., $ 1 \leq j,q \leq 2n^{2} $)
\end{lemma}
$ \mathbf{Proof:} $ 

For (a):
{\small \begin{equation}
\begin{split} 
\sum_{k=0}^{\infty}& \alpha_{k} \| {\bf z}_{j}(k) - {\bf z}_{q}(k) \| \leq  \\
  & \sum _{k=0}^{\infty} \alpha_{k}  \sum_{l=1}^{4n} \| [{\bf Q}_{\epsilon , {\bf T }} ^{k}]_{jl}-[{\bf Q}_{\epsilon , {\bf T }} ^{k}]_{ql} \| \| {\bf z}_{l}(0) \| \\
  & + \sum _{k=0}^{\infty} \alpha_{k}  \sum_{r=1}^{k-1}\sum_{l=1}^{4n}\| [{\bf Q}_{\epsilon , {\bf T }}^{k-r}]_{jl} - [{\bf Q}_{\epsilon , {\bf T }}^{k-r}]_{ql} \|\alpha_{r-1} \| \nabla{\overline{g}}_{l}({\bf z}_{l}(r-1)) \| \\
  & + \sum _{k=0}^{\infty} \alpha_{k}   \alpha_{k-1} \| \nabla{\overline{g}}_{j}({\bf z}_{j}(k-1))-\nabla{\overline{g}}_{q}({\bf z}_{q}(k-1)) \|.
\end{split}
\end{equation}}

Using $ {\bf P} $ where $ {\bf P} $ has identical rows for the first $ 2n^{2} $ rows. 

\begin{equation}
\begin{split} 
\sum_{k=0}^{\infty} \alpha_{k}& \| {\bf z}_{j}(k) - {\bf z}_{q}(k) \| \leq  \\
  & \sum _{k=0}^{\infty} \alpha_{k}  \sum_{l=1}^{4n} \| [{\bf Q}_{\epsilon , {\bf T }} ^{k}]_{jl}-[{\bf P}]_{jl} \| \| {\bf z}_{l}(0) \| \\
  &  + \sum _{k=0}^{\infty} \alpha_{k}  \sum_{l=1}^{4n} \| [{\bf P}]_{jl}-[{\bf P}]_{ql} \| \| {\bf z}_{l}(0) \| \\
  & + \sum _{k=0}^{\infty} \alpha_{k}  \sum_{l=1}^{4n} \| [{\bf Q}_{\epsilon , {\bf T }} ^{k}]_{ql}-[{\bf P}]_{ql} \| \| {\bf z}_{l}(0) \| \\ 
  & +   \sum _{k=0}^{\infty} \alpha_{k}  \sum_{r=1}^{k-1}\sum_{l=1}^{4n}\| [{\bf Q}_{\epsilon , {\bf T }}^{k-r}]_{jl} - [{\bf P}]_{jl} \|\alpha_{r-1} \| \nabla{\overline{g}}_{l}({\bf z}_{l}(r-1)) \| \\
  & + \sum _{k=0}^{\infty} \alpha_{k}  \sum_{r=1}^{k-1}\sum_{l=1}^{4n}\| [{\bf P}]_{jl} - [{\bf P}]_{ql} \|\alpha_{r-1} \| \nabla{\overline{g}}_{l}({\bf z}_{l}(r-1)) \| \\
  & + \sum _{k=0}^{\infty} \alpha_{k}  \sum_{r=1}^{k-1}\sum_{l=1}^{4n}\| [{\bf Q}_{\epsilon , {\bf T }}^{k-r}]_{ql} - [{\bf P}]_{ql} \|\alpha_{r-1} \| \nabla{\overline{g}}_{l}({\bf z}_{l}(r-1)) \| \\
  & + \sum _{k=0}^{\infty} \alpha_{k}   \alpha_{k-1} \| \nabla{\overline{g}}_{j}({\bf z}_{j}(k-1))-\nabla{\overline{g}}_{q}({\bf z}_{q}(k-1)) \|.
\end{split}
\end{equation}

But $ [{\bf P}]_{jl}= [{\bf P}]_{ql} $ for $ 1 \leq j,q \leq 2n^{2} $ and $ 1 \leq l \leq 2n^{2} $. (see $ \mathbf{ Theorem \ 1 \ (a)} $).
  
Then the above becomes
\begin{equation}
\begin{split} 
\sum_{k=0}^{\infty} \alpha_{k} \| {\bf z}_{j}(k) - & {\bf z}_{q}(k) \|   \leq  
  2 \sum _{k=0}^{\infty} \alpha_{k}  \sum_{l=1}^{4n} \Gamma \gamma ^{k} \| {\bf z}_{l}(0) \|  \\
&  + 2 \sum _{k=0}^{\infty} \alpha_{k}  \sum_{r=1}^{k-1}\sum_{l=1}^{4n} \Gamma \gamma ^{k-r} \alpha_{r-1} \| \nabla{\overline{g}}_{l}({\bf z}_{l}(r-1)) \| \\ 
  & + \sum _{k=0}^{\infty} \alpha_{k}   \alpha_{k-1} \| \nabla{\overline{g}}_{j}({\bf z}_{j}(k-1))-\nabla{\overline{g}}_{q}({\bf z}_{q}(k-1)) \|.
\end{split}
\end{equation}

Similarly, using bounds as in the proof of Lemma 11 we get 
\begin{equation}
\begin{split} 
\sum_{k=0}^{\infty} \alpha_{k} \| {\bf z}_{j}(k) - {\bf z}_{q}(k) \| < \infty.
  \end{split}
\end{equation} 

Thus (a) follows. \\
For (b):

And similarly as part (a), under the condition that the first $ 2n^{2} $ rows of $ P $ are identical, we get
\begin{equation}
\begin{split} 
\lim_{k \rightarrow \infty} & \| {\bf z}_{j}(k) - {\bf z}_{q}(k) \| \leq  
   2 \sum_{l=1}^{4n} \Gamma \gamma ^{k} \| {\bf z}_{l}(0) \| \\
  & + 2  \sum_{r=1}^{k-1}\sum_{l=1}^{4n} \Gamma \gamma ^{k-r} \alpha_{r-1} \| \nabla{\overline{g}}_{l}({\bf z}_{l}(r-1)) \| \\
  & +  \alpha_{k-1} \| \nabla{\overline{g}}_{j}({\bf z}_{j}(k-1))-\nabla{\overline{g}}_{q}({\bf z}_{q}(k-1)) \|.
\end{split}
\end{equation}

But $ \alpha_{k} \rightarrow 0 $ and $ \gamma_{k} \rightarrow 0 $ as $ k \rightarrow \infty $ and $ \| \nabla{g}_{j} \| \leq G $, then
\begin{equation}
\begin{split} 
\lim_{k \rightarrow \infty}  \| {\bf z}_{j}(k) - {\bf z}_{q}(k) \| \rightarrow 0.
\end{split}
\end{equation}   
Thus, (b) follows.

\begin{remark}
The use of $ \mathbf{ Theorem \ 1 \ (a)} $ in (a) and (b) restricts j, q to $ 1 \leq j, q \leq 2n^{2} $ in (a) and (b) for the result to follow.
\end{remark}

Let
\begin{equation}
\begin{split}
 E_{k} & = 2  \sum_{l=1}^{4n} \Gamma \gamma ^{k} \| {\bf z}_{l}(0) \| + 8 n \sum_{r=1}^{k-1} \Gamma \gamma ^{k-r} \alpha_{r-1}\| {\bf B} \|_{2,\infty} \sqrt{n} F \\
 & +  2 \alpha_{k-1} \| {\bf B} \|_{2,\infty} \sqrt{n} F.
 \end{split}
 \end{equation}
Then
\begin{equation}\label{Ek}
\begin{split}
\sum_{k=0}^{K} \alpha_{k} \| {\bf z}_{j}(k) - {\bf z}_{q}(k) \| \leq  \tilde{E}_{K}, 
 \end{split}
 \end{equation}
where
\begin{equation}\label{Ek}
\begin{split}
\tilde{E}_{K} & = \sum_{k=0}^{K}\alpha_{k}E_{k} \\ 
& = 2 \sum_{l=1}^{4n} \Gamma  \| {\bf z}_{l}(0) \|(\sum_{k=0}^{K} \frac{1}{2} \alpha_{k}^{2}+ \frac{1}{2} \frac{1}{1-\gamma^{2}}) \\
& + 8 n  \Gamma \| {\bf B} \|_{2,\infty} \sqrt{n} F \frac{1}{1-\gamma} \sum_{k=0}^{K} \alpha_{k}^{2} \\
 & +  2  \| {\bf B} \|_{2,\infty} \sqrt{n} F \sum_{k=0}^{K} \alpha_{k}^{2}.
\end{split}
\end{equation}

\begin{lemma}\label{L14}
Let \textbf{Assumption \ref{Assump1}} holds. Then
{\small \begin{equation}\label{eqn_L13}
\begin{split}
 2 \sum_{k=k^{'}}^{\infty}& \alpha_{k} \sum_{i=1}^{4n^{2}} \omega_{i} [{\bf P}]_{ji}    (f({\bf z}_{j}(k)) -  f({\bf x}))   \leq \\
  \| \hat{\bar{{\bf z}}}_{j}(k^{'}) & - {\bf x} \| ^{2} \\
 + \sum_{k=k^{'}}^{\infty}2 \alpha_{k} & \sum_{i=1}^{4n^{2}} [{\bf P}]_{ji}  \sum_{q \in \Gamma_{i}} \|  {\bf A}_{fit(i)_n,q}\nabla{g}_{q}({\bf z}_{q}(k)) \| \| \hat{\bar{{\bf z}}}_{j}(k) - {\bf z}_{j}(k) \| \\
& + 4 \sum_{k=k^{'}}^{\infty} \alpha_{k} \sum_{i=1}^{4n^{2}} [{\bf P}]_{ji} n F \| {\bf z}_{j}(k)- {\bf z}_{q}(k) \| \\
 & + \alpha_{k} ^{2} \| \sum_{i=1}^{4n^{2}} [{\bf P}]_{ji} \sum_{q \in \Gamma_{i}(k)} {\bf A}_{fit(i)_n,q}\nabla{\overline{g}}_{q}({\bf z}_{q}(k)) \| ^{2}.
\end{split}
\end{equation}}
\end{lemma}

$ \mathbf{Proof:} $
Since $ \lim_{k \rightarrow \infty}{\bf z}_{j}(k)=\lim_{k \rightarrow \infty}{\bf z}_{q}(k) $ for $ 1 \leq j,q \leq 2n^{2} $ we can let $ k^{'} $ be such that for all $ k \geq k^{'} $, $ {\bf z}_{j}(k)={\bf z}_{q}(k) $ where $ 1 \leq j,q \leq 2n ^{2} $. (The consensus point) \\
Using (\ref{eqn_L9}) in Lemma~\ref{L10} and summing from $ k=k^{'} $ to $ K $, 
we have
\begin{equation}
\begin{split}
 \| & \hat{\bar{{\bf z}}}_{j}(K) - {\bf x} \| ^{2} =  \|  \hat{\bar{{\bf z}}}_{j}(k^{'}) - {\bf x} \| ^{2} \\
 & - 2 \sum_{k=k^{'}}^{K} \alpha_{k} \sum_{i=1}^{4n^{2}} [{\bf P}]_{ji} \sum_{q \in \Gamma_{i}(k)} {\bf A}_{fit(i)_n,q}\nabla{\overline{g}}_{q}({\bf z}_{q}(k))(\hat{\bar{{\bf z}}}_{j}(k) - {\bf x} )  \\
 + & \sum_{k=k^{'}}^{K} \alpha_{k}^{2}\| \sum_{i=1}^{4n^{2}} [{\bf P}]_{ji} \sum_{q \in \Gamma_{i}(k)} {\bf A}_{fit(i)_n,q}\nabla{\overline{g}}_{q}({\bf z}_{q}(k)) \| ^{2}.
\end{split}
\end{equation}
Letting $ K \rightarrow \infty $ we get
\begin{equation}\label{e1}
\begin{split}
0 & =  \|  \hat{\bar{{\bf z}}}_{j}(k^{'}) - {\bf x} \| ^{2} \\
 & - 2 \sum_{k=k^{'}}^{\infty} \alpha_{k} \sum_{i=1}^{4n^{2}} [{\bf P}]_{ji} \sum_{q \in \Gamma_{i}(k)} {\bf A}_{fit(i)_n,q}\nabla{\overline{g}}_{q}({\bf z}_{q}(k))(\hat{\bar{{\bf z}}}_{j}(k) - {\bf x} )  \\
 + & \sum_{k=k^{'}}^{\infty} \alpha_{k}^{2}\| \sum_{i=1}^{4n^{2}} [{\bf P}]_{ji} \sum_{q \in \Gamma_{i}(k)} {\bf A}_{fit(i)_n,q}\nabla{\overline{g}}_{q}({\bf z}_{q}(k)) \| ^{2}.
\end{split}
\end{equation}

But for $ k \geq k^{'} $ we have ${\bf z}_{j}(k)={\bf z}_{q}(k)$ where $ 1 \leq j,q \leq 2n^{2} $, so we have $ \sum_{q \in \Gamma_{i}(k)} {\bf A}_{fit(i)_n,q}\nabla{\overline{g}}_{q}({\bf z}_{q}(k)) = \omega_{i} \nabla{f}({\bf z}_{q}(k)) $. Therefore, the above becomes
\begin{equation}
\begin{split}
 2 \sum_{k=k^{'}}^{\infty} \alpha_{k} \sum_{i=1}^{4n^{2}}& [{\bf P}]_{ji} \omega_{i} \langle \nabla{f}({\bf z}_{q}(k)),(\hat{\bar{{\bf z}}}_{j}(k) - {\bf x} ) \rangle  =  \|  \hat{\bar{{\bf z}}}_{j}(k^{'}) - {\bf x} \| ^{2} \\ 
 + & \sum_{k=k^{'}}^{\infty} \alpha_{k}^{2}\| \sum_{i=1}^{4n^{2}} [{\bf P}]_{ji} \nabla{f}({\bf z}_{q}(k)) \| ^{2}.
\end{split}
\end{equation}

And we have
\begin{equation}
\begin{split}
 \langle \nabla{f}({\bf z}_{q}(k)),(\hat{\bar{{\bf z}}}_{j}(k) - {\bf x} ) \rangle   
& = \langle \nabla{f}({\bf z}_{q}(k)),(\hat{\bar{{\bf z}}}_{j}(k) - {\bf z}_{j}(k) ) \rangle \\
& + \langle \nabla{f}({\bf z}_{q}(k))),({\bf z}_{j}(k) - {\bf x} ) \rangle \\
& \geq - \|\nabla{f}({\bf z}_{q}(k)) \| \| \hat{\bar{{\bf z}}}_{j}(k) - {\bf z}_{j}(k) \| \\
& + \nabla{f}({\bf z}_{q}(k)))^{'}( {\bf z}_{j}(k) - {\bf x} ).
\end{split}
\end{equation}

And since $ f $ is assumed to be convex,

\begin{equation}
\begin{split}
\langle (\nabla{f}({\bf z}_{q}(k)),({\bf z}_{j}(k) - {\bf x} ) \rangle & \geq  f({\bf z}_{j}(k)) - f({\bf x}) \\
& - 2 n F \| {\bf z}_{j}(k)- {\bf z}_{q}(k) \|.
\end{split}
\end{equation}


By combining the above, we have 
\begin{equation}
\begin{split}
\langle \nabla{f}({\bf z}_{q}(k)),(\hat{\bar{{\bf z}}}_{j}(k) - {\bf x} ) \rangle
 & \geq  - \| \nabla{f}({\bf z}_{q}(k)) \|  \| \hat{\bar{{\bf z}}}_{j}(k) - {\bf z}_{j}(k) \| \\
 & - 2 n F \| {\bf z}_{j}(k)- {\bf z}_{q}(k) \| \\
& + f({\bf z}_{j}(k)) - f({\bf x}), 
\end{split}
\end{equation}
which can also be written
\begin{equation}
\begin{split}
\langle \nabla{f}({\bf z}_{q}(k)), & (\hat{\bar{{\bf z}}}_{j}(k) - {\bf x} ) \rangle
  \geq \\
 & - \| \sum_{q \in \Gamma_{i}}   {\bf A}_{fit(i)_n,q}\nabla{g}_{q}({\bf z}_{q}(k)) \|  \| \hat{\bar{{\bf z}}}_{j}(k) - {\bf z}_{j}(k) \| \\
 & - 2 n F \| {\bf z}_{j}(k)- {\bf z}_{q}(k) \| \\
& + f({\bf z}_{j}(k)) - f({\bf x}). 
\end{split}
\end{equation}

That is,
\begin{equation}\label{f1}
\begin{split}
& \langle \nabla{f}({\bf z}_{q}(k)),(\hat{\bar{{\bf z}}}_{j}(k) - x ) \rangle \geq \\
&  -  \sum_{q \in \Gamma_{i}} \|  {\bf A}_{fit(i)_n,q}\nabla{g}_{q}({\bf z}_{q}(k)) \|  \| \hat{\bar{{\bf z}}}_{j}(k) - {\bf z}_{j}(k) \|\\
& - 2 n F \| {\bf z}_{j}(k)- {\bf z}_{q}(k) \|\\
& + f({\bf z}_{j}(k)) - f({\bf x}). 
\end{split}
\end{equation}

Then by substituting (\ref{f1}) in (\ref{e1}) the result will follow.

\begin{remark}
\textbf{Remark on Proving Convergence:} \\
If we take $ k $ before consensus point $ k^{'} $ then we need the convexity of each $ g_{q} $ for the analysis to follow, and having $ {\bf z}_{j}(k) $ in its argument will make the local gradients form an exact global gradient $ \nabla{f}({\bf z}_{j}(k)) $ which we can use to prove that $ \lim_{k \rightarrow \infty}\inf(f({\bf z}_{j}(k))-f({\bf x}^{*}))=0 $, thus proving convergence of the algorithm.
But if we want only to use the convexity of the global function $ f $ without any further restriction on the local coded functions you need to work after consensus point $ k^{'} $. That way by using the local functions $ g_{q} $ and since $ \sum_{q \in \Gamma_{i}}   {\bf A}_{fit(i)_n,q}\nabla{g}_{q}({\bf z}_{q}(k)) $ in the consensus region so they form an exact $ \nabla{f} $ at the point $ {\bf z}_{j}(k) $ since all $ {\bf z}_{q}(k) $ for $ 1 \leq q \leq 2n $ are the same (i.e., consenus). Then by using the convexity of $ f $ and the un-boundedness of the tail of the sum of $ \alpha_{k} $ (i.e., $ \sum_{k=k^{'}}^{\infty}\alpha_{k}=\infty $) we are able to prove convergence.
N.B. The updating equations are written explicitly in terms of the local functions $ g_{i} $, thus this would be the initial form of the equations before we are able to use the global function $ f $ in any route we follow in our analysis.
\end{remark}

\begin{lemma}\label{L15}
From Lemmas ~\ref{L12}, \ref{l12}(a), \ref{l12}(b) and \ref{L14} and Assumption 1(d), we have for $ 1 \leq i, j \leq 2n^{2} $, $  \lim_{k \rightarrow \infty}f({\bf z}_{i}(k)) = \lim_{k \rightarrow \infty}f({\bf z}_{j}(k)) = f ^{*} $. 
\end{lemma}

$ \mathbf{Proof:} $
Take $ {\bf x} = {\bf x} ^{*} $ the optimal value, in Lemma 13.
Inspecting the RHS of the inequality of (\ref{eqn_L13}), we have: \\
$ \| \hat{\bar{{\bf z}}}_{j}(k^{'}) - {\bf x}^{*} \| ^{2} < \infty $ as any estimate and the solution are fixed (Bounded space). And since from Lemma~\ref{L12} we have $ \sum_{k=0}^{\infty}\alpha_{k} \| {\bf z}_{j}(k)-\hat{\bar{{\bf z}}}_{j}(k) \| < \infty $ for the involved estimates $ 1 \leq i \leq 2n^{2} $. And $ \sum_{k=0}^{\infty}\alpha_{k} \| {\bf z}_{j}(k)-{\bf z}_{q}(k) \| < \infty $ from Lemma~\ref{l12}(a). And the fourth term of (\ref{eqn_L13}) bounded by Assumption 1(d) and the fixed $ {\bf P} $ and $ \hat{\bf A} $. Then we get that the LHS is finite, that is \\
$  2 \sum_{k=0}^{\infty} \alpha_{k} \sum_{i=1}^{4n^{2}} \omega_{i} [{\bf P}]_{ji} \sum_{q \in \Gamma_{i}(k)}   ({\bf A}_{fit(i)_n,q}g_{q}({\bf z}_{j}(k)) - {\bf A}_{fit(i)_n,q}g_{q}({\bf x}^{*}) )  < \infty $. \\ 

But from Lemma~\ref{l12}(b) we have ${\bf z}_{j}(k)={\bf z}_{q}(k)$ for $ 1 \leq j,q \leq 2n^{2} $ and we have $f({\bf z}_{j}(k))= \sum_{q \in \Gamma_{i}(k)}   ({\bf A}_{fit(i)_n,q}g_{q}({\bf z}_{j}(k))$. Similarly, \\
$f({\bf x}^{*})= \sum_{q \in \Gamma_{i}(k)}  ({\bf A}_{fit(i)_n,q}g_{q}({\bf z}_{j}(k)) - {\bf A}_{fit(i)_n,q}g_{q}({\bf x}^{*}) )$. But for $ k \geq 0 $, we have $  \omega_{i} > 0 $, $ f({\bf z}_{j}(k)) - f({\bf x}^{*}) \geq 0 $ and $ \sum_{k=0}^{\infty}\alpha_{k}=\infty $, $ \alpha_{k} < \infty $ (i.e, this implies $ \sum_{k=k^{'}}^{\infty}\alpha_{k}=\infty $, then we get $ \lim_{k \rightarrow \infty}\inf(f({\bf z}_{j}(k))-f({\bf x}^{*}))=0$.
Thus, $ \lim_{k \rightarrow \infty}f({\bf z}_{j}(k)) = f ^{*} $.

\textbf{Therefore, by Theorem~\ref{Thm3e} and Lemma~\ref{L10} to Lemma~\ref{L15}, the algorithm converges to the optimal result and consensus over all nodes. That is, for $ 1 \leq i, j \leq 2n $, $  \lim_{k \rightarrow \infty}f({\bf x}_{l}^{+}(2nk+i)) = \lim_{k \rightarrow \infty}f({\bf x}_{l}^{-}(2nk+i)) = \lim_{k \rightarrow \infty}f({\bf x}_{s}^{+}(2nk+i)) = \lim_{k \rightarrow \infty}f({\bf x}_{s}^{-}(2nk+i)) =  f ^{*} $ for $ 1 \leq l, s \leq n $ and $ 1 \leq i \leq 2n $. . That is Theorem~\ref{Thm3} follows.}

\begin{theorem}\label{Thm5}
Suppose that $ {\bf O} = {\bf Q}_{\epsilon , {\bf T}} $ is the matrix defined in \eqref{fixedQepsT} with the parameter $ \epsilon $ satisfying $ \epsilon \in (0, {\epsilon_{0}({\bf Q},{\bf T})}  ) $ where $ {\epsilon_{0}({\bf Q},{\bf T})} $ is defined in (\ref{eqn_e_double}). Then the algorithm defined by $ {\bf z}_{i}(k+1)= \sum_{j=1}^{4n^{2}}[{\bf Q}_{\epsilon , {\bf T}}]_{ij}{\bf z}_{j}(k)-\alpha_{k}\nabla{g}_{i}({\bf z}_{i}(k)) $ converges to the optimal result and consensus over all nodes. That is, for $ 1 \leq i, j \leq 2n $, $  \lim_{k \rightarrow \infty}f({\bf x}_{l}^{+}(2nk+i)) = \lim_{k \rightarrow \infty}f({\bf x}_{l}^{-}(2nk+i)) = \lim_{k \rightarrow \infty}f({\bf x}_{s}^{+}(2nk+i)) = \lim_{k \rightarrow \infty}f({\bf x}_{s}^{-}(2nk+i)) =  f ^{*} $ for $ 1 \leq l, s \leq n $ and $ 1 \leq i \leq 2n $. 
\end{theorem}

\textbf{Proof:} \\
By using Theorem~\ref{Thm3e} and through the use of Lemma~\ref{L10} to Lemma~\ref{L15} with the replacement of $ {\bf Q}_{\epsilon , {\bf T }} $ by $ {\bf Q}_{\epsilon , {\bf T}} $, $ {\bf P} = \lim_{k \rightarrow \infty} {\bf Q}_{\epsilon , {\bf T }} $ by $ {\bf P} = \lim_{k \rightarrow \infty} {\bf Q}_{\epsilon , {\bf T}} $ , the proof follows with $ {\bf O} = {\bf Q}_{\epsilon , {\bf T}} = {\bf Q} + \epsilon {\bf G} $.

\section{Appendix D: Convergence Analysis for Time-varying (Random) $ \hat{\bf A} $ and $ \hat{\bf D} $ }\label{Sec-Time-Varying-Conv}

We analyze the convergence for the $ {\bf Q}_{\epsilon , {\bf T}} $ variant in \eqref{updating_eqn}.

That is, 
\begin{equation}\label{varyingQepsT}
{\bf Q}_{\epsilon , {\bf T}}(k)= {\bf Q}^{(k)} + \epsilon {\bf G}(k) 
\end{equation}

where $ {\bf Q}^{(k)} $, $ {\bf G}(k) $ and $ {\bf T}(k) $ structures are shown in the final page.

\begin{theorem}\label{Thm3etv}
Suppose $ {\bf Q}_{\epsilon , {\bf T}}(k) $ is the matrix defined in \eqref{varyingQepsT} with $ {\bf T}(k) $ instead of $ {\bf T} $ is $ {\bf T}(k) = {\bf I}_{n \times n} - \hat{\bf D} (k) $ or $ {\bf T}(k) = {\bf I}_{n \times n} - \hat{\bf D} (k) ^{-1} $ with the parameter $ \epsilon $ satisfying $ \epsilon \in (0, \min_{{\bf Q}^{(k)}}{\epsilon_{0}({\bf Q}^{(k)},{\bf T})}  ) $ where $ {\epsilon_{0}({\bf Q}^{(k)},{\bf T})} $ is defined in (\ref{eqn_e_double}) analogously for $ {\bf Q}^{(k)} $ instead of $ {\bf Q} $, with $ \lambda_{4}({\bf Q}^{(k)}) $ the fourth largest eigenvalue of matrix $ {\bf Q}_{\epsilon , {\bf T}}(k) $ by setting $ \epsilon = 0 $. Then \\
(a) $ \lim_{l\rightarrow \infty}\prod_{k=1}^{l}{\bf Q}_{\epsilon , {\bf T}}(k)  \rightarrow {\bf P}  $. Specifically, $ [\lim_{l\rightarrow \infty}\prod_{k=1}^{l}{\bf Q}_{\epsilon , {\bf T}}(k)] _{[1:2n^{2}][1:2n^{2}]} = {\bf 1}_{2n^{2}} {\boldsymbol \pi}^{T} $, 
and \\
$ [\lim_{l\rightarrow \infty}\prod_{k=1}^{l}{\bf Q}_{\epsilon , {\bf T}}(k)] _{[1:2n^{2}][2n^{2}+1:4n^{2}]}= {\bf 0}_{2n^{2} \times 2n^{2}} $, where $ [{\bf Q}_{\epsilon , {\bf T}}(k)]_{[1:2n^{2}][1:2n^{2}]} = \\
 \hat{\bf A}(k) $, and \\
$ [{\bf Q}_{\epsilon , {\bf T}}(k)]_{[1:2n^{2}][2n^{2}+1:4n^{2}]} = \epsilon  {\bf T(k)} $
(b) For all $ i, j \in V, \ [\lim_{l\rightarrow \infty}\prod_{k=1}^{l}{\bf Q}_{\epsilon , {\bf T}}(k)]_{ij} $ converge to $ {\bf P} $ as $ k \rightarrow \infty $ at a geometric rate. That is,
$ \| \prod_{k=1}^{n}{\bf Q}_{\epsilon , {\bf T}}(k) - {\bf P} \| \leq \tilde{\Gamma}  \tilde{\gamma^{n}} $ where $ S_{\hat{\bf A}} $ is the set of infinitely occurring matrices $ \hat{\bf A} $ is finite (i.e., $ \hat{\bf A} \in S_{\hat{\bf A}} $ where $ | S_{\hat{\bf A}} | < \infty $) and $ 0 < \tilde{\gamma} < 1 $ and $ \tilde{\Gamma} > 0 $. \\
(c) $ {\epsilon_{0}({\bf Q},{\bf T})} $ is a necessary and sufficient bound such that for every
$ \epsilon < \max_{k}{\epsilon_{0}({\bf Q},{\bf T})(k)} $ we have (a) and (b) above. \end{theorem}

\begin{proof}
By ordinary matrix multiplication $ [\lim_{l\rightarrow \infty}\prod_{k=1}^{l}{\bf Q}_{\epsilon , {\bf T}}(k)] _{[1:2n^{2}][1:2n^{2}]} =  \lim_{l\rightarrow \infty}\prod_{k=1}^{l}\hat{\bf A}(k) + \lim_{l\rightarrow \infty} \epsilon^{l} h_{1}(\hat{\bf A}(k))= \lim_{l\rightarrow \infty}\prod_{k=1}^{l}\hat{\bf A}(k)= {\bf 1}_{2n} {\boldsymbol \pi}^{T} $ since each $ \hat{\bf A}(k) \in S_{\hat{\bf A}} $ is SIA and $ | S_{\hat{\bf A}} | < \infty $  (i.e., through utilizing SIA Theorem~1 (2) in \citep{Xia2015}, product of finite sets of SIA matrices is SIA). 
And $ [\lim_{l\rightarrow \infty}\prod_{k=1}^{l}{\bf Q}_{\epsilon , {\bf T}}(k)] _{[1:2n^{2}][2n^{2}+1:4n^{2}]} =  \lim_{l\rightarrow \infty} \epsilon^{l} h_{2}(\hat{\bf A}(k))= {\bf 0}_{2n^{2} \times 2n^{2}} $.

Since from the analysis of the fixed $ \hat{\bf A} $ part we showed that the eigenstructure of $ {\bf Q}_{\epsilon , {\bf T}} $ is the same as the eigenstructure of $ {\bf Q} $ then each matrix $  {\bf Q}_{\epsilon , {\bf T}}(k) $ in the time-varying case has the same eigenstructure of a corresponding matrix $  {\bf Q}^{(k)} $ where matrices of different $ k $ can have the same or different matrix $ {\bf Q}^{(k)} $ (i.e., valid for both cases). Then $ \rho({\bf Q}_{\epsilon , {\bf T}}(k))) = \rho({\bf Q}^{(k)}) = 1 $ from Proposition. And eigenvalue $ 1 $ of $ {\bf Q}_{\epsilon , {\bf T}}(k) $ has right eigenvector $ \left(\begin{array}{c}
{\bf 1}_{2n^{2} \times 1} \\ {\bf 0}_{2n^{2} \times 1} \end{array}\right) $ and left eigenvectors $ \left(\begin{array}{c}
{\bf 0}_{2n^{2} \times 1} \\ \bar{\bf 1}^{+} 
\end{array}\right) $ and $ \left(\begin{array}{c}
{\bf 0}_{2n^{2} \times 1} \\ \bar{\bf 1}^{-}\end{array}\right) $, where the $ 2n^{2} \times 1 $ vector
\begin{equation}
 \bar{\bf 1}^{+} (j) = \begin{cases}  1   \ \ \  & if \ 1 \leq j \mod 2n \leq n \\
 0  \ \ \ & otherwise \end{cases}
\end{equation}
and the $ 2n^{2} \times 1 $ vector
\begin{equation}
 \bar{\bf 1}^{-} (j) = \begin{cases}  1    \ \ \  & if \  n + 1 \leq j \mod 2n \leq 2 n \\
 0  \ \ \ & otherwise \end{cases}
\end{equation}
Then the spectral radius $ \rho( \lim_{n \rightarrow \infty} \prod_{k=1}^{n}{\bf Q}_{\epsilon , {\bf T}}(k)) \leq \lim_{n \rightarrow \infty} \prod_{k=1}^{n} \rho({\bf Q}_{\epsilon , {\bf T}}(k)) =  \lim_{n \rightarrow \infty} \prod_{k=1}^{n} \rho({\bf Q}^{(k)}) = 1 $. And $ \lim_{n \rightarrow \infty} \prod_{k=1}^{n}{\bf Q}_{\epsilon , {\bf T}}(k) $ has eigenvalue $ 1 $ with the same eigenvectors: i.e.,  $ \left(\begin{array}{c} {\bf 1}_{2n^{2} \times 1} \\ {\bf 0}_{2n^{2} \times 1} \end{array}\right) $ and left eigenvectors $ \left(\begin{array}{c}
{\bf 0}_{2n^{2} \times 1} \\ \bar{\bf 1}^{+} 
\end{array}\right) $ and $ \left(\begin{array}{c}
{\bf 0}_{2n^{2} \times 1} \\ \bar{\bf 1}^{-}\end{array}\right) $ where the algebraic multiplicity is equal to the geometric multiplicity (i.e., Induction on $ k $ where if two matrices $ {\bf A} $ and $ {\bf B} $ have the same eigenvector $ {\bf v} $ for eigenvalues $ \lambda $ and $ \mu $, respectively then $ {\bf A}{\bf B} {\bf v} = {\bf A} \mu {\bf v} = \lambda \mu {\bf v} $. That is the product has eigenvalue $ \lambda \mu $ with corresponding eigenvector $ {\bf v} $. In our case, all eigenvalues $ \lambda(k) $ in consideration are equal to $ 1 $, so their product is $ 1 $).
Thus, the infinite product $  \prod_{k=1}^{n}{\bf Q}_{\epsilon , {\bf T}}(k) $ has semi-simple eigenvalue $ 1 $ with its algebraic multiplicity equals to its geometric multiplicity equals to $ 3 $. Consequently,
$  \| \prod_{k=1}^{n}{\bf Q}_{\epsilon , {\bf T}}(k) - {\bf P} \| \leq  \tilde{\Gamma}  \tilde{\gamma^{n}} $ where $ S_{\hat{\bf A}} $ is the set of infinitely occurring matrices $ \hat{\bf A} $ and $ | S_{\hat{\bf A}} | < \infty $.
\end{proof}

\begin{theorem}\label{Thm5}
Suppose that $ {\bf O}(k) = {\bf Q}_{\epsilon , {\bf T}}(k) $ is the matrix defined in (\ref{matrixO}) with the parameter $ \epsilon $ satisfying $ \epsilon \in (0, {\epsilon_{0}({\bf Q},{\bf T})}  ) $ where $ {\epsilon_{0}({\bf Q},{\bf T})} $ is defined in (\ref{eqn_e_double}). Then the algorithm defined by $ {\bf z}_{i}(k+1)= \sum_{j=1}^{4n}[{\bf Q}_{\epsilon , {\bf T}}(k)]_{ij}{\bf z}_{j}(k)-\alpha_{k}\nabla{g}_{i}({\bf z}_{i}(k)) $ converges to the optimal result and consensus over all nodes. That is, for $ 1 \leq i, j \leq 2n $, $  \lim_{k \rightarrow \infty}f({\bf x}_{l}^{+}(2nk+i)) = \lim_{k \rightarrow \infty}f({\bf x}_{l}^{-}(2nk+i)) = \lim_{k \rightarrow \infty}f({\bf x}_{s}^{+}(2nk+i)) = \lim_{k \rightarrow \infty}f({\bf x}_{s}^{-}(2nk+i)) =  f ^{*} $ for $ 1 \leq l, s \leq n $ and $ 1 \leq i \leq 2n $. 
\end{theorem} 
 
\begin{proof} 
To prove convergence we utilize again Theorem 1 and lemmas 15 through 20 but first we need to check their validity for the time-varying case.
The Lemmas are all valid with the replacement of the product $ {\bf Q}(k)^{n} $ by $ \prod_{k=0}^{n}{\bf Q}(k) $ since in all of these lemmas the product begins from $ k=0 $. The essential Lemma 18 becomes valid also in the time varying case since the two requirements of the matrix $ {\bf P} $ of being of identical first $ 2n $ rows and its relation with decay of $ \| {\bf P} - \prod_{k=0}^{n}{\bf Q}(k) \| \leq \Gamma \gamma^{n} $ is also true. All other lemmas are also true based on the previously discussed behavior of $ {\bf P} $ in the time-varying case which is validated in the above analysis of this section.
Therefore, we have the following theorem.
\end{proof}
\qed

\begin{remark}
We have analyzed the convergence proof here for the variant corresponding to the matrix $ {\bf Q}_{\epsilon , {\bf T}} $ because the structure of this matrix satisfies the described characteristics surely through theoretical analysis. We could as well have use the matrix $ {\bf Q}_{\epsilon , {\bf I}} $ but in that case the characteristics are based on untight bounds and are only verified empirically.
\end{remark}

\section{Appendix E: Rate of Convergence}\label{rateofconv} 

The convergence rate analysis is similar for both cases, the original $ {\bf O} = {\bf Q}_{\epsilon , {\bf I}} $ algorithm and $ {\bf O} = {\bf Q}_{\epsilon , {\bf T}} $ algorithm, as both the original $ {\bf Q}_{\epsilon , {\bf I}} ^ {k} $ and $ {\bf Q}_{\epsilon , {\bf T}} ^{k} $ have the same limit $ {\bf P} = \lim_{k \rightarrow \infty} {\bf Q} ^{k} $.

We divide finding the rate of convergence based on the behavior of the algorithm before the first consensus point $ k^{'} $ which is influenced by the local coded functions $ g_{i} $ and its behavior after that point which is affected by the global function $ f $ .


\subsection{Convergence Rate in the Consensus Region}

The convergence rate in the consensus region is governed by inequality (\ref{eqn_L13}) of Lemma 13. We have

{\small \begin{equation}\label{eqn_L13_1}
\begin{split}
 2 \sum_{k=k^{'}}^{K}& \alpha_{k} \sum_{i=1}^{4n^{2}} \omega_{i} [{\bf P}]_{ji}    (f({\bf z}_{j}(k)) -  f({\bf x}^{*}))   \leq \\
 &  \| \hat{\bar{{\bf z}}}_{j}(k^{'}) - {\bf x}^{*} \| ^{2} - \| \hat{\bar{{\bf z}}}_{j}(K) - {\bf x}^{*} \| ^{2} \\
 & + \sum_{k=k^{'}}^{K}2 \alpha_{k}  \| \sum_{i=1}^{4n^{2}} [{\bf P}]_{ji} \sum_{q \in \Gamma_{i}}{\bf A}_{fit(i)_n,q}\nabla{g}_{q}({\bf z}_{q}(k)) \| \| \hat{\bar{{\bf z}}}_{j}(k) - {\bf z}_{j}(k) \| \\
& + 4 \sum_{k=k^{'}}^{K} \alpha_{k} \| \sum_{i=1}^{4n^{2}} [{\bf P}]_{ji} n F \| \| {\bf z}_{j}(k)- {\bf z}_{q}(k) \| \\
 & + \alpha_{k} ^{2} \| \sum_{i=1}^{4n^{2}} [{\bf P}]_{ji} \sum_{q \in \Gamma_{i}(k)} {\bf A}_{fit(i)_n,q}\nabla{\overline{g}}_{q}({\bf z}_{q}(k)) \| ^{2}.
\end{split}
\end{equation}}

\subsection{ Convergence rate in the Non-consensus Region for Specific Type of Local Functions}

For specific type of local functions we are able to find an explicit behavior structure for the convergence rate.

\begin{lemma}\label{L16}
Let the assumptions hold and each $ \overline{g}_{i} $ is either convex or concave. Then the sequence $ \hat{\bar{{\bf z}}}_{j}(k) $ for $ 1 \leq j \leq 2n^{2} $, defined earlier follows
{\small \begin{equation}\label{eqn_L15}
\begin{split}
\| & \hat{  \bar{z}}_{j}  (k+1)- {\bf x} \| ^{2}  \leq  \| \hat{\bar{{\bf z}}}_{j}(k) - {\bf x} \| ^{2} \\
& + 2 \alpha_{k}  \sum_{i=1}^{4n^{2}} [{\bf P}]_{ji} \sum_{q \in \Gamma_{i}(k)}  \| {\bf A}_{fit(i)_n,q}\nabla{\overline{g}}_{q}({\bf z}_{q}(k)) \| \| \hat{\bar{{\bf z}}}_{j}(k) - {\bf z}_{j}(k) \| \\
& + 4 \alpha_{k}  \sum_{i=1}^{4n^{2}} [{\bf P}]_{ji}  \sum_{q \in \Gamma_{i}(k)} G {\bf A}_{fit(i)_n,q} \| {\bf z}_{j}(k) - {\bf z}_{q}(k) \| \\
& + 4 \alpha_{k}  \sum_{i=1}^{4n^{2}} [{\bf P}]_{ji}  \sum_{q \in \Gamma_{i}(k), \overline{g}_{q} concave} G {\bf A}_{fit(i)_n,q}  \| {\bf x} - {\bf z}_{j}(k) \| \\
& -  2 \alpha_{k} \sum_{i=1}^{4n^{2}} [{\bf P}]_{ji} \sum_{q \in \Gamma_{i}(k)} ({\bf A}_{fit(i)_n,q}\overline{g}_{q}({\bf z}_{j}(k)) - {\bf A}_{fit(i)_n,q}\overline{g}_{q}({\bf x}) )  \\
 & + \alpha_{k} ^{2} \| \sum_{i=1}^{4n^{2}} [{\bf P}]_{ji} \sum_{q \in \Gamma_{i}(k)} {\bf A}_{fit(i)_n,q}\nabla{\overline{g}}_{q}({\bf z}_{q}(k)) \| ^{2}.
\end{split}
\end{equation}}
\end{lemma}

$ \mathbf{Proof:} $ \\
We have for $ 1 \leq j \leq 2n^{2} \ (\nabla{g}_{j} \neq 0) $ \\
\begin{equation*}
\hat{\bar{{\bf z}}}_{j}(k+1) = \hat{\bar{{\bf z}}}_{j}(k)
-\alpha_{k}\sum_{i=1}^{4n^{2}} [{\bf P}]_{ji}\sum_{q \in \Gamma_{i}(k)} {\bf A}_{fit(i)_n,q}\nabla{\overline{g}}_{q}({\bf z}_{q}(k)),  
\end{equation*} 
then 
\begin{equation}\label{o1}
\begin{split}
 \|  \hat{\bar{{\bf z}}}_{j}&(k+1)  - {\bf x} \| ^{2}   =  \|  \hat{\bar{{\bf z}}}_{j}(k) - {\bf x} \| ^{2} \\
  & - 2 \alpha_{k} \sum_{i=1}^{4n^{2}} [{\bf P}]_{ji} \sum_{q \in \Gamma_{i}(k)} {\bf A}_{fit(i)_n,q}\nabla{\overline{g}}_{q}({\bf z}_{q}(k))(\hat{\bar{{\bf z}}}_{j}(k) - {\bf x} )  \\
 & + \alpha_{k}^{2}\| \sum_{i=1}^{4n^{2}} [{\bf P}]_{ji} \sum_{q \in \Gamma_{i}(k)} {\bf A}_{fit(i)_n,q}\nabla{\overline{g}}_{q}({\bf z}_{q}(k)) \| ^{2}.
\end{split}
\end{equation}

But 
\begin{equation}
\begin{split}
\langle ( & {\bf A}_{fit(i)_n,q}\nabla{\overline{g}}_{q}({\bf z}_{q}(k))),  (\hat{\bar{{\bf z}}}_{j}(k) - {\bf x} ) \rangle   \\
& = \langle ({\bf A}_{fit(i)_n,q}\nabla{\overline{g}}_{q}({\bf z}_{q}(k))),(\hat{\bar{{\bf z}}}_{j}(k) - {\bf z}_{j}(k) ) \rangle \\
& \hspace{3cm} + \langle ({\bf A}_{fit(i)_n,q}\nabla{\overline{g}}_{q}({\bf z}_{q}(k))),({\bf z}_{j}(k) - {\bf x} ) \rangle \\
& \geq - \| {\bf A}_{fit(i)_n,q}\nabla{\overline{g}}_{q}({\bf z}_{q}(k)) \| \| \hat{\bar{{\bf z}}}_{j}(k) - {\bf z}_{j}(k) \| \\
& \hspace{3cm} + \langle ({\bf A}_{fit(i)_n,q}\nabla{\overline{g}}_{q}({\bf z}_{q}(k))),( {\bf z}_{j}(k) - {\bf x} ) \rangle.
\end{split}
\end{equation}

For $ \overline{g}_{q} $ convex, we have
\begin{equation}
\begin{split}
\langle ( {\bf A}_{fit(i)_n,q}& \nabla{\overline{g}}_{q}({\bf z}_{q}(k))),({\bf z}_{j}(k) - {\bf x} ) \rangle \geq  {\bf A}_{fit(i)_n,q}\overline{g}_{q}({\bf z}_{j}(k)) \\
& - {\bf A}_{fit(i)_n,q}\overline{g}_{q}({\bf x}) - 2 G {\bf A}_{fit(i)n,q} \| {\bf z}_{j}(k)- {\bf z}_{q}(k) \|.
\end{split}
\end{equation}

 
By combining the above, we have 
\begin{equation}\label{o2}
\begin{split}
\langle ( {\bf A}_{fit(i)_n,q}&\nabla{\overline{g}}_{q}({\bf z}_{q}(k))),(\hat{\bar{{\bf z}}}_{j}(k) - {\bf x} ) \rangle \geq \\
&- \| {\bf A}_{fit(i)_n,q}\nabla{\overline{g}}_{q}({\bf z}_{q}(k)) \|  \| \hat{\bar{{\bf z}}}_{j}(k) - {\bf z}_{j}(k) \|\\
& - 2 G {\bf A}_{fit(i)n,q} \| {\bf z}_{j}(k)- {\bf z}_{q}(k) \|\\
& + {\bf A}_{fit(i)_n,q}\overline{g}_{q}({\bf z}_{j}(k)) - {\bf A}_{fit(i)_n,q}\overline{g}_{q}({\bf x}). 
\end{split}
\end{equation}

Therefore, by substituting (\ref{o2}) in (\ref{o1}) we can modify the part of the second term of (\ref{o1}) concerned with the convex functions.

Similarly, for the concave functions part, we have for $ 1 \leq q \leq 2n^{2} $, if $ g_{q} $ is concave then:

Let $ h_{i} = -\overline{g}_{i} $, then
\begin{equation}
\begin{split}
        \langle {\bf A}_{fit,i,q}& \nabla{\overline{g}}_{q}({\bf z}_{q}(k)),(\hat{\overline{z}}_{j}(k)-{\bf x} ) \rangle = \\
        & - \langle {\bf A}_{fit,i,q}\nabla{h}_{q}({\bf z}_{q}(k)),(\hat{\overline{z}}_{j}(k)-{\bf x} ) \rangle \\
        &= \langle {\bf A}_{fit,i,q}\nabla{h}_{q}({\bf z}_{q}(k)),({\bf x}-\hat{\overline{z}}_{j}(k)) \rangle \\
        &= \langle {\bf A}_{fit,i,q}\nabla{h}_{q}({\bf z}_{q}(k)), ({\bf z}_{j}(k)-\hat{\overline{z}}_{j}(k)) \rangle \\
        & + \langle {\bf A}_{fit,i,q}\nabla{h}_{q}({\bf z}_{q}(k)),({\bf x}-{\bf z}_{j}(k)) \rangle
\end{split}
\end{equation}
$ \implies $
\begin{equation}
\begin{split}
        \langle {\bf A}_{fit,i,q}&\nabla{\overline{g}}_{q}({\bf z}_{q}(k)),(\hat{\overline{z}}_{j}(k)-{\bf x} ) \rangle \geq \\
        & - \|{\bf A}_{fit,i,q}\nabla{h}_{q}({\bf z}_{q}(k)) \| \| {\bf z}_{j}(k)-\hat{\overline{z}}_{j}(k)\| \\
        & + \langle {\bf A}_{fit,i,q}\nabla{h}_{q}({\bf z}_{q}(k)),({\bf x}-{\bf z}_{j}(k)) \rangle.
\end{split}
\end{equation}

But $ h_{q} $ is convex, then
\begin{equation}
\begin{split}
        \langle {\bf A}_{fit,i,q}\nabla{h}_{q} & ({\bf z}_{q}(k)), ({\bf x}-{\bf z}_{j}(k)) \rangle  \geq \\ 
        & {\bf A}_{fit,i,q}h_{q}({\bf x})         - {\bf A}_{fit,i,q}h_{q}({\bf z}_{j}(k)) \\
        & - 2 G {\bf A}_{fit,i,q} \| {\bf x} -{\bf z}_{q}(k) \|.
\end{split}
\end{equation}

Then 
{ \begin{equation}
\begin{split}
        \langle {\bf A}_{fit,i,q}\nabla{\overline{g}}_{q} & ({\bf z}_{q}(k)),(\hat{\overline{z}}_{j}(k)- {\bf x} ) \rangle  \geq \\
         & - \|{\bf A}_{fit,i,q}\nabla{h}_{q}({\bf z}_{q}(k)) \| \| {\bf z}_{j}(k)-\hat{\overline{z}}_{j}(k)\| \\ 
        & + {\bf A}_{fit,i,q}h_{q}({\bf x}) - {\bf A}_{fit,i,q}h_{q}({\bf z}_{j}(k)) \\
        & - 2 G {\bf A}_{fit,i,q} \| {\bf x} -{\bf z}_{q}(k) \|.
\end{split}
\end{equation}}

Substituting back $ g_{i} $ we have 
{ \begin{equation}
\begin{split}
        \langle {\bf A}_{fit,i,q}\nabla{\overline{g}}_{q}({\bf z}_{q}(k)) & ,(\hat{\overline{z}}_{j}(k)-x) \rangle   \geq \\
       & - \|{\bf A}_{fit,i,q}\nabla{h}_{q}({\bf z}_{q}(k)) \| \| {\bf z}_{j}(k)-\hat{\overline{z}}_{j}(k)\| \\
        & + {\bf A}_{fit,i,q}\overline{g}_{q}({\bf z}_{j}(k)) - {\bf A}_{fit,i,q}\overline{g}_{q}({\bf x}) \\
        & - 2 G {\bf A}_{fit,i,q} \| {\bf x} -{\bf z}_{q}(k) \|.
\end{split}
\end{equation}}

Thus, also substituting the concave functions part in (\ref{o1}) we get
{ \begin{equation}
\begin{split}
\|& \hat{  \bar{z}}_{j}(k+1)- {\bf x} \| ^{2} \leq  \| \hat{\bar{{\bf z}}}_{j}(k) - {\bf x} \| ^{2} \\ & + 2 \alpha_{k} \sum_{i=1}^{4n^{2}} [{\bf P}]_{ji} \sum_{q \in \Gamma_{i}(k)}   \| {\bf A}_{fit(i)_n,q}\nabla{\overline{g}}_{q}({\bf z}_{q}(k)) \| \| \hat{\bar{{\bf z}}}_{j}(k) - {\bf z}_{j}(k) \| \\
& + 4 \alpha_{k} \sum_{i=1}^{4n^{2}} [{\bf P}]_{ji} \sum_{q \in \Gamma_{i}(k), \overline{g}_{q} convex} G {\bf A}_{fit(i)_n,q}\| {\bf z}_{j}(k) - {\bf z}_{q}(k) \| \\
& + 4 \alpha_{k} \sum_{i=1}^{4n^{2}} [{\bf P}]_{ji} \sum_{q \in \Gamma_{i}(k), \overline{g}_{q} concave} G {\bf A}_{fit(i)_n,q}\| {\bf x} - {\bf z}_{q}(k) \| \\
& -  2 \alpha_{k} \sum_{i=1}^{4n^{2}} [{\bf P}]_{ji} \sum_{q \in \Gamma_{i}(k)} ({\bf A}_{fit(i)_n,q}\overline{g}_{q}({\bf z}_{j}(k)) - {\bf A}_{fit(i)_n,q}\overline{g}_{q}({\bf x}) )  \\
 & + \alpha_{k} ^{2} \| \sum_{i=1}^{4n^{2}} [{\bf P}]_{ji} \sum_{q \in \Gamma_{i}(k)} {\bf A}_{fit(i)_n,q}\nabla{\overline{g}}_{q}({\bf z}_{q}(k)) \| ^{2}.
\end{split}
\end{equation}}

Then by adding and subtracting \\
$ 4 \alpha_{k} \sum_{i=1}^{4n^{2}} [{\bf P}]_{ji} \sum_{q \in \Gamma_{i}(k), \overline{g}_{q} concave}^{|n-s|} G {\bf A}_{fit(i)_n,q}\| {\bf z}_{j}(k) - {\bf z}_{q}(k) \| $, we have
\begin{equation}
\begin{split}
\|& \hat{  \bar{z}}_{j}(k+1)- {\bf x} \| ^{2} \leq  \| \hat{\bar{{\bf z}}}_{j}(k) - {\bf x} \| ^{2} \\ & + 2 \alpha_{k} \sum_{i=1}^{4n^{2}} [{\bf P}]_{ji} \sum_{q \in \Gamma_{i}(k)}   \| {\bf A}_{fit(i)_n,q}\nabla{\overline{g}}_{q}({\bf z}_{q}(k)) \| \| \hat{\bar{{\bf z}}}_{j}(k) - {\bf z}_{j}(k) \| \\
& + 4 \alpha_{k} \sum_{i=1}^{4n^{2}} [{\bf P}]_{ji} \sum_{q \in \Gamma_{i}(k)} G {\bf A}_{fit(i)_n,q}\| {\bf z}_{j}(k) - {\bf z}_{q}(k) \| \\
& + 4 \alpha_{k} \sum_{i=1}^{4n^{2}} [{\bf P}]_{ji} \sum_{q \in \Gamma_{i}(k), \overline{g}_{q} concave} G {\bf A}_{fit(i)_n,q}\| {\bf x} - {\bf z}_{q}(k) \| \\
& - 4 \alpha_{k} \sum_{i=1}^{4n^{2}} [{\bf P}]_{ji} \sum_{q \in \Gamma_{i}(k), \overline{g}_{q} concave} G {\bf A}_{fit(i)_n,q}\| {\bf z}_{j}(k) - {\bf z}_{q}(k) \| \\
& -  2 \alpha_{k} \sum_{i=1}^{4n^{2}} [{\bf P}]_{ji} \sum_{q \in \Gamma_{i}(k)} ({\bf A}_{fit(i)_n,q}\overline{g}_{q}({\bf z}_{j}(k)) - {\bf A}_{fit(i)_n,q}\overline{g}_{q}({\bf x}) )  \\
 & + \alpha_{k} ^{2} \| \sum_{i=1}^{4n^{2}} [{\bf P}]_{ji} \sum_{q \in \Gamma_{i}(k)} {\bf A}_{fit(i)_n,q}\nabla{\overline{g}}_{q}({\bf z}_{q}(k)) \| ^{2}.
\end{split}
\end{equation}

By the triangle difference inequality we get
\begin{equation}
\begin{split}
\| & \hat{  \bar{z}}_{j}(k+1)- {\bf x} \| ^{2}  \leq  \| \hat{\bar{{\bf z}}}_{j}(k) - {\bf x} \| ^{2} \\
& + 2 \alpha_{k} \sum_{i=1}^{4n^{2}} [{\bf P}]_{ji} \sum_{q \in \Gamma_{i}(k)}   \| {\bf A}_{fit(i)_n,q}\nabla{\overline{g}}_{q}({\bf z}_{q}(k)) \| \| \hat{\bar{{\bf z}}}_{j}(k) - {\bf z}_{j}(k) \| \\
& + 4 \alpha_{k} \sum_{i=1}^{4n^{2}} [{\bf P}]_{ji} \sum_{q \in \Gamma_{i}(k)} G {\bf A}_{fit(i)_n,q}\| {\bf z}_{j}(k) - {\bf z}_{q}(k) \| \\
& + 4 \alpha_{k} \sum_{i=1}^{4n^{2}} [{\bf P}]_{ji} \sum_{q \in \Gamma_{i}(k), \overline{g}_{q} concave} G {\bf A}_{fit(i)_n,q}\| {\bf x} - {\bf z}_{j}(k) \|  \\
& -  2 \alpha_{k} \sum_{i=1}^{4n^{2}} [{\bf P}]_{ji} \sum_{q \in \Gamma_{i}(k)} ({\bf A}_{fit(i)_n,q}\overline{g}_{q}({\bf z}_{j}(k)) - {\bf A}_{fit(i)_n,q}\overline{g}_{q}({\bf x}) ) \\ 
& + \alpha_{k} ^{2} \| \sum_{i=1}^{4n^{2}} [{\bf P}]_{ji} \sum_{q \in \Gamma_{i}(k)} {\bf A}_{fit(i)_n,q}\nabla{\overline{g}}_{q}({\bf z}_{q}(k)) \| ^{2}.
\end{split}
\end{equation}

\begin{remark}
But $ \| {\bf x} - {\bf z}_{j}(k) \| < D $ belongs to a bounded space.
\end{remark}

\begin{lemma}\label{L17}
Let assumptions hold. Then for $ 1 \leq j \leq 2n^{2} $ we have

{\scriptsize \begin{equation}
\begin{split}
2 \sum_{k=0}^{k^{'}} & \alpha_{k} \sum_{i=1}^{4n^{2}} [{\bf P}]_{ji} \sum_{q \in \Gamma_{i}(k)} ({\bf A}_{fit(i)_n,q}\overline{g}_{q}({\bf z}_{j}(k)) -  {\bf A}_{fit(i)_n,q}g_{q}({\bf x}^{*}) ) \leq \\ 
 & 4 \sum_{k=0}^{k^{'}}\alpha_{k} \sum_{i=1}^{4n^{2}} [{\bf P}]_{ji} \sum_{q \in \Gamma_{i}(k), \overline{g}_{q} concave} G {\bf A}_{fit(i)_n,q}\| {\bf x}^{*} - {\bf z}_{j}(k) \| 
  \| \hat{\bar{{\bf z}}}_{j}(0) - {\bf x}^{*} \| ^{2} \\
  & - \| \hat{\bar{{\bf z}}}_{j}(k^{'}) - {\bf x}^{*} \| ^{2} \\
 & + 4 \sum_{k=0}^{k^{'}}\alpha_{k} \sum_{i=1}^{4n^{2}} [{\bf P}]_{ji} \sum_{q \in \Gamma_{i}(k)} G {\bf A}_{fit(i)_n,q}\| {\bf z}_{j}(k) - {\bf z}_{q}(k) \| \\
 & + \sum_{k=0}^{k^{'}}2 \alpha_{k} \sum_{i=1}^{4n^{2}} [{\bf P}]_{ji} \sum_{q \in \Gamma_{i}(k)}\| {\bf A}_{fit(i)_n,q}\nabla{\overline{g}}_{q}({\bf z}_{q}(k)) \| \| \hat{\bar{{\bf z}}}_{j}(k) - {\bf z}_{j}(k) \| \\
& \hspace{-1cm} + 4 \sum_{k=0}^{k^{'}} \alpha_{k} \sum_{i=1}^{4n^{2}} [{\bf P}]_{ji} \sum_{q \in \Gamma_{i}(k)} G {\bf A}_{fit(i)_n,q}\| {\bf z}_{j}(k)- {\bf z}_{q}(k) \| \\
& + 4 \sum_{k=0}^{k^{'}}\alpha_{k} ^{2} \| \sum_{i=1}^{4n^{2}} [{\bf P}]_{ji} \sum_{q \in \Gamma_{i}(k)} {\bf A}_{fit(i)_n,q}\nabla{\overline{g}}_{q}({\bf z}_{q}(k)) \| ^{2}.
\end{split}
\end{equation}}
\end{lemma} 
\vspace{-0.25cm}
\textbf{Proof:} 
\\
Then summing the inequality (\ref{eqn_L15}) in Lemma~\ref{L16} from $ k=0 $ to $ k^{'} $ and taking $ {\bf x} = {\bf x}^{*} $, the optimal value, we get the desired result.
\qed

\subsection{ Overall Convergence Rate }

Then to find the convergence rate to any iteration $ K $ we add the terms from $ k = k^{'} $ to $ K $ of (\ref{eqn_L13_1}) to get
{\footnotesize \begin{equation}\label{a11}
\begin{split}
 2 & \sum_{k=0}^{K}  \alpha_{k} \sum_{i=1}^{4n^{2}} \omega_{i} [{\bf P}]_{ji} 
 (f({\bf z}_{j}(k)) -  f({\bf x}^{*})) \\
 & = 2 \sum_{k=0}^{k^{'}} \alpha_{k} \sum_{i=1}^{4n^{2}} [{\bf P}]_{ji} \sum_{q \in \Gamma_{i}(k)}  ({\bf A}_{fit(i)_n,q}\overline{g}_{q}({\bf z}_{j}(k))  -  {\bf A}_{fit(i)_n,q}g_{q}({\bf x}^{*}) ) \\ & + 2 \sum_{k=k^{'}}^{K} \alpha_{k} \sum_{i=1}^{4n^{2}} [{\bf P}]_{ji}  (f({\bf z}_{j}(k)) -  f({\bf x}^{*})) \\ & \leq 4 \sum_{k=0}^{k^{'}}\alpha_{k} \sum_{i=1}^{4n^{2}} [{\bf P}]_{ji} \sum_{q \in \Gamma_{i}(k), \overline{g}_{q} concave} G {\bf A}_{fit(i)_n,q}\| {\bf x}^{*} - {\bf z}_{j}(k) \| \\ 
 & \hspace{1cm} +  \| \hat{\bar{{\bf z}}}_{j}(0) - {\bf x}^{*} \| ^{2} - \| \hat{\bar{{\bf z}}}_{j}(K) - {\bf x}^{*} \| ^{2} \\
 & + \sum_{k=0}^{K}2 \alpha_{k} \sum_{i=1}^{4n^{2}} [{\bf P}]_{ji} \sum_{q \in \Gamma_{i}(k)}\| {\bf A}_{fit(i)_n,q}\nabla{\overline{g}}_{q}({\bf z}_{q}(k)) \| \| \hat{\bar{{\bf z}}}_{j}(k) - {\bf z}_{j}(k) \| \\
& + 4 \sum_{k=0}^{K} \alpha_{k} \sum_{i=1}^{4n^{2}} [{\bf P}]_{ji} \sum_{q \in \Gamma_{i}(k)} G {\bf A}_{fit(i)_n,q}\| {\bf z}_{j}(k)- {\bf z}_{q}(k) \| \\
 & \sum_{k=0}^{K}\alpha_{k} ^{2} \| \sum_{i=1}^{4n^{2}} [{\bf P}]_{ji} \sum_{q \in \Gamma_{i}(k)} {\bf A}_{fit(i)_n,q}\nabla{\overline{g}}_{q}({\bf z}_{q}(k)) \| ^{2}.
\end{split}
\end{equation}}


After elaborating more on the change $ \| \hat{\bar{{\bf z}}}_{j}(K) - {\bf x} \| ^{2} =  \|  \hat{\bar{{\bf z}}}_{j}(k^{'}) - {\bf x} \| ^{2}  $ before and after consensus using Lemma~\ref{L14} and the corresponding result in Lemma~\ref{L10} that are based on the characteristics of the global function $ f $ and local functions $ g_{i} $, respectively. And summing those changes from $ k=0 $ to $ k = k^{'} $ and from $ k= k^{'} $ to $ \infty $ separately, and knowing that $ \sum_{k=0}^{K} \alpha_{k}(f_{min} - f({\bf x} ^{*})) \leq \sum_{k=0}^{K} \alpha_{k} (f({\bf z}_{i}(k)) -f({\bf x}^{*}))$, where $  f_{min} = min_{0 \leq k \leq K}f({\bf z}_{i}(k)) $, using the bounds on $ \| \nabla{g}_{i} \| $, $ E_{k} $ and $ D_{k} $ on Assumption 1 (d), (\ref{Ek}) and (\ref{Dk}) respectively, we get the following convergence rate for a global convex function $ f $ composed of only convex local functions $ \bar{g}_{i} $:

\vspace{0.5cm}

\begin{remark}

We will have a follow up paper for a detailed discussion of the analysis of this convergence rate, while we restrain our discussion here to the results obtained only.

\end{remark}

\vspace{-1cm}

\begin{equation}\label{Convergence_Rate}
\begin{split}
 f_{min} - f({\bf x} ^{*})  \leq \frac{{\bf A}_{*}}{ \omega_{i}\sum_{k=0}^{K}\alpha_{k}} + \frac{{\bf B}_{*} \sum_{k=0}^{K} \alpha_{k}^{2}}{ \omega_{i} \sum_{k=0}^{K}\alpha_{k}}
\end{split}
\end{equation}

where 
\begin{equation}\label{eqn_A*}
\begin{split}
 {\bf A}_{*}& =\frac{1}{2 \| {\bf P} \|_{2,\infty}} (dist^{2}( \hat{\bar{{\bf z}}}(0), \mathcal{X} ^{*} ) - \frac{1}{2 \| {\bf P} \|_{2,\infty}} dist^{2}( \hat{\bar{{\bf z}}}(k), \mathcal{X} ^{*} ) \\ 
 & + \frac{5}{2}\frac{ \| \hat{\bf A} \|_{2,\infty}   \Gamma \| {\bf B} \|_{2,\infty} \sqrt{n} F }{1- \gamma^{2}} \sum_{l=1}^{4n} \| {\bf z}_{l}(0) \| 
\end{split}
\end{equation} 
and 
\begin{equation}\label{eqn_B*}
\begin{split}
 {\bf B}_{*} & = \frac{7}{2}  \| \hat{\bf A} \|_{2,\infty} \| {\bf B} \|_{2,\infty} ^{2} n F^{2} (\| {\bf P} \|_{2,\infty} \|\| \hat{\bf A} \|_{2,\infty} + \frac{10}{7}) \\
 & 4 n \| \hat{\bf A} \|_{2, \infty} \| {\bf B} \|_{2,\infty} ^{2} n F^{2} \frac{\Gamma}{1-\gamma} ( \| \hat{\bf A} \|_{2,\infty} + 4) \\
 & +  3 \| \hat{\bf A} \|_{2, \infty}  \| {\bf B} \|_{2,\infty} \sqrt{n} F \Gamma \sum_{l=1}^{4n} \| {\bf z}_{l}(0) \|  
\end{split}
\end{equation}

where we used {\footnotesize $ G= \sqrt{n} \| {\bf B} \|_{2,\infty} F $}.
For $ \alpha_{k}=\frac{1}{\sqrt{k}} $ then the convergence rate is a scaled coefficient adequate to the coding scheme/network topology of rate $ \mathbf{O(\frac{\ln{k}}{\sqrt{k}})} $.

\begin{remark}
We took the bound for \\
$ \sum_{q \in \Gamma_{i}(k)}\| {\bf A}_{fit(i)_n,q}\nabla{\overline{g}}_{q}({\bf z}_{q}(k)) \| $ to be $ \| \hat{\bf A} \|_{2, \infty}  \| {\bf B} \|_{2,\infty} \sqrt{n} F $ for $ k < k^{'} $ (before consensus) due to the inexactness of the global function gradient defined by this term where we are unable to use the bound of $ \nabla{f} $. While for $ k \geq k^{'} $ (in the consensus region) we can have a stricter bound for $ \sum_{q \in \Gamma_{i}(k)}\| {\bf A}_{fit(i)_n,q}\nabla{\overline{g}}_{q}({\bf z}_{q}(k)) \| $ which is that of $ \nabla{f} $, that is $ n F $ (i.e, $ \nabla{f}=\sum_{i=1}^{n}\nabla{f}_{i} $ and $ \nabla{f}_{i} \leq F $).
While in our analysis of the overall convergence rate concerning both regions we used the more relaxed bound dependent on $ \hat{\bf A} $ and $ {\bf B} $.
\end{remark}

However, there is an accumulation term if the global function $ f $ is formed also of concave functions that add up from $ k = 0 $ until consensus point $ k = K^{'} $ which will increase the convergence rate by a supplementary value less than $ H $ where 
\begin{equation}\label{Bound}
    \begin{split}
        H =  4 \sum_{k=0}^{k^{'}}\alpha_{k} \| {\bf P} \|_{2,  \infty} \| \hat{\bf A} \|_{2, \infty}  \| {\bf B} \|_{2,\infty} \sqrt{n} F  \sqrt{2} R.
    \end{split}
\end{equation}

Here, we took the space $ X $ where the estimates are defined to be of radius $ \sqrt{2} R $ where $ R = \max_{{\bf z}_{1}, {\bf z}_{2}} \| {\bf z}_{1} - {\bf z}_{2} \| $, i.e., $ \| {\bf x}^{*} - {\bf z}_{j}(k) \| \leq \sqrt{2} R $.
\begin{remark}
We see this specific accumulation bound for a global convex function if it is composed of local coded functions $ \bar{g}_{i} $ that are either convex or concave for $ 1 \leq i \leq 2n $. And this bound is an effect of concavity of $ \bar{g}_{i} $ here. We can generalize this bound to any combinations of local coded functions where it is the result of the non-convexity of those local functions. Thus, the convergence rate is the least for global function $ f $ composed of only convex coded local functions $ \bar{g}_{i} $ for $ 1 \leq i \leq 2n $ as is (\ref{Convergence_Rate}), where it lacks the supplementary term (\ref{Bound}). N.B. $ \bar{g}_{i} $ convex for $ 1\leq i \leq 2n $ means that $ g_{i} $ convex if $ 1 \leq i \leq n $ and $ -g_{i-n} $ convex for $ n+1 \leq i \leq 2n $. That is, $ g_{i} $ convex if the corresponding coefficients of node $ i $ used in $ \hat{\bf A} $ are positive and $ g_{i} $ concave if the corresponding coefficients of node $ i $ used in $ \hat{\bf A} $ are negative and $ g_{i} $ linear if the corresponding coefficients of node $ i $ are positive and negative.
\end{remark}

Therefore, the convergence rate of this algorithm is a scaled version of the convergence rate of the distributed gradient descent algorithm.
Thus, we can perfectly adjust this scaling factor according to a desired coding scheme so that the algorithm is tuned to perform better than DGD by that factor yet still under $ \mathbf{O(\frac{\ln{k}}{\sqrt{k}})} $ for $ \alpha_{k}= \frac{1}{\sqrt{k}} $.

However, from the norm inequality
\begin{equation}
\begin{split}
 \| {\bf M} \|_{F} \leq \sqrt{n} \| {\bf M} \|_{2, \infty},  
\end{split}    
\end{equation}
we have the minimum of $ \| {\bf M} \|_{2, \infty} $ attained when $ {\bf M}= \frac{1}{n}\mathbf{1}^{T}\mathbf{1} $  where $ \| {\bf M} \|_{F}=1 $ and $ \| {\bf M} \|_{2 ,\infty} = \frac{\| {\bf M} \|_{F}}{\sqrt{n}}= \frac{1}{\sqrt{n}} $. 

But $ \hat{\bf A} $ and $ {\bf P} $ are stochastic matrices (particularly row normalized). The first relative to $ n $ and the latter relative to $ 2n $ although the sizes are $ 2n \times 2n $ and $ 4n \times 4n $ respectively. (i.e., $ {\bf P} $ involved is for  $ 1 \leq i \leq 2n $ rows and the corresponding nonzero columns are $ 1 \leq j \leq 2n $. Thus, when we use $ \| {\bf P} \|_{2,\infty} $ we mean the norm restricted to this part of $ {\bf P} $ which is row stochastic).
Therefore, \\
$ \frac{1}{\sqrt{n}} \leq \| \hat{\bf A} \|_{2 , \infty} \leq 1 $, and $ \frac{1}{\sqrt{2n}} \leq \| {\bf P} \|_{2, \infty} \leq 1 $.

Thus the scaling in (\ref{eqn_A*}) and (\ref{eqn_B*}) can be adjusted to be less than one. Hence, a better convergence rate than DGD. 

Moreover, if we implement our algorithm but with no coding (no redundancy that is, with local functions $ f_{i} $) where we have $ \| {\bf B} \|_{2,\infty} =1 $ and $ \| \hat{\bf A} \|_{2, \infty} = \frac{1}{\sqrt{n}} $, we are still able to achieve a better convergence rate than DGD although our updating matrix can be of a larger size (i.e., $ 4n $ rather than $ n $). We can reach that by suitably choosing the updating matrix $ {\bf Q}_{\epsilon , {\bf I}} $ so that its limit $ {\bf P} $ has a value of $ \| {\bf P} \|_{2, \infty} $ as close  as possible as $  \frac{1}{\sqrt{2n}} $.

\section{ Appendix F: Preliminary Lemmas and Theorems }

\subsection{Proof of Proposition~1}\label{ProofL1}

\begin{proof}

$ {\bf Q} $ is a block lower triangular matrix, its spectrum is $ \sigma({\bf Q}) = \sigma( \hat{\bf A}) \cup \sigma(\hat{\bf D}) $.

$ \left(\begin{array}{c}
{\bf 1}_{2n^{2} \times 1} \end{array}\right) $ is easily seen to be the only right eigenvector of $ \hat{\bf A} $ for the eigenvalue $ 1 $. Then geometric multiplicity of eigenvalue $ 1 $ is $ 1 $ for matrix $ \hat{\bf A} $.

$ \bar{\bf v}^{+} $ and $ \bar{\bf v}^{-} $ are $ 2n^{2} \times 1 $ vectors where
{\scriptsize\begin{equation}
 \bar{\bf v}^{+} (j) = \begin{cases}  {\bf v}^{+} (j - (j \ div \ 2n)*n)   \ \ \  & if \ 1 \leq j \mod 2n \leq n \\
 0  \ \ \ & otherwise \end{cases}
\end{equation}}
and the $ 2n^{2} \times 1 $ vector
{\scriptsize \begin{equation}
 \bar{\bf v}^{-} (j) = \begin{cases} {\bf v}^{-}(j - (j \ div \ 2n + 1 )*n)    \ \ \  & if \  n + 1 \leq j \mod 2n \leq 2 n \\
 0  \ \ \ & otherwise \end{cases}
\end{equation}}

And $ {\bf v}^{+} $ and $ {\bf v}^{-} $ are the right eigenvector of eigenvalue $ 1 $ for the column stochastic matrices $ \hat{\bf D}^{++} $ or $ \hat{\bf D}^{--} $, respectively (i.e., $ \hat{\bf D}^{++} {\bf v}^{+} =  {\bf v}^{+} $ and $ \hat{\bf D}^{--} {\bf v}^{-} = {\bf v}^{-} $), with all values positive and scaled such that $ {\bf 1}_{ n^{2} \times 1} ^{T} {\bf v}^{+} = {\bf 1}_{ n^{2} \times 1} ^{T} {\bf v}^{-} = 1 $. Where $ {\bf v}^{+}(i) = {\bf v}^{+}(j) $ for $ (i-j) \mod n = 0 $ and Where $ {\bf v}^{-}(i) = {\bf v}^{-}(j) $ for $ (i-j) \mod n = 0 $.

Then $ \bar{\bf v}^{+} $ and $ \bar{\bf v}^{-} $ are the only two independent right eigenvectors for $ \hat{\bf D} $ of eigenvalue $ 1 $. Thus, the geometric multiplicity of eigenvalue $ 1 $ is $ 2 $ for matrix $ \hat{\bf D} $. 

But matrices $ \hat{\bf A} $ and $ \hat{\bf D} $ are stochastic matrices, then the geometric multiplicities of their eigenvalues are equal to the algebraic multiplicities. So, Then algebraic multiplicity of eigenvalue $ 1 $ is $ 1 $ for matrix $ \hat{\bf A} $. And the algebraic multiplicity of eigenvalue $ 1 $ is $ 2 $ for matrix $ \hat{\bf D} $. And all of their other eigenvalues lie in the unit circle with their geometric multiplicities equal to their algebraic multiplicities. 

But as mentioned earlier, $ {\bf Q} $ is a block lower triangular matrix, its spectrum is $ \sigma({\bf Q}) = \sigma( \hat{\bf A}) \cup \sigma(\hat{\bf D}) $. 
And matrix $ \hat{\bf A} $ and the matrix $ \hat{\bf D} $ are row and column stochastic, respectively, so their spectral radii satisfy $ \rho( \hat{\bf A}) = \rho(\hat{\bf D}) = 1 $.

We have $ \rho( \hat{\bf A}) = 1 $ is a simple eigenvalue of $ \hat{\bf A} $ ( i.e., algebraic multiplicity is equal to geometric multiplicity is equal to $ 1 $). Thus $ \rho( \hat{\bf A})=\rho(\hat{\bf D})=1 $ is a semi-simple eigenvalue  of $ \hat{\bf D} $ (i.e., algebraic multiplicity is equal to geometric multiplicity is equal to $ 2 $). Then $ 1 $ is an eigenvalue of $ {\bf Q} $ of algebraic multiplicity equals to $ 3 $. Moreover, the rank of $ {\bf Q} - {\bf I} = 4n - 3 $, so the geometric multiplicity of eigenvalue $ 1 $ is equal to $ 3 $. Thus, eigenvalue $ 1 $ is a semi-simple eigenvalue of $ {\bf Q} $ of multiplicity $ 3 $.

Matrix $ {\bf Q} $ has $ 3 $ independent eigenvectors for eigenvalue $ 1 $.

$ \left(\begin{array}{c}
{\bf 1}_{2n^{2} \times 1} \\ {\bf 0}_{2n^{2} \times 1} \end{array}\right) $ and $ \left(\begin{array}{c}
{\bf 0}_{2n^{2} \times 1} \\ \bar{\bf v}^{+} 
\end{array}\right) $ and $ \left(\begin{array}{c}
{\bf 0}_{2n^{2} \times 1} \\ \bar{\bf v}^{-}\end{array}\right) $, where the $ 2n^{2} \times 1 $ vector
{\small \begin{equation}
 \bar{\bf v}^{+} (j) = \begin{cases}  {\bf v}^{+} (j - (j \ div \ 2n)*n)   \ \ \  & if \ 1 \leq j \mod 2n \leq n \\
 0  \ \ \ & otherwise \end{cases}
\end{equation}}
and the $ 2n^{2} \times 1 $ vector
{\small \begin{equation}
 \bar{\bf v}^{-} (j) = \begin{cases} {\bf v}^{-}(j - (j \ div \ 2n + 1 )*n)    \ \ \  & if \  n + 1 \leq j \mod 2n \leq 2 n \\
 0  \ \ \ & otherwise \end{cases}
\end{equation}}

are the $ 3 $ independent right eigenvector of $ {\bf Q} $ for the eigenvalue $ 1 $.

\end{proof}

\subsection{Proof of Lemma~1} \label{ProofLAfit}

\begin{proof}
Since $ \bar{\bf A}^{l} $ is chosen to be scrambling then $ \hat{\bf A}^{2} $ is scrambling. Thus, $ \hat{\bf A} $ is stochastic indecomposable and aperiodic. Then $ \lim_{n \rightarrow \infty} \hat{\bf A}^{n} = \mathbf{1}\pi^{T} $. And all rows of  $ \lim_{n \rightarrow \infty} \hat{\bf A}^{n} $ are the same.

$\hat{\bf A} =
\begin{pmatrix}
    \diagentry{\bar{\bf A}^{1}_{+}} {\bf E}_{+} {\bf E}_{+} {\bf E}_{+} \ldots\\
    \diagentry{\bar{\bf A}^{1}_{-}}  {\bf E}_{-} {\bf E}_{-} {\bf E}_{-} \ldots\\
    {\bf E}_{+}\diagentry{\bar{\bf A}^{2}_{+}} {\bf E}_{+} {\bf E}_{+} \ldots\\
    {\bf E}_{-}\diagentry{\bar{\bf A}^{2}_{-}} {\bf E}_{-} {\bf E}_{-} \ldots \\
    &&\diagentry{\xddots}\\
    &&\diagentry{\xddots}\\
    {\bf E}_{+} {\bf E}_{+} {\bf E}_{+} \ldots \diagentry{\bar{\bf A}^{n}_{+}}\\
    {\bf E}_{-} {\bf E}_{-} {\bf E}_{-} \ldots \diagentry{\bar{\bf A}^{n}_{-}}
\end{pmatrix}$  
and each subblock is an $ n \times 2 n $ matrix.

Where if we choose $ \bar{\bf A}^{l}=\bar{\bf A}^{l}_{+} = \bar{\bf A}^{l}_{-} $  for $ 1 \leq l \leq n $, then $ {\bf E} = {\bf E}_{+} =  {\bf E}_{-} $ and

$\hat{\bf A} =
\begin{pmatrix}
    \diagentry{\bar{\bf A}^{1}} {\bf E} {\bf E} {\bf E} \ldots\\
    \diagentry{\bar{\bf A}^{1}}  {\bf E} {\bf E} {\bf E} \ldots\\
    {\bf E}\diagentry{\bar{\bf A}^{2}} {\bf E} {\bf E} \ldots\\
    {\bf E}\diagentry{\bar{\bf A}^{2}} {\bf E} {\bf E} \ldots \\
    &&\diagentry{\xddots}\\
    &&\diagentry{\xddots}\\
    {\bf E} {\bf E} {\bf E} \ldots \diagentry{\bar{\bf A}^{n}}\\
    {\bf E} {\bf E} {\bf E} \ldots \diagentry{\bar{\bf A}^{n}}
\end{pmatrix}$  
Let us denote by $ \bar{\bar{\bf A}}^{l} = \begin{pmatrix}
    \diagentry{\bar{\bf A}^{l}} \\
    \diagentry{\bar{\bf A}^{l}} 
    \end{pmatrix} = \begin{pmatrix}
    \diagentry{\hat{\bar{\bf A}}^{l}_{1} \hat{\bar{\bf A}}^{l}_{2}} \\
    \diagentry{\hat{\bar{\bf A}}^{l}_{1} \hat{\bar{\bf A}}^{l}_{2}} 
    \end{pmatrix}$
    
and $ \bar{\bar{\bf E}} = \begin{pmatrix}
    \diagentry{{\bf E}} \\
    \diagentry{{\bf E}} 
    \end{pmatrix}= \begin{pmatrix}
    {\bar{\bf E}  \ \mathbf{0}_{n \times n}} \\
    {\bar{\bf E}  \ \mathbf{0}_{n \times n}} 
    \end{pmatrix}$

Then 
$\hat{\bf A} =
\begin{pmatrix}
    {\hat{\bar{\bf A}}^{1}_{1} \hat{\bar{\bf A}}^{1}_{2}} {\bar{\bf E}} \ \mathbf{0}_{n \times n} {\bar{\bf E}} \ \mathbf{0}_{n \times n} {\bar{\bf E}} \ \mathbf{0}_{n \times n} \ldots\\
    {\hat{\bar{\bf A}}^{1}_{1}  \hat{\bar{\bf A}}^{1}_{2}} {\bar{\bf E}} \ \mathbf{0}_{n \times n} {\bar{\bf E}} \ \mathbf{0}_{n \times n} {\bar{\bf E}} \ \mathbf{0}_{n \times n} \ldots\\
    {\bar{\bf E}} \ \mathbf{0}_{n \times n} {\hat{\bar{\bf A}}^{2}_{1} \hat{\bar{\bf A}}^{2}_{2}} {\bar{\bf E}} \ \mathbf{0}_{n \times n} {\bar{\bf E}} \ \mathbf{0}_{n \times n} \ldots \\
     {\bar{\bf E}} \ \mathbf{0}_{n \times n} {\hat{\bar{\bf A}}^{2}_{1} \hat{\bar{\bf A}}^{2}_{2}} {\bar{\bf E}} \ \mathbf{0}_{n \times n} {\bar{\bf E}} \ \mathbf{0}_{n \times n} \ldots\\
    &&\diagentry{\xddots}\\
    &&\diagentry{\xddots}\\
    {\bar{\bf E}} \ \mathbf{0}_{n \times n}  {\bar{\bf E}} \ \mathbf{0}_{n \times n} {\bar{\bf E}} \ \mathbf{0}_{n \times n} \ldots {\hat{\bar{\bf A}}^{n}_{1} \hat{\bar{\bf A}}^{n}_{2}}\\
    {\bar{\bf E}} \ \mathbf{0}_{n \times n}  {\bar{\bf E}} \ \mathbf{0}_{n \times n} {\bar{\bf E}} \ \mathbf{0}_{n \times n} \ldots {\hat{\bar{\bf A}}^{n}_{1} \hat{\bar{\bf A}}^{n}_{2}}
\end{pmatrix}$  

and

$\hat{\bf A}^{2} =
\begin{pmatrix}
    { \mathcal{\bf A^{(2)}}^{1}_{1} \mathcal{\bf A^{(2)}}^{1}_{2}} \mathbf{\bf E^{(2)}}_{12} \ \mathbf{0}_{n \times n} \mathbf{\bf E^{(2)}}_{13}  \ \mathbf{0}_{n \times n} \mathbf{\bf E^{(2)}}_{14}  \ \mathbf{0}_{n \times n} \ldots\\
    { \mathcal{\bf A^{(2)}}^{1}_{1}  \mathcal{\bf A^{(2)}}^{1}_{2}} \mathbf{\bf E^{(2)}}_{12}  \ \mathbf{0}_{n \times n} \mathbf{\bf E^{(2)}}_{13}   \ \mathbf{0}_{n \times n} \mathbf{\bf E^{(2)}}_{14}  \ \mathbf{0}_{n \times n} \ldots\\
    \mathbf{\bf E^{(2)}}_{21}  \ \mathbf{0}_{n \times n} { \mathcal{\bf A^{(2)}}^{2}_{1}  \mathcal{\bf A^{(2)}}^{2}_{2}} \mathbf{\bf E^{(2)}}_{23}  \ \mathbf{0}_{n \times n} \mathbf{\bf E^{(2)}}_{24}  \ \mathbf{0}_{n \times n} \ldots\\
    \mathbf{\bf E^{(2)}}_{21}  \ \mathbf{0}_{n \times n} { \mathcal{\bf A^{(2)}}^{2}_{1}  \mathcal{\bf A^{(2)}}^{2}_{2}} \mathbf{\bf E^{(2)}}_{23}  \ \mathbf{0}_{n \times n} \mathbf{\bf E^{(2)}}_{24}   \ \mathbf{0}_{n \times n}\ldots\\
    &&\diagentry{\xddots}\\
    &&\diagentry{\xddots}\\
    \mathbf{\bf E^{(2)}}_{n1}  \ \mathbf{0}_{n \times n} \mathbf{\bf E^{(2)}}_{n2}  \ \mathbf{0}_{n \times n} \mathbf{\bf E^{(2)}}_{n3}  \ \mathbf{0}_{n \times n} \ldots \diagentry{ \mathcal{\bf A^{(2)}}^{n}_{1}  \mathcal{\bf A^{(2)}}^{n}_{2}}\\
    \mathbf{\bf E^{(2)}}_{n1}  \ \mathbf{0}_{n \times n} \mathbf{\bf E^{(2)}}{n2}  \ \mathbf{0}_{n \times n} \mathbf{\bf E^{(2)}}{n3}  \ \mathbf{0}_{n \times n} \ldots \diagentry{ \mathcal{\bf A^{(2)}}^{n}_{1}  \mathcal{\bf A^{(2)}}^{n}_{2}}
\end{pmatrix}$  

where $  \mathcal{\bf A^{(2)}}^{l}_{1} = (\hat{\bar{\bf A}}^{l}_{1})^{2} + \hat{\bar{\bf A}}^{l}_{2}\hat{\bar{\bf A}}^{l}_{1} + (n-1) {\bar{\bf E}}^{2} $ and $  \mathcal{\bf A^{(2)}}^{l}_{2} = (\hat{\bar{\bf A}}^{l}_{2})^{2} + \hat{\bar{\bf A}}^{l}_{1}\hat{\bar{\bf A}}^{l}_{2} + (n-1) {\bar{\bf E}}^{2} $.
\\
$ \mathbf{\bf E^{(2)}}_{ij} = \hat{\bar{\bf A}}^{i}_{1} {\bar{\bf E}} + \hat{\bar{\bf A}}^{i}_{2}{\bar{\bf E}} + {\bar{\bf E}} \hat{\bar{\bf A}}^{j}_{1} + {\bar{\bf E}} \hat{\bar{\bf A}}^{j}_{1}  + (n-2) {\bar{\bf E}}^{2} $.
Since $ \bar{\bf E} $ is diagonal then  $ {\bar{\bf E}} \hat{\bar{\bf A}}^{j}_{1} \tilde{=} \hat{\bar{\bf A}}^{j}_{1} $. Since all these matrices are nonnegative then $  \mathcal{\bf A^{(2)}}^{l}_{1} = (\hat{\bar{\bf A}}^{l}_{1})^{2} + {\bf \Lambda}^{(2)}_{1} $, $  \mathcal{\bf A^{(2)}}^{l}_{2} =  (\hat{\bar{\bf A}}^{l}_{2})^{2} + {\bf \Lambda}^{(2)}_{2} $ and $ \mathbf{\bf E^{(2)}}_{ij} \tilde{=}  \hat{\bar{\bf A}}^{j}_{1} + {\bf \Lambda}^{(2)}_{3}  $ where $  {\bf \Lambda}^{(2)}_{1} $, $ {\bf \Lambda}^{(2)}_{2} $ and $ {\bf \Lambda}^{(2)}_{3} $ are nonnegative.

And $\hat{\bf A}^{n+1} = $
\\
$\begin{pmatrix}
 { \mathcal{\bf A^{(n+1)}}^{1}_{1} \mathcal{\bf A^{(n+1)}}^{1}_{2}} \mathbf{\bf E^{(n+1)}}_{12} \ \mathbf{0}_{n \times n} \mathbf{\bf E^{(n+1)}}_{13}  \ \mathbf{0}_{n \times n} \mathbf{\bf E^{(n+1)}}_{14}  \ \mathbf{0}_{n \times n} \ldots\\
    { \mathcal{\bf A^{(n+1)}}^{1}_{1}  \mathcal{\bf A^{(n+1)}}^{1}_{2}} \mathbf{\bf E^{(n+1)}}_{12}  \ \mathbf{0}_{n \times n} \mathbf{\bf E^{(n+1)}}_{13}   \ \mathbf{0}_{n \times n} \mathbf{\bf E^{(n+1)}}_{14}  \ \mathbf{0}_{n \times n} \ldots\\
    \mathbf{\bf E^{(n+1)}}_{21}  \ \mathbf{0}_{n \times n} { \mathcal{\bf A^{(n+1)}}^{2}_{1}  \mathcal{\bf A^{(n+1)}}^{2}_{2}} \mathbf{\bf E^{(n+1)}}_{23}  \ \mathbf{0}_{n \times n} \mathbf{\bf E^{(n)}}_{24}  \ \mathbf{0}_{n \times n} \ldots\\
{ \mathcal{\bf A^{(n+1)}}^{2}_{1}  \mathcal{\bf A^{(n+1)}}^{2}_{2}} \mathbf{\bf E^{(n+1)}}_{23}  \ \mathbf{0}_{n \times n} \mathbf{\bf E^{(n+1)}}_{24}   \ \mathbf{0}_{n \times n}\ldots\\
    &&\diagentry{\xddots}\\
    &&\diagentry{\xddots}\\
    \mathbf{\bf E^{(n+1)}}_{n1}  \ \mathbf{0}_{n \times n} \mathbf{\bf E^{(n+1)}}_{n2}  \ \mathbf{0}_{n \times n} \mathbf{\bf E^{(n+1)}}_{n3}  \ \mathbf{0}_{n \times n} \ldots \diagentry{ \mathcal{\bf A^{(n+1)}}^{n}_{1}  \mathcal{\bf A^{(n+1)}}^{n}_{2}}\\
    \mathbf{\bf E^{(n+1)}}_{n1}  \ \mathbf{0}_{n \times n} \mathbf{\bf E^{(n+1)}}{n2}  \ \mathbf{0}_{n \times n} \mathbf{\bf E^{(n+1)}}{n3}  \ \mathbf{0}_{n \times n} \ldots \diagentry{ \mathcal{\bf A^{(n+1)}}^{n}_{1}  \mathcal{\bf A^{(n+1)}}^{n}_{2}}
\end{pmatrix}$  

$  \mathcal{\bf A^{(n+1)}}^{l}_{1} = (\hat{\bar{\bf A}}^{l}_{1})^{n+1} + {\bf \Lambda}^{(n+1)}_{1} $, $  \mathcal{\bf A^{(n+1)}}^{l}_{2} =  (\hat{\bar{\bf A}}^{l}_{2})^{n+1} + {\bf \Lambda}^{(n+1)}_{2} $ and $ \mathbf{\bf E^{(n+1)}}_{ij} =  (\hat{\bar{\bf A}}^{j}_{1})^{n} + {\bf \Lambda}^{(n+1)}_{3}  $ where $  {\bf \Lambda}^{(n+1)}_{1} $, $ {\bf \Lambda}^{(n+1)}_{2} $ and $ {\bf \Lambda}^{(n+1)}_{3} $ are nonnegative.

Thus, since $ \bar{\bar{\bf A}}^{j} $ ( $\hat{\bar{\bf A}}^{j}_{1} $) is stochastic indecomposable and aperiodic (SIA) then there exist $ k > 0 $ such that $ (\hat{\bar{\bf A}}^{j}_{1})^{k} $ has a positive column. Then $ (\hat{\bar{\bf A}}^{j}_{1})^{k+1} $ has a positive column on the corresponding entries as $ (\hat{\bar{\bf A}}^{j}_{1})^{k} $.
Thus, picking the block column $ j $ in $ \hat{\bf A}^{k+1} $, we see that $  \mathcal{\bf A^{(k+1)}}^{j}_{1} = (\hat{\bar{\bf A}}^{j}_{1})^{k+1} + {\bf \Lambda}^{(k+1)}_{1} $ has the a positive column on the corresponding entries of  $ \mathbf{\bf E^{(k+1)}}_{ij} \tilde{=}  (\hat{\bar{\bf A}}^{j}_{1})^{k} + {\bf \Lambda}^{(k+1)}_{3}  $ since  $  {\bf \Lambda}^{(k+1)}_{1} $ and $ {\bf \Lambda}^{(k+1)}_{3} $ are nonnegative. Then there exist a positive column in block column $ j $ in $ \hat{\bf A}^{k+1} $. 
Thus, $ \hat{\bf A} $ is SIA.

Therefore, $ \lim_{n \rightarrow \infty} \hat{\bf A}^{n} = \mathbf{1}\pi^{T} $. And all rows of  $ \lim_{n \rightarrow \infty} \hat{\bf A}^{n} $ are the same.

Similarly, for 
$\hat{\bf A} =
\begin{pmatrix}
    \diagentry{\bar{\bf A}^{1}_{+}} {\bf E}_{+} {\bf E}_{+} {\bf E}_{+} \ldots\\
    \diagentry{\bar{\bf A}^{1}_{-}}  {\bf E}_{-} {\bf E}_{-} {\bf E}_{-} \ldots\\
    {\bf E}_{+}\diagentry{\bar{\bf A}^{2}_{+}} {\bf E}_{+} {\bf E}_{+} \ldots\\
    {\bf E}_{-}\diagentry{\bar{\bf A}^{2}_{-}} {\bf E}_{-} {\bf E}_{-} \ldots \\
    &&\diagentry{\xddots}\\
    &&\diagentry{\xddots}\\
    {\bf E}_{+} {\bf E}_{+} {\bf E}_{+} \ldots \diagentry{\bar{\bf A}^{n}_{+}}\\
    {\bf E}_{-} {\bf E}_{-} {\bf E}_{-} \ldots \diagentry{\bar{\bf A}^{n}_{-}}
\end{pmatrix}$  
and each subblock is an $ n \times 2 n $ matrix.

Where if we choose $ \bar{\bf A}^{l}_{+} = [\bar{\bf A}^{l}_{+,1} \ \bar{\bf A}^{l}_{+,2}] $ and $ \bar{\bf A}^{l}_{-} = [\bar{\bf A}^{l}_{-,1} \ \bar{\bf A}^{l}_{-,2}] $  for $ 1 \leq l \leq n $. $  {\bf E}_{+} = [\bar{\bf E}_{+} \ \mathbf{0}_{n \times n}] $ and $ {\bf E}_{-} = [\mathbf{0}_{n \times n} \ \bar{\bf E}_{-} ]  $.

Then 
$\hat{\bf A} =
\begin{pmatrix}
   {\hat{\bar{\bf A}}^{1}_{+,1} \hat{\bar{\bf A}}^{1}_{+,2}} {\bar{\bf E}_{+}} \ \mathbf{0}_{n \times n} {\bar{\bf E}_{+}} \ \mathbf{0}_{n \times n} {\bar{\bf E}_{+}} \ \mathbf{0}_{n \times n} \ldots\\
    {\hat{\bar{\bf A}}^{1}_{-,1} \hat{\bar{\bf A}}^{1}_{-,2}} \ \mathbf{0}_{n \times n} {\bar{\bf E}_{-}} \ \mathbf{0}_{n \times n} {\bar{\bf E}_{-}} \ \mathbf{0}_{n \times n} {\bar{\bf E}_{-}} \ldots\\
    {\bar{\bf E}_{+}} \ \mathbf{0}_{n \times n} {\hat{\bar{\bf A}}^{2}_{+,1} \hat{\bar{\bf A}}^{2}_{+,2}} {\bar{\bf E}_{+}} \ \mathbf{0}_{n \times n} {\bar{\bf E}_{+}} \ \mathbf{0}_{n \times n} \ldots \\
   \mathbf{0}_{n \times n} \ {\bar{\bf E}_{-}} {\hat{\bar{\bf A}}^{2}_{-,1} \hat{\bar{\bf A}}^{2}_{-,2}}  \ \mathbf{0}_{n \times n} {\bar{\bf E}_{-}} \ \mathbf{0}_{n \times n} {\bar{\bf E}_{-}} \ldots\\
    &&\diagentry{\xddots}\\
    &&\diagentry{\xddots}\\
    {\bar{\bf E}_{+}} \ \mathbf{0}_{n \times n}  {\bar{\bf E}_{+}} \ \mathbf{0}_{n \times n} {\bar{\bf E}_{+}} \ \mathbf{0}_{n \times n} \ldots \diagentry{\hat{\bar{\bf A}}^{n}_{+,1} \hat{\bar{\bf A}}^{n}_{+,2}}\\
     \mathbf{0}_{n \times n} \ {\bar{\bf E}_{-}} \ \mathbf{0}_{n \times n} {\bar{\bf E}_{-}} \ \mathbf{0}_{n \times n} {\bar{\bf E}_{-}} \ldots \diagentry{\hat{\bar{\bf A}}^{n}_{-,1} \hat{\bar{\bf A}}^{n}_{-,2}}
\end{pmatrix}$ 

and $\hat{\bf A}^{2} = $
\\
$\begin{pmatrix}
   { \mathcal{\bf A^{(2)}}^{1}_{+,1} \mathcal{\bf A^{(2)}}^{1}_{+,2}} \mathbf{\bf E^{(2)}_{+,1}}_{12} \ \mathbf{\bf E^{(2)}_{+,2}}_{12} \mathbf{\bf E^{(2)}_{+,1}}_{13}  \ \mathbf{\bf E^{(2)}_{+,2}}_{13} \mathbf{\bf E^{(2)}_{+,1}}_{14}  \ \mathbf{\bf E^{(2)}_{+,2}}_{14} \ldots\\
   { \mathcal{\bf A^{(2)}}^{1}_{-,1}  \mathcal{\bf A^{(2)}}^{1}_{-,2}}   \ \mathbf{\bf E^{(2)}_{-,1}}_{12} \mathbf{\bf E^{(2)}_{-,2}}_{12}   \ \mathbf{\bf E^{(2)}_{-,1}}_{13} \mathbf{\bf E^{(2)}_{-,2}}_{13}  \mathbf{\bf E^{(2)}_{-,1}}_{14} \mathbf{\bf E^{(2)}_{-,2}}_{14} \ldots\\
    \mathbf{\bf E^{(2)}_{+,1}}_{21}  \ \mathbf{\bf E^{(2)}_{+,2}}_{21}  { \mathcal{\bf A^{(2)}}^{2}_{+,1}  \mathcal{\bf A^{(2)}}^{2}_{+,2}} \mathbf{\bf E^{(2)}_{+,1}}_{23}  \ \mathbf{\bf E^{(2)}_{+,2}}_{23} \mathbf{\bf E^{(2)}_{+,1}}_{24}  \ \mathbf{\bf E^{(2)}_{+,2}}_{24}  \ldots\\
     \mathbf{\bf E^{(2)}_{-,1}}_{21} \ \mathbf{\bf E^{(2)}_{-,2}}_{21}  { \mathcal{\bf A^{(2)}}^{2}_{-,1}  \mathcal{\bf A^{(2)}}^{2}_{-,2}} \mathbf{\bf E^{(2)}_{-,1}}_{23} \mathbf{\bf E^{(2)}_{-,2}}_{23}   \ \mathbf{\bf E^{(2)}_{-,1}}_{24}\ \mathbf{\bf E^{(2)}_{-,2}}_{24} \ldots \\
    &&\diagentry{\xddots}\\
    &&\diagentry{\xddots}\\
    \mathbf{\bf E^{(2)}_{+,1}}_{n1}  \ \mathbf{\bf E^{(2)}_{+,2}}_{n1} \mathbf{\bf E^{(2)}_{+,1}}_{n2}  \  \mathbf{\bf E^{(2)}_{+,2}}_{n2} \mathbf{\bf E^{(2)}_{+,1}}_{n3}  \  \mathbf{\bf E^{(2)}_{+,2}}_{n3} \ldots \diagentry{ \mathcal{\bf A^{(2)}}^{n}_{+,1}  \mathcal{\bf A^{(2)}}^{n}_{+,2}}\\
     \mathbf{\bf E^{(2)}_{-,1}}_{n1} \mathbf{\bf E^{(2)}_{-,1}}_{n1}  \ \mathbf{\bf E^{(2)}_{-,1}}_{n2} \mathbf{\bf E^{(2)}_{-,2}}_{n2}  \ \mathbf{\bf E^{(2)}_{-,1}}_{n3} \ \mathbf{\bf E^{(2)}_{-,2}}_{n3} \ldots \diagentry{ \mathcal{\bf A^{(2)}}^{n}_{-,1}  \mathcal{\bf A^{(2)}}^{n}_{-,2}}
\end{pmatrix}$

where $  \mathcal{\bf A^{(2)}}^{l}_{+,1} = (\hat{\bar{\bf A}}^{l}_{+,1})^{2} + \hat{\bar{\bf A}}^{l}_{+,2}\hat{\bar{\bf A}}^{l}_{-,1} + (n-1) {\bar{\bf E}_{+}}^{2} $, $  \mathcal{\bf A^{(2)}}^{l}_{+,2} =  \hat{\bar{\bf A}}^{l}_{+,1}\hat{\bar{\bf A}}^{l}_{+,2} + \hat{\bar{\bf A}}^{l}_{+,2}\hat{\bar{\bf A}}^{l}_{- ,2} $, $  \mathcal{\bf A^{(2)}}^{l}_{-,1} =  \hat{\bar{\bf A}}^{l}_{-,1}\hat{\bar{\bf A}}^{l}_{+,1} + \hat{\bar{\bf A}}^{l}_{-,2}\hat{\bar{\bf A}}^{l}_{- ,1} $ and $  \mathcal{\bf A^{(2)}}^{l}_{-,2} = (\hat{\bar{\bf A}}^{l}_{-,2})^{2} + \hat{\bar{\bf A}}^{l}_{-,1}\hat{\bar{\bf A}}^{l}_{+,2} + (n-1) {\bar{\bf E}_{-}}^{2} $.
\\
$ \mathbf{\bf E^{(2)}_{+,1}}_{ij} = \hat{\bar{\bf A}}^{i}_{+,1} {\bf E}_{+} + {\bf E}_{+} \hat{\bar{\bf A}}^{j}_{+,1} + (n-2) {\bf E}_{+}^{2} $, $ \mathbf{\bf E^{(2)}_{+,2}}_{ij} = \hat{\bar{\bf A}}^{i}_{+,2} {\bf E}_{-} + {\bf E}_{+} \hat{\bar{\bf A}}^{j}_{+,2} $, $ \mathbf{\bf E^{(2)}_{-,1}}_{ij} = \hat{\bar{\bf A}}^{i}_{-,1} {\bf E}_{+} + {\bf E}_{-} \hat{\bar{\bf A}}^{j}_{-,1} $ and $ \mathbf{\bf E^{(2)}_{-,2}}_{ij} = \hat{\bar{\bf A}}^{i}_{-,2} {\bf E}_{-} + {\bf E}_{-} \hat{\bar{\bf A}}^{j}_{-,2} + (n-2) {\bf E}_{-}^{2} $. 
Since $ \bar{\bf E}_{+} $ and $ \bar{\bf E}_{+} $ are diagonal then  $ {\bar{\bf E}_{+}} \hat{\bar{\bf A}}^{j}_{+,1} \tilde{=} \hat{\bar{\bf A}}^{j}_{+,1} $,  $ {\bar{\bf E}_{+}} \hat{\bar{\bf A}}^{j}_{+,2} \tilde{=} \hat{\bar{\bf A}}^{j}_{+,2} $,  $ {\bar{\bf E}_{-}} \hat{\bar{\bf A}}^{j}_{-,1} \tilde{=} \hat{\bar{\bf A}}^{j}_{-,1} $ and  $ {\bar{\bf E}_{-}} \hat{\bar{\bf A}}^{j}_{-,2} \tilde{=} \hat{\bar{\bf A}}^{j}_{-,2} $.
Since $ \hat{\bar{\bf A}}^{i}_{+,1} $ and $ \hat{\bar{\bf A}}^{i}_{-,2} $ have diagonal entries then  $ \hat{\bar{\bf A}}^{j}_{+,1} \hat{\bar{\bf A}}^{j}_{+,2} \tilde{=} \hat{\bar{\bf A}}^{j}_{+,2} $ and $ \hat{\bar{\bf A}}^{j}_{-,1} \hat{\bar{\bf A}}^{j}_{-,2} \tilde{=} \hat{\bar{\bf A}}^{j}_{-,1} $.

Since all these matrices are nonnegative then $  \mathcal{\bf A^{(2)}}^{l}_{+,1} = (\hat{\bar{\bf A}}^{l}_{+,1})^{2} + {\bf \Lambda}^{(2)}_{+,1} $, $  \mathcal{\bf A^{(2)}}^{l}_{+,2} \tilde{=} \hat{\bar{\bf A}}^{l}_{+,2} + {\bf \Lambda}^{(2)}_{+,2} $, $  \mathcal{\bf A^{(2)}}^{l}_{-,1} \tilde{=} \hat{\bar{\bf A}}^{l}_{-,1} + {\bf \Lambda}^{(2)}_{-,1} $, $  \mathcal{\bf A^{(2)}}^{l}_{-,2} =  (\hat{\bar{\bf A}}^{l}_{-,2})^{2} + {\bf \Lambda}^{(2)}_{-,2} $, $ \mathbf{\bf E^{(2)}_{+,1}}_{ij} \tilde{=}  \hat{\bar{\bf A}}^{j}_{+,1} + {\bf \tilde \Lambda}^{(2)}_{+,1}  $, $ \mathbf{\bf E^{(2)}_{-,1}}_{ij} \tilde{=}  \hat{\bar{\bf A}}^{j}_{-,1} + {\bf \tilde \Lambda}^{(2)}_{-,1}  $, $ \mathbf{\bf E^{(2)}_{+,2}}_{ij} \tilde{=}  \hat{\bar{\bf A}}^{j}_{+,2} + {\bf \tilde \Lambda}^{(2)}_{+,2}  $, $ \mathbf{\bf E^{(2)}_{-,2}}_{ij} \tilde{=}  \hat{\bar{\bf A}}^{j}_{-,2} + {\bf \tilde \Lambda}^{(2)}_{-,2}  $ where $  {\bf \Lambda}^{(2)}_{+,1} $, $ {\bf \Lambda}^{(2)}_{+,2} $, $  {\bf \Lambda}^{(2)}_{-,1} $, $ {\bf \Lambda}^{(2)}_{-,2} $, $ {\bf \tilde \Lambda}^{(2)}_{+,1} $, $ {\bf \tilde \Lambda}^{(2)}_{+,2} $, $  {\bf \tilde \Lambda}^{(2)}_{-,1} $ and $ {\bf \tilde \Lambda}^{(2)}_{-,2} $ are nonnegative.

And $\hat{\bf A}^{n+1} = $ as shown in the next page.

$  \mathcal{\bf A^{(n+1)}}^{l}_{+,1} = (\hat{\bar{\bf A}}^{l}_{+,1})^{n+1} + {\bf \Lambda}^{(n+1)}_{+,1} $, $  \mathcal{\bf A^{(n+1)}}^{l}_{+,2} \tilde{=} (\hat{\bar{\bf A}}^{l}_{+,2})^{n} + {\bf \Lambda}^{(n+1)}_{+,2} $, $  \mathcal{\bf A^{(n+1)}}^{l}_{-,1} \tilde{=} (\hat{\bar{\bf A}}^{l}_{-,1})^{n} + {\bf \Lambda}^{(n+1)}_{-,1} $, $  \mathcal{\bf A^{(n+1)}}^{l}_{-,2} =  (\hat{\bar{\bf A}}^{l}_{-,2})^{n+1} + {\bf \Lambda}^{(n+1)}_{-,2} $, $ \mathbf{\bf E^{(n+1)}_{+,1}}_{ij} \tilde{=}  (\hat{\bar{\bf A}}^{j}_{+,1})^{n} + {\bf \tilde \Lambda}^{(n+1)}_{+,1}  $, $ \mathbf{\bf E^{(n+1)}_{-,1}}_{ij} \tilde{=}  (\hat{\bar{\bf A}}^{j}_{-,1})^{n} + {\bf \tilde \Lambda}^{(n+1)}_{-,1}  $, $ \mathbf{\bf E^{(n+1)}_{+,2}}_{ij} \tilde{=}  (\hat{\bar{\bf A}}^{j}_{+,2})^{n} + {\bf \tilde \Lambda}^{(n+1)}_{+,2}  $, $ \mathbf{\bf E^{(n+1)}_{-,2}}_{ij} \tilde{=}  (\hat{\bar{\bf A}}^{j}_{-,2})^{n} + {\bf \tilde \Lambda}^{(n+1)}_{-,2}  $ where $  {\bf \Lambda}^{(n+1)}_{+,1} $, $ {\bf \Lambda}^{(n+1)}_{+,2} $, $  {\bf \Lambda}^{(n+1)}_{-,1} $, $ {\bf \Lambda}^{(n+1)}_{-,2} $, $ {\bf \tilde \Lambda}^{(n+1)}_{+,1} $, $ {\bf \tilde \Lambda}^{(n+1)}_{+,2} $, $  {\bf \tilde \Lambda}^{(n+1)}_{-,1} $ and $ {\bf \tilde \Lambda}^{(n+1)}_{-,2} $ are nonnegative.

Thus, since $ \bar{\bar{\bf A}}^{j} $  is stochastic indecomposable and aperiodic (SIA) then there exist $ k > 0 $ such that $ (\bar{\bar{\bf A}}^{j})^{k} $ has a positive column $ \bar{j} $. If the positive column $ 1 \leq \bar{j} \leq n $ then $ (\hat{\bar{\bf A}}^{j}_{+,1})^{k} $ has a positive column $ \bar{j} $ and $ (\hat{\bar{\bf A}}^{j}_{-,1})^{k-1} $ has a positive column $ \bar{j} $. If the positive column $ n + 1 \leq \bar{j} \leq 2 n  $ then $ (\hat{\bar{\bf A}}^{j}_{+,2})^{k-1} $ has a positive column $ \bar{j} $ and $ (\hat{\bar{\bf A}}^{j}_{-,2})^{k} $ has a positive column $ \bar{j} $.
Thus, picking the block column $ j $ in $ \hat{\bf A}^{k+1} $, we see that either $  \mathcal{\bf A^{(k+1)}}^{j}_{+,1} = (\hat{\bar{\bf A}}^{j}_{+,1})^{k+1} + {\bf \Lambda}^{(k+1)}_{+,1} $ has the a positive column with the corresponding entries of $  \mathcal{\bf A^{(k+1)}}^{j}_{-,1} \tilde{=} (\hat{\bar{\bf A}}^{j}_{-,1})^{k} + {\bf \Lambda}^{(k+1)}_{-,1} $, $ \mathbf{\bf E^{(k+1)}_{+,1}}_{ij} \tilde{=}  (\hat{\bar{\bf A}}^{j}_{+,1})^{k} + {\bf \Lambda}^{(k+1)}_{+,3}  $ and  $ \mathbf{\bf E^{(k+1)}_{-,1}}_{ij} \tilde{=}  (\hat{\bar{\bf A}}^{j}_{-,1})^{k} + {\bf \Lambda}^{(k+1)}_{-,3}  $ (i.e., all having a positive column $ \bar{j} $ where $ 1 \leq \bar{j} \leq n $ since  $  {\bf \Lambda}^{(k+1)}_{+,1} $,  $  {\bf \Lambda}^{(k+1)}_{-,1} $, $ {\bf \Lambda}^{(k+1)}_{+,3} $ and  $ {\bf \Lambda}^{(k+1)}_{-,3} $  are nonnegative.
Or $  \mathcal{\bf A^{(k+1)}}^{j}_{+,2} \tilde{=} (\hat{\bar{\bf A}}^{j}_{+,1})^{k} + {\bf \Lambda}^{(k+1)}_{+,2} $ has the a positive column with the corresponding entries of $  \mathcal{\bf A^{(k+1)}}^{j}_{-,2} = (\hat{\bar{\bf A}}^{j}_{-,1})^{k+1} + {\bf \Lambda}^{(k+1)}_{-,2} $, $ \mathbf{\bf E^{(k+1)}_{+,2}}_{ij} \tilde{=}  (\hat{\bar{\bf A}}^{j}_{+,2})^{k} + {\bf \Lambda}^{(k+1)}_{+,4}  $ and  $ \mathbf{\bf E^{(k+1)}_{-,2}}_{ij} \tilde{=}  (\hat{\bar{\bf A}}^{j}_{-,2})^{k} + {\bf \Lambda}^{(k+1)}_{-,4}  $ (i.e., all having a positive column $ \bar{j} $ where $ n + 1 \leq \bar{j} \leq 2 n $ since  $  {\bf \Lambda}^{(k+1)}_{+,2} $,  $  {\bf \Lambda}^{(k+1)}_{-,2} $, $ {\bf \Lambda}^{(k+1)}_{+,4} $ and  $ {\bf \Lambda}^{(k+1)}_{-,4} $  are nonnegative.

Then there exist a positive column in block column $ \bar{j} $ in $ \hat{\bf A}^{k+1} $. 
Thus, $ \hat{\bf A} $ is SIA.

Therefore, $ \lim_{n \rightarrow \infty} \hat{\bf A}^{n} = \mathbf{1}\pi^{T} $. And all rows of  $ \lim_{n \rightarrow \infty} \hat{\bf A}^{n} $ are the same.

\clearpage

$\hat{\bf A}^{n+1} =  \begin{pmatrix}
    { \mathcal{\bf A^{(n+1)}}^{1}_{+,1} \mathcal{\bf A^{(n+1)}}^{1}_{+,2}} \mathbf{\bf E^{(n+1)}_{+,1}}_{12} \ \mathbf{\bf E^{(n+1)}_{+,2}}_{12} \mathbf{\bf E^{(n+1)}_{+,1}}_{13}  \ \mathbf{\bf E^{(n+1)}_{+,2}}_{13} \mathbf{\bf E^{(n+1)}_{+,1}}_{14}  \ \mathbf{\bf E^{(n+1)}_{+,2}}_{14} \ldots\\
    { \mathcal{\bf A^{(n+1)}}^{1}_{-,1}  \mathcal{\bf A^{(n+1)}}^{1}_{-,2}}   \ \mathbf{\bf E^{(n+1)}_{-,1}}_{12} \mathbf{\bf E^{(n+1)}_{-,2}}_{12}   \ \mathbf{\bf E^{(n+1)}_{-,1}}_{13} \mathbf{\bf E^{(n+1)}_{-,2}}_{13}  \mathbf{\bf E^{(n+1)}_{-,1}}_{14} \mathbf{\bf E^{(n+1)}_{-,2}}_{14} \ldots\\
    \mathbf{\bf E^{(n+1)}_{+,1}}_{21}  \ \mathbf{\bf E^{(n+1)}_{+,2}}_{21}  { \mathcal{\bf A^{(n+1)}}^{2}_{+,1}  \mathcal{\bf A^{(n+1)}}^{2}_{+,2}} \mathbf{\bf E^{(n+1)}_{+,1}}_{23}  \ \mathbf{\bf E^{(n+1)}_{+,2}}_{23} \mathbf{\bf E^{(n+1)}_{+,1}}_{24}  \ \mathbf{\bf E^{(n+1)}_{+,2}}_{24}  \ldots\\
     \mathbf{\bf E^{(n+1)}_{-,1}}_{21} \ \mathbf{\bf E^{(n+1)}_{-,2}}_{21}  { \mathcal{\bf A^{(n+1)}}^{2}_{-,1}  \mathcal{\bf A^{(n+1)}}^{2}_{-,2}} \mathbf{\bf E^{(n+1)}_{-,1}}_{23} \mathbf{\bf E^{(n+1)}_{-,2}}_{23}   \ \mathbf{\bf E^{(n+1)}_{-,1}}_{24}\ \mathbf{\bf E^{(n+1)}_{-,2}}_{24} \ldots \\
    &&\diagentry{\xddots}\\
    &&\diagentry{\xddots}\\
    \mathbf{\bf E^{(n+1)}_{+,1}}_{n1}  \ \mathbf{\bf E^{(n+1)}_{+,2}}_{n1} \mathbf{\bf E^{(n+1)}_{+,1}}_{n2}  \  \mathbf{\bf E^{(n+1)}_{+,2}}_{n2} \mathbf{\bf E^{(n+1)}_{+,1}}_{n3}  \  \mathbf{\bf E^{(n+1)}_{+,2}}_{n3} \ldots \diagentry{ \mathcal{\bf A^{(n+1)}}^{n}_{+,1}  \mathcal{\bf A^{(n+1)}}^{n}_{+,2}}\\
     \mathbf{\bf E^{(n+1)}_{-,1}}_{n1} \mathbf{\bf E^{(n+1)}_{-,1}}_{n1}  \ \mathbf{\bf E^{(n+1)}_{-,1}}_{n2} \mathbf{\bf E^{(n+1)}_{-,2}}_{n2}  \ \mathbf{\bf E^{(n+1)}_{-,1}}_{n3} \ \mathbf{\bf E^{(n+1)}_{-,2}}_{n3} \ldots \diagentry{ \mathcal{\bf A^{(n+1)}}^{n}_{-,1}  \mathcal{\bf A^{(n+1)}}^{n}_{-,2}}
\end{pmatrix}$

where
{\scriptsize
\begin{equation}
\begin{split}
\begin{align*}
  {\bf Q}       =  \left(\begin{array}{cc}
                \hat{\bf A} & \ \ \ \ \  {\bf 0}_{2n^{2} \times 2n^{2}}  \\
         {\bf I}_{2n^{2} \times 2n^{2}} - \hat{\bf A}  &  
         \begin{pmatrix}
    \diagentry{\frac{1}{n}{\bf D}^{1}} & & \frac{1}{n}{\bf D}^{2}  & & & \ldots &  \frac{1}{n}{\bf D}^{n}\\
    &\diagentry{\frac{1}{n}{\bf D}^{1}} & & \frac{1}{n}{\bf D}^{2} & & & \ldots &  \frac{1}{n}{\bf D}^{n}\\
    \frac{1}{n}{\bf D}^{1} & & \diagentry{\frac{1}{n}{\bf D}^{2}}\\
    & \frac{1}{n}{\bf D}^{1} & &\diagentry{\frac{1}{n}{\bf D}^{2}}\\
    &&&&\diagentry{\xddots}\\
    &&&&&\diagentry{\xddots}\\
    \frac{1}{n}{\bf D}^{1}&&&&&&\diagentry{\frac{1}{n}{\bf D}^{n}}\\
    &\frac{1}{n}{\bf D}^{1}&&&&&&\diagentry{\frac{1}{n}{\bf D}^{n}}
\end{pmatrix}
    \end{array}\right) 
    \end{align*} 
\end{split}
\end{equation}}
 and 
 
{\scriptsize \begin{equation}\label{matrixG}
 {\bf G}     =  \left(\begin{array}{cc}
        {\bf 0}_{2n^{2} \times 2n^{2}} &   
        \begin{pmatrix}
    \diagentry{{\bf T}^{1}} & & {\bf T}^{2}  & & & \ldots &  {\bf T}^{n}\\
    &\diagentry{{\bf T}^{1}} & & {\bf T}^{2} & & & \ldots &  {\bf T}^{n}\\
    {\bf T}^{1} & & \diagentry{{\bf T}^{2}}\\
    & {\bf T}^{1} & &\diagentry{{\bf T}^{2}}\\
    &&&&\diagentry{\xddots}\\
    &&&&&\diagentry{\xddots}\\
    {\bf T}^{1}&&&&&&\diagentry{{\bf T}^{n}}\\
    &{\bf T}^{1}&&&&&&\diagentry{{\bf T}^{n}}
\end{pmatrix} \\
        {\bf 0}_{2n^{2} \times 2n^{2}} & 
        - \begin{pmatrix}
    \diagentry{{\bf T}^{1}} & & {\bf T}^{2}  & & & \ldots &  {\bf T}^{n}\\
    &\diagentry{{\bf T}^{1}} & & {\bf T}^{2} & & & \ldots &  {\bf T}^{n}\\
    {\bf T}^{1} & & \diagentry{{\bf T}^{2}}\\
    & {\bf T}^{1} & &\diagentry{{\bf T}^{2}}\\
    &&&&\diagentry{\xddots}\\
    &&&&&\diagentry{\xddots}\\
    {\bf T}^{1}&&&&&&\diagentry{{\bf T}^{n}}\\
    &{\bf T}^{1}&&&&&&\diagentry{{\bf T}^{n}}
\end{pmatrix} \\
    \end{array}\right) 
\end{equation}}

where $ {\bf T} $ is $ 2n^{2} \times 2n^{2} $ such that $ {\bf T} = $
$ 
 \begin{pmatrix}
    \diagentry{{\bf T}^{1}} & & {\bf T}^{2}  & & & \ldots &  {\bf T}^{n}\\
    &\diagentry{{\bf T}^{1}} & & {\bf T}^{2} & & & \ldots &  {\bf T}^{n}\\
    {\bf T}^{1} & & \diagentry{{\bf T}^{2}}\\
    & {\bf T}^{1} & &\diagentry{{\bf T}^{2}}\\
    &&&&\diagentry{\xddots}\\
    &&&&&\diagentry{\xddots}\\
    {\bf T}^{1}&&&&&&\diagentry{{\bf T}^{n}}\\
    &{\bf T}^{1}&&&&&&\diagentry{{\bf T}^{n}}
\end{pmatrix}$   

\clearpage

\begin{equation}
\begin{split}
\begin{align*}
 {\bf Q}^{(k)}      =  \left(\begin{array}{cc}
                \hat{\bf A}(k) & \ \ \ \ \  {\bf 0}_{2n^{2} \times 2n^{2}}  \\
         {\bf I}_{2n^{2} \times 2n^{2}} - \hat{\bf A}(k)  &  
         \begin{pmatrix}
    \diagentry{\frac{1}{n}{\bf D}^{1}(k)} & & \frac{1}{n}{\bf D}^{2}(k)  & & & \ldots &  \frac{1}{n}{\bf D}^{n}(k)\\
    &\diagentry{\frac{1}{n}{\bf D}^{1}(k)} & & \frac{1}{n}{\bf D}^{2}(k) & & & \ldots &  \frac{1}{n}{\bf D}^{n}(k)\\
    \frac{1}{n}{\bf D}^{1}(k) & & \diagentry{\frac{1}{n}{\bf D}^{2}(k)}\\
    & \frac{1}{n}{\bf D}^{1}(k) & &\diagentry{\frac{1}{n}{\bf D}^{2}(k)}\\
    &&&&\diagentry{\xddots}\\
    &&&&&\diagentry{\xddots}\\
    \frac{1}{n}{\bf D}^{1}(k)&&&&&&\diagentry{\frac{1}{n}{\bf D}^{n}(k)}\\
    &\frac{1}{n}{\bf D}^{1}(k)&&&&&&\diagentry{\frac{1}{n}{\bf D}^{n}(k)}
\end{pmatrix}
    \end{array}\right) 
    \end{align*} 
\end{split}
\end{equation}
 and 
 
\begin{equation}\label{matrixG}
\begin{split}
{\bf G}(k)      =  \left(\begin{array}{cc}
        {\bf 0}_{2n^{2} \times 2n^{2}} &   
        \begin{pmatrix}
    \diagentry{{\bf T}^{1}(k)} & & {\bf T}^{2}(k)  & & & \ldots &  {\bf T}^{n}(k)\\
    &\diagentry{{\bf T}^{1}(k)} & & {\bf T}^{2}(k) & & & \ldots &  {\bf T}^{n}(k)\\
    {\bf T}^{1}(k) & & \diagentry{{\bf T}^{2}(k)}\\
    & {\bf T}^{1}(k) & &\diagentry{{\bf T}^{2}(k)}\\
    &&&&\diagentry{\xddots}\\
    &&&&&\diagentry{\xddots}\\
    {\bf T}^{1}(k) &&&&&&\diagentry{{\bf T}^{n}(k)}\\
    &{\bf T}^{1}(k) &&&&&&\diagentry{{\bf T}^{n}(k)}
\end{pmatrix} \\
        {\bf 0}_{2n^{2} \times 2n^{2}} & 
         -  \begin{pmatrix}
    \diagentry{{\bf T}^{1}(k)} & & {\bf T}^{2}(k)  & & & \ldots &  {\bf T}^{n}(k)\\
    &\diagentry{{\bf T}^{1}(k)} & & {\bf T}^{2}(k) & & & \ldots &  {\bf T}^{n}(k)\\
    {\bf T}^{1}(k) & & \diagentry{{\bf T}^{2}(k)}\\
    & {\bf T}^{1}(k) & &\diagentry{{\bf T}^{2}(k)}\\
    &&&&\diagentry{\xddots}\\
    &&&&&\diagentry{\xddots}\\
    {\bf T}^{1}(k) &&&&&&\diagentry{{\bf T}^{n}(k)}\\
    &{\bf T}^{1}(k) &&&&&&\diagentry{{\bf T}^{n}(k)}
\end{pmatrix} \\
    \end{array}\right) 
 \end{split}
\end{equation}

where $ {\bf T}(k) $ is $ 2n^{2} \times 2n^{2} $ such that 
$ {\bf T}(k) = \begin{pmatrix}
    \diagentry{{\bf T}^{1}(k)} & & {\bf T}^{2}(k)  & & & \ldots &  {\bf T}^{n}(k)\\
    &\diagentry{{\bf T}^{1}(k)} & & {\bf T}^{2}(k) & & & \ldots &  {\bf T}^{n}(k)\\
    {\bf T}^{1}(k) & & \diagentry{{\bf T}^{2}(k)}\\
    & {\bf T}^{1}(k) & &\diagentry{{\bf T}^{2}(k)}\\
    &&&&\diagentry{\xddots}\\
    &&&&&\diagentry{\xddots}\\
    {\bf T}^{1}(k) &&&&&&\diagentry{{\bf T}^{n}(k)}\\
    &{\bf T}^{1}(k) &&&&&&\diagentry{{\bf T}^{n}(k)}
\end{pmatrix} $

\clearpage

\begin{equation*}
\begin{split}
{\bf P} =  \lim_{k \rightarrow \infty} {\bf Q}_{\epsilon , {\bf T}} ^{k} = \lim_{k \rightarrow \infty}{\bf Q} ^{k} = {\bf P} + \lim_{k \rightarrow \infty}\sum_{i=4}^{4n}{\bf P}_{i}{\bf J}_{i}{\bf Q}_{i} 
\end{split}
\end{equation*}

\begin{equation}
\begin{split}
\begin{align*}
  = \lim_{k \rightarrow \infty}  \left(\begin{array}{cc}
                \hat{\bf A} & \ \ \ \ \  {\bf 0}_{2n^{2} \times 2n^{2}}  \\
         {\bf I}_{2n^{2} \times 2n^{2}} - \hat{\bf A}  &  
        \begin{pmatrix}
    \diagentry{\frac{1}{n}{\bf D}^{1}} & & \frac{1}{n}{\bf D}^{2}  & & & \ldots &  \frac{1}{n}{\bf D}^{n}\\
    &\diagentry{\frac{1}{n}{\bf D}^{1}} & & \frac{1}{n}{\bf D}^{2} & & & \ldots &  \frac{1}{n}{\bf D}^{n}\\
    \frac{1}{n}{\bf D}^{1} & & \diagentry{\frac{1}{n}{\bf D}^{2}}\\
    & \frac{1}{n}{\bf D}^{1} & &\diagentry{\frac{1}{n}{\bf D}^{2}}\\
    &&&&\diagentry{\xddots}\\
    &&&&&\diagentry{\xddots}\\
    \frac{1}{n}{\bf D}^{1}&&&&&&\diagentry{\frac{1}{n}{\bf D}^{n}}\\
    &\frac{1}{n}{\bf D}^{1}&&&&&&\diagentry{\frac{1}{n}{\bf D}^{n}}
\end{pmatrix}
    \end{array}\right) 
    \end{align*} 
\end{split}
\end{equation}

\end{proof}

\end{document}